\documentclass{article}
\usepackage{iclr2027_conference,times}

\usepackage{amsmath,amssymb,amsthm}

\usepackage{booktabs}
\usepackage{multirow}
\usepackage{graphicx}
\graphicspath{{figures/}}
\usepackage{subcaption}
\usepackage[dvipsnames,table]{xcolor}
\definecolor{oursrow}{RGB}{228,240,252}

\usepackage{algorithm}
\usepackage{algpseudocode}

\PassOptionsToPackage{hyphens}{url}
\usepackage{hyperref}
\usepackage{url}
\usepackage{placeins}
\usepackage{float}
\usepackage{enumitem}

\hypersetup{colorlinks=true, linkcolor=NavyBlue, citecolor=NavyBlue, urlcolor=NavyBlue,
  pdftitle={QUASAR: Lowering the Loss Floor of Quantization-Aware Training with Loss-Aware Reconstruction},
  pdfauthor={Vincent Counathe, Ben Athiwaratkun, Christopher De Sa, Tianyi Zhang}}

\theoremstyle{plain}
\newtheorem{theorem}{Theorem}
\theoremstyle{remark}
\newtheorem{remark}{Remark}

\DeclareMathOperator{\round}{round}

\newcommand{\R}{\mathbb{R}}
\newcommand{\E}{\mathbb{E}}

\newcommand{\loss}{\mathcal{L}}
\newcommand{\Hess}{H}

\newcommand{\norm}[1]{\left\lVert #1\right\rVert}

\title{QUASAR: Lowering the Loss Floor of Quantization-Aware Training with Loss-Aware Reconstruction}

\author{Vincent Counathe$^{1,2}$, Ben Athiwaratkun$^{2}$, Christopher De Sa$^{1,2}$ \& Tianyi Zhang$^{2}$ \\
$^{1}$Cornell University \quad $^{2}$Together AI \\ \\
\textbf{Code:} \href{https://github.com/vincentcounathe/quasar-qat}{\texttt{github.com/vincentcounathe/quasar-qat}} \\
\textbf{Checkpoints:} \href{https://huggingface.co/QUASAR-QAT}{\texttt{huggingface.co/QUASAR-QAT}}
}

\iclrfinalcopy

\begin{document}
\maketitle
\lhead{Preprint}

\begin{abstract}

As large language model inference shifts toward lower precision, post-training quantization (PTQ) becomes increasingly brittle, making quantization-aware training (QAT) essential for preserving model quality. However, QAT has a structural mismatch: gradient updates are applied to latent full-precision weights, while the loss and gradients are computed on lossy reconstructions of those weights. This mismatch can lead to suboptimal training trajectories and a higher loss floor. Second-order PTQ methods address a similar problem by minimizing loss-aware reconstruction error, but applying such expensive reconstruction repeatedly during QAT as the weights evolve is impractical. We introduce QUASAR, a QAT method that brings lightweight, loss-aware reconstruction into the training loop. At each training step, QUASAR reconstructs the latent weights by searching over a small set of clipping ranges and fitting dequantization parameters through saliency-weighted least squares. We use an exponential moving average of squared gradients as the per-parameter saliency signal. Our theoretical analysis shows that optimizing QUASAR's reconstruction objective tightens both the convergence and final-loss bounds of QAT. We evaluate QUASAR across four model families and across INT4, INT3, INT2, and NVFP4 quantization formats. QUASAR consistently achieves lower training and evaluation loss than competitive QAT methods and outperforms QAT and PTQ baselines on downstream benchmarks. At INT2, QUASAR improves average accuracy over the best QAT baseline by 13.3 points with quantization-aware distillation and by 10.9 points with QAT on mathematical reasoning data. Notably, after distillation on only $\sim$600M tokens, QUASAR's INT4 Gemma-4 E4B checkpoint outperforms the corresponding QAT checkpoint released by Google, with 66\% lower KL divergence and 1.8 points higher average accuracy.

\end{abstract}

\begingroup

\section{Introduction}

As large language models (LLMs) grow in size and adoption, inference is overtaking
training as the dominant recurring cost of deployment~\citep{ai-inference-trends,beyond-chinchilla}.
To reduce this cost, inference is rapidly moving toward low-precision formats such as
FP4 and INT4~\citep{mxfp4,nvfp4,gptq}, which shrink the model's memory footprint,
reduce decoding latency, and enable higher concurrency and
throughput~\citep{qserve}. However, since models are typically trained in high
precision but served in low precision~\citep{survey-efficient-inference}, preserving
model quality as precision becomes increasingly aggressive remains a critical
challenge. Post-training quantization (PTQ) has been the standard way to convert a
trained model to low precision, but it is becoming increasingly brittle on newer
models. Reasoning and agentic models often operate over long contexts, where
quantization errors can accumulate and degrade
performance~\citep{quantization-hurts-reasoning,quantization-meets-reasoning,quantization-long-context}.
PTQ models may also exhibit behavioral shifts and biases relative to their
full-precision counterparts~\citep{quantization-multilingual,quantbias,accuracy-not-all}.
Producing a native low-precision model with stronger quality guarantees therefore
requires training under the same precision used for serving.

\begin{figure}[t]
\centering
\includegraphics[width=\linewidth]{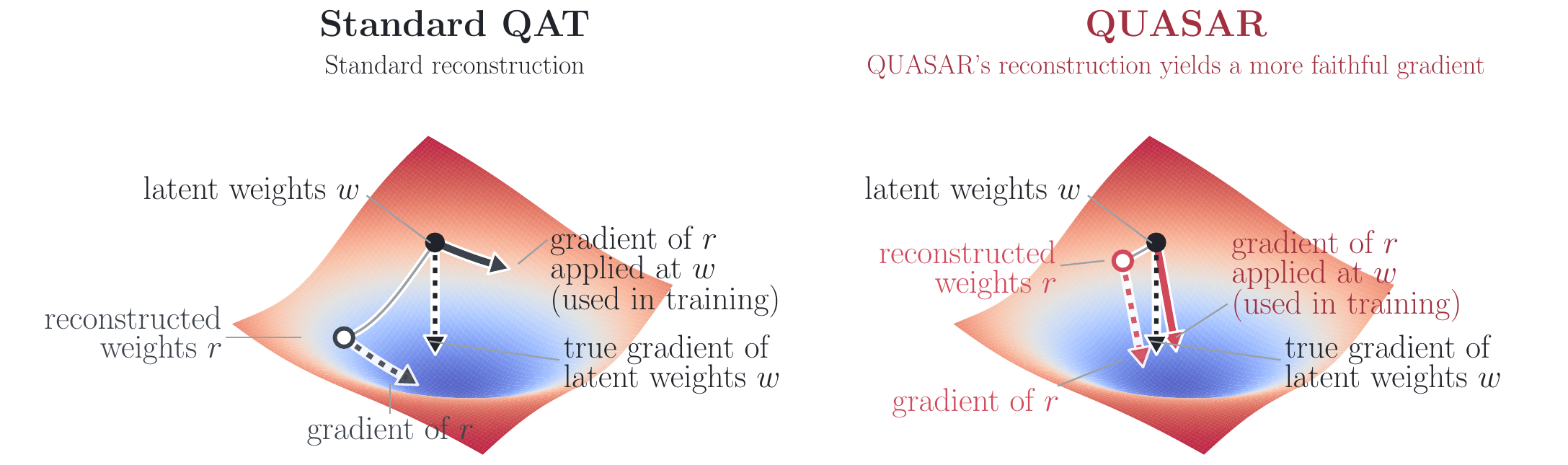}
\caption{\textbf{The structural mismatch of gradients in standard QAT versus QUASAR.} Under the STE, both methods evaluate the gradient at the reconstructed weights $r$ and apply it to the latent weights $w$. QUASAR minimizes the loss-aware reconstruction error, making the gradient at $r$ a more faithful proxy for the gradient at $w$.}
\label{fig:borrowed-grad}
\end{figure}
\FloatBarrier

Quantization-aware training (QAT) addresses this need by inserting quantization into
the training loop, allowing the model to adapt to quantization noise as it
learns. QAT, however, introduces a structural mismatch into optimization
(Figure~\ref{fig:borrowed-grad}):
the forward pass and loss computation use quantized--dequantized reconstruction weights $r$,
while the gradient updates are applied to the latent full-precision weights $w$. Because
quantization has zero gradient almost everywhere, QAT typically uses the
straight-through estimator (STE)~\citep{ste}, approximating
$\frac{\partial \loss}{\partial w}$ with $\frac{\partial \loss}{\partial r}$. This surrogate is generally not the optimal descent direction for the latent weights,
leading training along a suboptimal trajectory. The consequence is a \emph{loss-floor gap}:
given the same model and training data, QAT converges to a higher final training and
evaluation loss than full-precision training~\citep{chen2025qatscaling,kumar2025scalingprecision}. Prior QAT methods narrow this gap by softening the rounding operation, learning
quantization parameters jointly with the weights, or improving initialization~\citep{dsq,lsq,bitdistiller}. We instead draw on a complementary insight
from second-order PTQ methods: the excess training loss caused by QAT's structural mismatch can be reduced by optimizing reconstruction.

A second-order expansion of the loss around the full-precision
weights $w$, assuming that the first-order term is negligible near an optimum, gives
\begin{equation*}
\loss(r)-\loss(w) \approx \frac{1}{2}S,
\qquad
S=(r-w)^{\top}H(r-w),
\end{equation*}
where $H(w)$ is the Hessian evaluated at $w$. We refer to $S$ as the
\emph{loss-aware reconstruction error}. Many modern PTQ methods~\citep{adaround,brecq,obq,gptq,quip} minimize this
quantity, or an approximation to it, to construct a low-precision model whose loss
remains close to that of a frozen full-precision model. Applying the same principle to QAT would keep the low-precision loss \(\loss(r)\) closer to the full-precision loss \(\loss(w)\) as the latent weights evolve, thereby mitigating the effect of the structural mismatch. The challenge, however, is computational cost. In PTQ, the full-precision model is
frozen, so its reconstruction needs to be optimized only once, typically requiring
minutes to hours for a billion-parameter LLM. In QAT, however, the latent weights
change at every step, requiring the reconstruction to be continually re-optimized.
Running a conventional second-order optimization procedure on every forward pass
would therefore be prohibitively expensive.

We introduce \textbf{QUASAR} (\textbf{Qu}antization-\textbf{a}ware training with lo\textbf{s}s-\textbf{a}ware \textbf{r}econstruction), a QAT method that continually minimizes loss-aware reconstruction error to improve the training trajectory and lower the loss floor. QUASAR decomposes reconstruction into two stages: \emph{quantization}, which maps full-precision weights to discrete codes while treating the clipping range as a free parameter, and \emph{dequantization}, which maps those codes back to reconstructed weights using a learned scale and offset. To make the loss-aware objective tractable during training, QUASAR approximates the Hessian using an exponential moving average of squared gradients, yielding per-parameter saliency scores that estimate each weight's effect on the loss. It then searches over a small set of candidate clipping ranges, each defining a different assignment of discrete codes; for each assignment, it solves for the dequantization parameters in closed form by saliency-weighted least squares and selects the candidate with the smallest loss-aware reconstruction error. Our theoretical analysis shows that reducing this error tightens the QAT convergence bound and final-loss bound (Section~\ref{sec:theory}).

Our contributions are threefold. \textbf{(i)} We identify loss-aware reconstruction error, previously optimized for frozen models in second-order PTQ, as a principled objective for QAT. We introduce QUASAR to minimize this objective at every training step using an exponential moving average of squared gradients, closed-form saliency-weighted least squares, and a small clipping-range search.
\textbf{(ii)} Under empirically validated assumptions, we prove that reducing this error tightens bounds on both QAT convergence and the loss of the final quantized model. \textbf{(iii)} Across four model families and four quantization formats, QUASAR consistently reduces loss and improves downstream performance over competitive QAT and PTQ baselines. At INT2, it improves average accuracy by up to 13.3 points over the strongest QAT baseline; at INT4, its Gemma-4 E4B checkpoint achieves 66\% lower held-out KL and 1.8 points higher average accuracy than Google's public QAT checkpoint.

\section{Related Work}\label{sec:setup}

QUASAR sits between two lines of work: PTQ and QAT. PTQ treats quantization as a one-shot algorithm: given a trained model and a small calibration set, select low-precision values that damage the loss least, typically by minimizing a second-order expansion of the loss \citep{obs,obq,gptq} or learning the rounding against a reconstruction objective \citep{adaround,brecq,yaqa}. All rest on the premise that the weights are frozen, so the reconstruction happens only once. A second PTQ line spreads outliers with random orthogonal transforms and replaces the quantization grid with lattice, trellis, or additive codebooks \citep{quip,quipsharp,qtip,aqlm}; these are the strongest PTQ methods at two and three bits, but their inference-time transforms and decodes hurt throughput, while hardware moves toward simpler formats such as MXFP4 and NVFP4 \citep{mxfp4,nvfp4}.

Below four bits, the PTQ methods that do not add inference overhead degrade sharply. QAT reaches higher quality at these bit widths by learning step sizes \citep{lsq}, supervising with the full-precision model \citep{llmqat,bitdistiller}, or freezing the quantized weights and training low-rank adapters \citep{qlora}; none of these optimize the map from latent weights to reconstructed weights, which is set by heuristics. QUASAR directly optimizes that map: it brings the reconstruction PTQ solves once into the training loop as the weights adapt, while still deploying exactly as an RTN-quantized model. Other methods are complementary and compose with QUASAR: EfficientQAT's second stage learns only the step sizes, the same quantizer parameter as LSQ, and its block-wise first stage is a complementary initialization \citep{efficientqat}; ParetoQ splits a fixed training budget between full-precision training and QAT \citep{paretoq}.

\section{Methodology}\label{sec:method}

We first describe weight reconstruction in standard QAT and explain how reconstruction error affects the training loss. We then introduce \textsc{QUASAR}, which makes this reconstruction loss-aware through an efficient curvature approximation, scale search over clipping ranges, and optimal dequantization.

\subsection{Preliminary: Reconstruction in QAT}\label{sec:borrowed}

QAT maintains and updates full-precision latent weights $w$, but computes the forward pass and loss using reconstructed weights $r$. At each step, QAT maps $w$ to low-bit codes $q$ and then dequantizes those codes to obtain $r$:
$w \mapsto q \mapsto r$.
We refer to this process as \emph{weight reconstruction}. It exposes the model to inference-time quantization error during training while allowing the latent weights to adapt. Because the code assignment involves rounding and has no useful derivative, QAT uses the straight-through estimator (STE)~\citep{ste}:
$\frac{\partial \loss}{\partial w} \approx \frac{\partial \loss}{\partial r}$.
The gradient evaluated at $r$ is therefore applied to $w$ (Figure~\ref{fig:borrowed-grad}).

For asymmetric integer quantization, each weight matrix is partitioned into groups of $g$ contiguous weights, commonly with $g=32$ or $128$. A group $w\in\R^g$ is represented by $b$-bit integer codes $q_i\in\{0,\ldots,2^b-1\}$. All weights in the group share a scale $s$ and an offset $z$, which determine the quantization grid. Quantization assigns codes and dequantization reconstructs weights as
\begin{equation}\label{eq:rtn}
q_i
=
\mathrm{clamp}_{[0,2^b-1]}\!\left(
\left\lfloor\frac{w_i-z}{s}\right\rceil
\right),
\qquad
r_i=sq_i+z.
\end{equation}
Reconstruction has two stages: \emph{code assignment}, $w\mapsto q$, and \emph{dequantization}, $q\mapsto r$. Standard QAT uses the minimum and maximum weights in each group for both stages, setting $z=\min_i w_i$ and $s=(\max_i w_i-\min_i w_i)/(2^b-1)$. This uniform grid represents the full range of weight values without clipping and preserves the extreme values in the group, but it may not be optimal for reducing the training loss.

Reconstruction quality directly affects the training loss of QAT. At step $t$, consider the difference between the loss at the reconstructed weights and the loss at the latent weights, $\loss(r_t)-\loss(w_t)$. When the first-order term is negligible, as near a stationary region, a second-order expansion around $w_t$, with $\Hess_t=\nabla^2\loss(w_t)$, gives
\begin{equation}\label{eq:st}
\loss(r_t)-\loss(w_t)
\approx
\frac{1}{2}S_t,
\qquad
\text{where } S_t=(r_t-w_t)^\top\Hess_t(r_t-w_t).
\end{equation}
We call $S_t$ the \emph{loss-aware reconstruction error}. It measures the local change in loss caused by reconstructing the latent weights. Because both the forward pass and the STE gradient are evaluated at $r_t$, this error can affect every optimization step. Second-order PTQ methods minimize this objective, or a related approximation, once for a frozen model~\citep{adaround,gptq,yaqa}, but the process can take hours for a large model. Repeating this optimization at every QAT training step is impractical because the latent weights change after every update. \textsc{QUASAR} instead minimizes a close approximation to $S_t$ at every training step using a lightweight procedure that combines scale search with a closed-form fit for dequantization.

\begin{figure}[t]
\centering
\includegraphics[width=\linewidth]{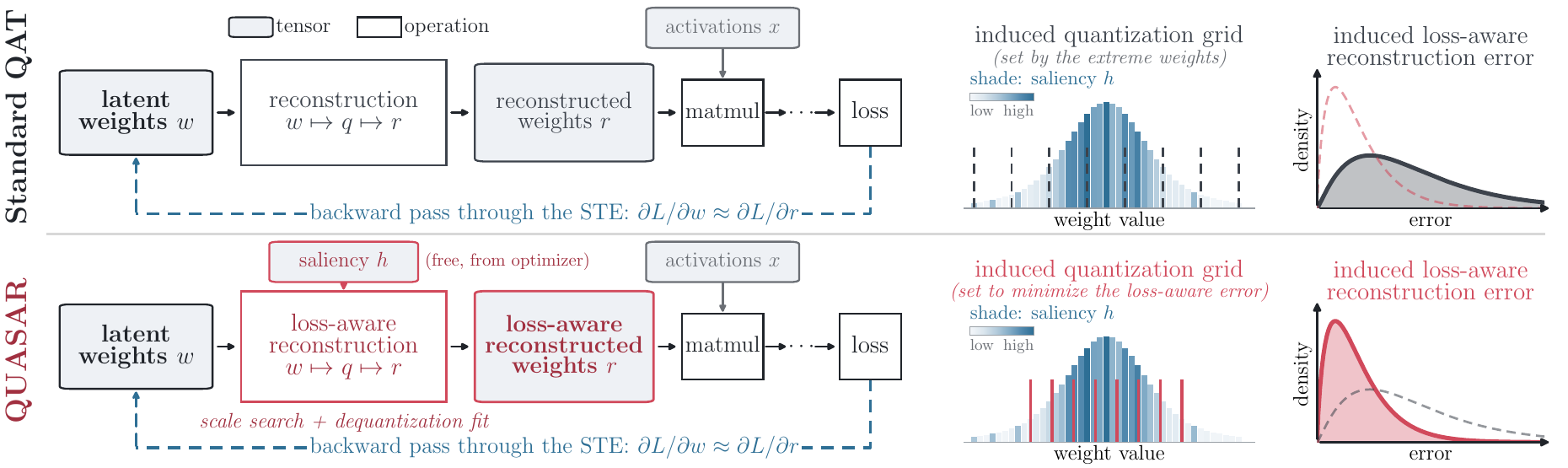}
\caption{The training loops for standard QAT and QUASAR differ only in the reconstruction $w\mapsto q\mapsto r$. Standard QAT determines the quantization grid from each group's extreme weights; QUASAR searches over code assignments and fits the dequantizer to minimize loss-aware reconstruction error. The right panels show the resulting grids and error distributions; empirical measurements appear in Appendix~\ref{app:method}.}
\label{fig:method-simple}
\end{figure}

\subsection{A Tractable Loss-Aware Objective}

\textsc{QUASAR} aims to minimize the loss-aware reconstruction error $S_t$ in Equation~\eqref{eq:st}. The exact Hessian $H$ in $S_t$ is intractable to compute, so an approximation is required. One could consider computing the Hessian Vector Product (HVP), but it is also too expensive since it requires computing an HVP for every candidate reconstruction for every training step. We therefore use the diagonal Fisher information matrix as a proxy for the Hessian. Following~\citet{fishersforfree}, we approximate it with an exponential moving average (EMA) of squared gradients. Adam and AdamW already maintain this quantity as the second-moment estimate $v_t$~\citep{adam}, allowing \textsc{QUASAR} to reuse it without additional memory overhead. The EMA also smooths noise from individual training steps.

Let $g_i=\partial\loss/\partial w_i$ be the gradient with respect to the latent weight $w_i$, and define $h_i=\operatorname{EMA}(g_i^2)$. We use $h_i$ as the $i$th diagonal entry of the Hessian approximation, giving the objective
\begin{equation}\label{eq:fisher-obj}
  S_t
  \;\approx\;
  \widehat S_t
  =
  \sum_i h_i\,(r_i-w_i)^2 .
\end{equation}
We call $h_i$ the \emph{saliency} of $w_i$. Saliency varies substantially within a weight group (Appendix~\ref{app:method}, Figure~\ref{fig:curvature-heatmap}). Because all $g$ weights in a group share dequantization parameters, reducing the reconstruction error of one weight can increase that of another. The weights $h_i$ govern this tradeoff: \textsc{QUASAR} prioritizes accurate reconstruction of high-saliency weights while tolerating more error on low-saliency weights (Appendix~\ref{app:method}, Figure~\ref{fig:beforeafter}).

\subsection{Optimizing Both Stages of Reconstruction}

To minimize the objective in Equation~\eqref{eq:fisher-obj}, \textsc{QUASAR} optimizes code assignment $w\mapsto q$ and dequantization $q\mapsto r$ separately. Standard QAT uses the same scale and offset to assign codes and reconstruct weights. In contrast, \textsc{QUASAR} searches over candidate clipping scales to find code assignments and, for each assignment, computes the optimal dequantization scale and offset in closed form. It then selects the candidate with the smallest loss-aware reconstruction error. Algorithm~\ref{alg:quasar} summarizes the procedure.

\textbf{A weighted least-squares fit gives the optimal dequantization.}
For a fixed code assignment $q$, \textsc{QUASAR} solves for $(s,z)$ in $r_i=sq_i+z$ to minimize Equation~\eqref{eq:fisher-obj}. This is a weighted linear regression and therefore has a closed-form solution (Equation~\eqref{eq:ridgefit}, Appendix~\ref{app:formats}).

\textbf{Scale search finds good code assignments.}
The closed-form fit is optimal for a fixed $q$, but the code assignment itself may be far from optimal. \textsc{QUASAR} therefore searches over scale factors that shrink the clipping range. Each range creates a different code assignment. For a group with midpoint $c$ and half-range $\rho$, a factor $f\in(0,1]$ defines the assignment range $I_f$, where
\begin{equation}\label{eq:assignment-range}
c=\frac{\min_i w_i+\max_i w_i}{2},
\qquad
\rho=\frac{\max_i w_i-\min_i w_i}{2},
\qquad
I_f=[c-f\rho,\,c+f\rho].
\end{equation}
The case $f=1$ is the full min--max range used by standard QAT. A factor $f<1$ narrows the range and clips weights outside it, so the quantization levels cover a smaller interval with finer spacing. In our integer experiments, we search $\mathcal F=\{0.30,0.35,\ldots,1.00\}$.

For each candidate $f\in\mathcal F$, \textsc{QUASAR} first uses the corresponding clipping range to assign codes $q_f$. It then fits the optimal dequantization parameters $(s_f,z_f)$ for $q_f$ using Equation~\eqref{eq:ridgefit} and computes the resulting error $\widehat S_t(f)$. After evaluating all candidates, \textsc{QUASAR} uses the reconstruction associated with $f^\star=\arg\min_{f\in\mathcal F}\widehat S_t(f)$.
This search is performed as part of each QAT forward pass. Its empirical behavior is shown in Figures~\ref{fig:range-decomp} and~\ref{fig:heald-ss} in Appendix~\ref{app:method}.

\textsc{QUASAR} changes only the weight reconstruction used in the forward pass. It leaves activation quantization and backpropagation unchanged: gradients still pass through the weight quantizer with the standard STE, and neither scale search nor dequantization fitting requires a separate backward rule. The same two-stage procedure applies to asymmetric and symmetric integer formats as well as floating-point formats such as NVFP4 and MXFP4 (the corresponding algorithms are presented in Appendix~\ref{app:formats}). We evaluate \textsc{QUASAR} for NVFP4 in Section~\ref{sec:exp-nvfp4}.

\begin{algorithm}[t]
\caption{\textbf{QUASAR reconstruction} of one weight group, run at every
training step. A fully specified version appears in
Algorithm~\ref{alg:quasar-affine} (Appendix~\ref{app:formats}).}
\label{alg:quasar}
\small
\begin{algorithmic}[1]
\Require latent weights $w$ of one group; saliency $h$ (Adam's second
moment); clipping factors $\mathcal F \subseteq (0,1]$
\State $\widehat S^\star \gets \infty$
\For{$f \in \mathcal F$}
  \State $q_f \gets$ quantized codes of $w$ using the clipping range $I_f$
         \Comment{Eqs.~\eqref{eq:rtn} and~\eqref{eq:assignment-range}}
  \State $(s_f,z_f) \gets$ $h$-weighted least-squares fit of $(s,z)$ given
         $(w,q_f)$ \Comment{Eq.~\eqref{eq:ridgefit}}
  \State $r_f \gets s_fq_f+z_f\mathbf{1}$
         \Comment{best reconstruction for clipping factor $f$}
  \State $\widehat S_f \gets \sum_i h_i(r_{f,i}-w_i)^2$
         \Comment{loss-aware reconstruction error}
  \State \textbf{if} $\widehat S_f<\widehat S^\star$ \textbf{then}
         $\widehat S^\star \gets \widehat S_f$; $r \gets r_f$ \textbf{end if}
\EndFor
\State \Return $r$
\end{algorithmic}
\end{algorithm}

\FloatBarrier

\section{Theory: Reconstruction Error Bounds QAT Final Loss}
\label{sec:theory}

In this section, we show that QUASAR lowers the excess loss of the final
quantized model, $\loss(r_T)-\loss^\star$. This gap is bounded by QUASAR's
loss-aware reconstruction error, so minimizing that error is a natural target. Under the identity STE, the gradient evaluated at the reconstruction $r_t$ is applied to the
latent weights $w_t$, creating the gradient mismatch
$e_t:=\nabla\loss(r_t)-\nabla\loss(w_t)$. Prior work has studied when STE
gradients remain useful descent directions \citep{yin2019understanding}, how
curvature shapes QAT plateaus \citep{winq}, and how a systematic gradient
bias slows SGD \citep{ajalloeian2020convergence}. To the best of our
knowledge, no prior work has studied the relationship between the reconstruction error of the quantizer to the convergence
of QAT or to the loss of the final quantized model. Let $\loss:\mathbb R^d\to\mathbb R$ be the training objective with
$\loss^\star:=\inf_x\loss(x)$, and $T \geq 1$ be the number of training steps. At step $t$,
$w_t, r_t\in\mathbb R^d$ are the latent and reconstructed weights,
$h_t\in\mathbb R_+^d$ the saliencies, and
$\widehat S_t:=\sum_i h_{t,i}(r_{t,i}-w_{t,i})^2$ the objective of
Section~\ref{sec:method} at step $t$. $\mathbb E_t[\cdot]$ denotes
expectation conditional on the training history $\mathcal H_t$, which
determines $h_t$ and $r_t$.

\textbf{Assumptions.} We make four assumptions. \textbf{(1)} The objective
$\loss$ is well behaved: twice continuously differentiable, $L$-smooth, and
bounded below. \textbf{(2)} The STE gradient is the true gradient at the
reconstruction plus zero-mean noise: $G_t=\nabla\loss(r_t)+\xi_t$ with
$\mathbb E_t[\xi_t]=0$ and
$\mathbb E_t\|\xi_t\|^2 \leq M\|\nabla\loss(r_t)\|^2+\sigma^2$ for constants
$M,\sigma^2\geq0$, the biased-SGD noise condition
\citep{ajalloeian2020convergence}; the update is $w_{t+1}=w_t-\eta G_t$.
\textbf{(3)} The reconstruction error controls the gradient mismatch: for
some finite constant $C$ and every $t\in\{0,\ldots,T\}$,
$\left\|\nabla\loss(r_t)-\nabla\loss(w_t)\right\|^2 \leq C\,\widehat S_t$.
Appendix~\ref{app:quasar-certificate} gives a sufficient curvature condition
with $C=L^2/h_{\min}$ under a saliency floor $h_{\min}>0$,
Figure~\ref{fig:eq11-dose} (Appendix~\ref{app:theory-validation}) confirms
the inequality empirically at all tested bit widths, and
Appendix~\ref{app:theory-discussion} discusses the assumptions further.
Appendix~\ref{app:saliency-power} compares the $H$- and $H^2$-motivated
saliency weightings.
\textbf{(4)} The objective satisfies the global Polyak--{\L}ojasiewicz (PL)
inequality $\|\nabla\loss(x)\|^2 \geq 2\mu(\loss(x)-\loss^\star)$ for all
$x\in\mathbb R^d$ and some $\mu>0$ \citep{polyak1963gradient,karimi2016linear};
it is needed only for the final-model bound.

\begin{theorem}[Reconstruction error controls QAT convergence and final loss]
\label{thm:quasar-certificate}
Under Assumptions~1--3, let $0<\eta<\frac{1}{L(1+M)}$. The following statements hold.
\begin{enumerate}[label=(\roman*),leftmargin=*,itemsep=0.15em,topsep=0.2em]
\item \textbf{Convergence.}
\[
 \tfrac1T\sum_{t=0}^{T-1}\mathbb E\big[
   \|\nabla\loss(w_t)\|^2
   +\bigl(1-L\eta(1+M)\bigr)\|\nabla\loss(r_t)\|^2
 \big]
 \leq
 \tfrac{2(\loss(w_0)-\loss^\star)}{\eta T}
 +L\eta\sigma^2
 +\tfrac{C}{T}
   \sum_{t=0}^{T-1}\mathbb E\widehat S_t.
\]
\item \textbf{Final quantized model.} If Assumption~4 also holds and
$\eta<\frac{1}{\mu}$, then, for a random training iterate $\tau$
(Appendix~\ref{app:quasar-certificate}),
\begin{equation}
 \mathbb E[\loss(r_T)-\loss^\star]
 \leq
 \frac{2L}{\mu}(1-\eta\mu)^T
   (\loss(w_0)-\loss^\star)
 +\frac{LC}{\mu^2}\mathbb E\widehat S_\tau
 +\frac{L^2\eta\sigma^2}{\mu^2}
 +\frac{C}{\mu}\mathbb E\widehat S_T.
\label{eq:quasar-final-bound}
\end{equation}
\end{enumerate}
\end{theorem}

Proofs are presented in Appendix~\ref{app:theory}. Theorem~\ref{thm:quasar-certificate} shows that the QAT convergence bound contains the usual stochastic-optimization terms (initialization, minibatch noise) plus a quantization term given by the loss-aware reconstruction error. QUASAR directly minimizes this term at every step, hence tightening the bound on the quantized model's final loss. Appendix~\ref{app:theory-validation} supports the theorem's assumptions (Figures~\ref{fig:eq11-dose} and~\ref{fig:pl-diag}), evaluates the bound (Figure~\ref{fig:bound-decomp}), shows that reconstruction error tracks final loss (Figures~\ref{fig:theory-tie}, \ref{fig:dose-trained}, and~\ref{fig:init-predictor}), and repeats the comparison under SGD (Figure~\ref{fig:sgd-ladder}).

\FloatBarrier

\section{Experiments}\label{sec:heal-distilled}

In this section, we evaluate QUASAR against competitive QAT and PTQ baselines. We assess its ability to preserve quantized model quality using Quantization-aware Distillation (QAD), to adapt quantized models to learn new tasks using QAT, and to generalize to different quantization formats including INT4/INT3/INT2 and NVFP4. Moreover, we compare QUASAR QAT'ed checkpoints against the QAT checkpoints released by Google~\citep{gemma4} under the same format.

\textbf{Setup and baselines.}
For healing, we apply QAD to Qwen3-4B-Thinking-2507~\citep{qwen3} and Llama-3.1-8B-Instruct~\citep{llama3} at INT2, INT3, and INT4 using Open-PerfectBlend data~\citep{openperfectblend} and logits from the full-precision teachers. For adaptation, we train all weights of Qwen3-4B-Base on OpenMathReasoning~\citep{openmathreasoning} using cross-entropy. We also apply QAD in NVFP4 to Qwen3-8B~\citep{qwen3}, Qwen3.5-9B~\citep{qwen35}, and Muse-Glimmer-30B~\citep{hf_muse_glimmer_30b}, and evaluate QUASAR against Google's QAT checkpoint, Gemma 4~\citep{gemma4}, and other independent checkpoints. Baselines in each experiment use the same data, objective, training budget, and tuned learning rate, with matched inference formats. We compare against Standard QAT (Section~\ref{sec:borrowed}), LSQ~\citep{lsq}, Denoising QAT~\citep{baseline}, BitDistiller~\citep{bitdistiller}, RTN, GPTQ~\citep{gptq}, and AWQ~\citep{awq}; for GGUF, we also include imatrix PTQ, which uses calibration activations to estimate weight importance. Full training protocols and hyperparameters are in Appendix~\ref{app:setup}.

\textbf{Evaluation.}
We evaluate quantized checkpoints and report fidelity to the full-precision model together with downstream accuracy. Our benchmark suites cover standard language understanding (MMLU, ARC, HellaSwag, WinoGrande, TruthfulQA, and IFEval)~\citep{mmlu,arc,hellaswag,winogrande,truthfulqa,ifeval}; mathematical and scientific reasoning (MATH-500, GSM8K, AIME, HMMT, MMLU-Pro, SuperGPQA, and GPQA-Diamond)~\citep{math,letsverify,gsm8k,aime,hmmt,mmlupro,supergpqa,gpqa}; coding (LiveCodeBench and OJBench)~\citep{livecodebench,ojbench}; and long-context performance (LongBench-v2, RULER, MRCR, and LongProc)~\citep{longbenchv2,ruler,openaimrcr,longproc}.

\subsection{Results on Healing and Adaptation}\label{sec:exp-healing}

\begin{figure}[t]
\centering
\includegraphics[width=\linewidth]{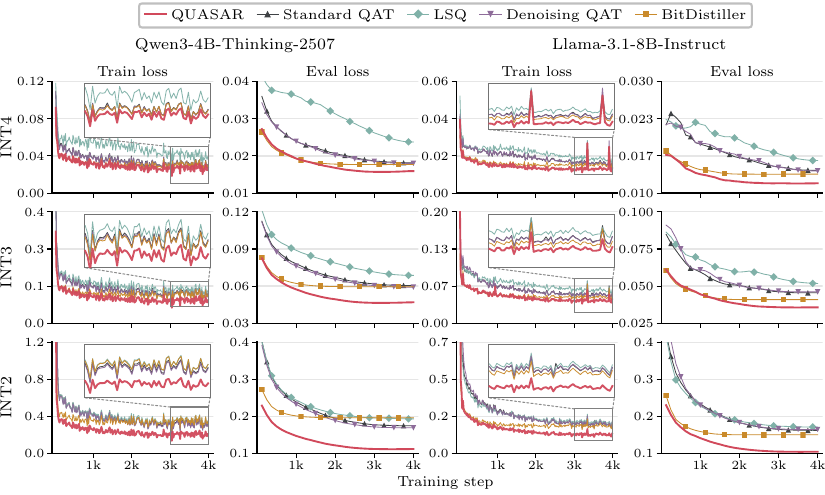}
\caption{QAD of Qwen3-4B-Thinking-2507 (left) and Llama-3.1-8B-Instruct (right) at INT4, INT3, and INT2 on Open-PerfectBlend. We show training loss over the full run, the final 1{,}000 steps, and held-out evaluation loss. All losses are forward KL to the full-precision model.}
\label{fig:heald-curves}
\label{fig:heald-curves-llama}
\end{figure}

\begin{table}[t]\centering\fontsize{6}{7}\selectfont
\renewcommand{\arraystretch}{0.74}
\setlength{\aboverulesep}{0.2ex}\setlength{\belowrulesep}{0.3ex}
\caption{Reasoning, coding, and long-context results for Qwen3-4B-Thinking-2507 and Llama-3.1-8B-Instruct at INT2, INT3, and INT4. We report KL and top-1 agreement between quantized checkpoints and the BF16 original, as well as benchmark accuracy.}
\label{tab:heald-hard}
\newlength{\hcolw}\setlength{\hcolw}{0.035\textwidth}
\newlength{\hklw}\setlength{\hklw}{0.044\textwidth}
\setlength{\tabcolsep}{1.1pt}
\begin{tabular}{l >{\columncolor{black!8}\raggedleft\arraybackslash}p{\hklw} >{\columncolor{black!8}\raggedleft\arraybackslash}p{\hcolw} >{\raggedleft\arraybackslash}p{\hcolw} >{\raggedleft\arraybackslash}p{\hcolw} >{\raggedleft\arraybackslash}p{\hcolw} >{\raggedleft\arraybackslash}p{\hcolw} >{\raggedleft\arraybackslash}p{\hcolw} >{\raggedleft\arraybackslash}p{\hcolw} >{\raggedleft\arraybackslash}p{\hcolw} >{\raggedleft\arraybackslash}p{\hcolw} >{\columncolor{black!8}\raggedleft\arraybackslash}p{\hcolw} !{\hspace{1.5pt}{\color{black!25}\vrule}\hspace{1.5pt}} >{\columncolor{black!8}\raggedleft\arraybackslash}p{\hklw} >{\columncolor{black!8}\raggedleft\arraybackslash}p{\hcolw} >{\raggedleft\arraybackslash}p{\hcolw} >{\raggedleft\arraybackslash}p{\hcolw} >{\raggedleft\arraybackslash}p{\hcolw} >{\raggedleft\arraybackslash}p{\hcolw} >{\raggedleft\arraybackslash}p{\hcolw} >{\raggedleft\arraybackslash}p{\hcolw} >{\columncolor{black!8}\raggedleft\arraybackslash}p{\hcolw}}
\toprule
& \multicolumn{11}{c}{Qwen3-4B-Thinking-2507} & \multicolumn{9}{c}{Llama-3.1-8B-Instruct} \\
\cmidrule(lr){2-12}\cmidrule(lr){13-21}
Method & \multicolumn{1}{>{\columncolor{black!8}}c}{\rotatebox{90}{KL $\downarrow$}} & \multicolumn{1}{>{\columncolor{black!8}}c}{\rotatebox{90}{Top-1 $\uparrow$}} & \multicolumn{1}{c}{\rotatebox{90}{HMMT '26}} & \multicolumn{1}{c}{\rotatebox{90}{AIME '25}} & \multicolumn{1}{c}{\rotatebox{90}{MATH-500}} & \multicolumn{1}{c}{\rotatebox{90}{MMLU-Pro}} & \multicolumn{1}{c}{\rotatebox{90}{SuperGPQA}} & \multicolumn{1}{c}{\rotatebox{90}{LCB v6}} & \multicolumn{1}{c}{\rotatebox{90}{LongBench-v2}} & \multicolumn{1}{c}{\rotatebox{90}{RULER}} & \multicolumn{1}{>{\columncolor{black!8}}c}{\rotatebox{90}{\textbf{Avg.}}} & \multicolumn{1}{>{\columncolor{black!8}}c}{\rotatebox{90}{KL $\downarrow$}} & \multicolumn{1}{>{\columncolor{black!8}}c}{\rotatebox{90}{Top-1 $\uparrow$}} & \multicolumn{1}{c}{\rotatebox{90}{MATH-500}} & \multicolumn{1}{c}{\rotatebox{90}{MMLU-Pro}} & \multicolumn{1}{c}{\rotatebox{90}{SuperGPQA}} & \multicolumn{1}{c}{\rotatebox{90}{LCB v6}} & \multicolumn{1}{c}{\rotatebox{90}{LongBench-v2}} & \multicolumn{1}{c}{\rotatebox{90}{RULER}} & \multicolumn{1}{>{\columncolor{black!8}}c}{\rotatebox{90}{\textbf{Avg.}}} \\
\midrule
Teacher (BF16) & -- & -- & 43.1 & 77.9 & 98.2 & 72.7 & 48.1 & 74.6 & 44.5 & 87.9 & 68.4 & -- & -- & 49.6 & 42.3 & 22.4 & 17.0 & 30.0 & 89.6 & 41.8 \\
\midrule
\multicolumn{21}{l}{\textit{2-bit (INT2)}}\\
RTN & 11.166 & 0.5 & 0.0 & 0.0 & 1.9 & 0.0 & 10.0 & 0.0 & 0.0 & 0.0 & 1.5 & 10.810 & 0.5 & 0.0 & 0.0 & 10.2 & 0.0 & 0.0 & 0.0 & 1.7 \\
GPTQ & 1.140 & 67.7 & 0.0 & 0.0 & 2.6 & 3.2 & 8.4 & 0.0 & 0.0 & 0.3 & 1.8 & 2.380 & 51.8 & 2.9 & 1.8 & 8.7 & 0.0 & 0.0 & 0.1 & 2.3 \\
AWQ & 3.016 & 37.0 & 0.0 & 0.0 & 2.4 & 0.0 & 5.2 & 0.0 & 0.0 & 0.0 & 1.0 & 6.769 & 12.9 & 1.6 & 0.0 & 10.7 & 0.0 & 0.0 & 0.0 & 2.0 \\
\arrayrulecolor{black!25}\cmidrule[0.4pt](lr){1-21}\arrayrulecolor{black}
Standard QAT & 0.186 & 89.9 & \underline{1.0} & \underline{1.7} & 49.5 & 26.6 & 17.4 & 8.2 & 20.5 & 39.3 & 20.5 & 0.163 & 90.1 & 21.6 & 16.8 & \underline{11.8} & 4.0 & 15.1 & 5.5 & 12.5 \\
LSQ & 0.205 & 89.2 & 0.8 & 0.5 & 47.2 & 26.2 & 17.8 & 7.4 & \underline{20.7} & 41.7 & 20.3 & 0.172 & 89.6 & 19.9 & 15.7 & 11.6 & 2.3 & 15.3 & 52.9 & 19.6 \\
Denoising QAT & \underline{0.179} & \underline{90.1} & \underline{1.0} & 1.4 & \underline{51.8} & 27.4 & \underline{19.0} & 9.2 & 18.9 & 39.4 & 21.0 & 0.165 & 90.1 & 20.8 & 16.5 & 11.0 & 3.6 & 8.5 & 0.0 & 10.1 \\
BitDistiller & 0.209 & 89.1 & 0.6 & 1.1 & 51.4 & \underline{29.3} & 16.8 & \underline{9.4} & 18.7 & \underline{44.9} & \underline{21.5} & \underline{0.151} & \underline{90.2} & \underline{21.7} & \underline{19.0} & \underline{11.8} & \underline{5.5} & \textbf{20.5} & \underline{63.5} & \underline{23.7} \\
\rowcolor{oursrow}
\textbf{QUASAR}~(ours) & \textbf{0.126} & \textbf{92.0} & \textbf{6.2} & \textbf{14.4} & \textbf{77.4} & \textbf{45.8} & \textbf{25.4} & \textbf{23.4} & \textbf{26.6} & \textbf{59.6} & \textbf{34.8} & \textbf{0.105} & \textbf{92.0} & \textbf{30.0} & \textbf{22.6} & \textbf{13.7} & \textbf{7.0} & \underline{15.9} & \textbf{66.4} & \textbf{25.9} \\
\midrule
\multicolumn{21}{l}{\textit{3-bit (INT3)}}\\
RTN & 0.597 & 78.8 & 0.0 & 0.7 & 13.0 & 16.3 & 15.6 & 0.0 & 8.7 & 49.3 & 13.0 & 0.276 & 85.6 & 8.2 & 17.2 & 12.5 & 0.7 & 21.1 & 49.0 & 18.1 \\
GPTQ & 0.130 & 91.5 & \underline{31.2} & 49.4 & \textbf{95.8} & 63.7 & 39.5 & 49.4 & 31.0 & 73.3 & 54.1 & 0.076 & 92.8 & 38.6 & 31.9 & 18.4 & 10.4 & 27.6 & 80.9 & 34.6 \\
AWQ & 0.240 & 88.1 & 22.7 & 31.0 & 90.6 & 61.5 & 37.0 & 35.7 & 30.2 & 76.0 & 48.1 & 0.197 & 88.0 & 16.2 & 15.3 & 13.9 & 2.4 & 22.7 & 69.0 & 23.3 \\
\arrayrulecolor{black!25}\cmidrule[0.4pt](lr){1-21}\arrayrulecolor{black}
Standard QAT & 0.062 & 94.2 & 29.9 & 47.4 & 94.2 & 64.9 & 40.0 & 53.3 & 29.6 & 80.8 & 55.0 & 0.046 & 94.7 & 39.5 & 34.8 & 18.6 & 12.7 & \underline{30.2} & 81.7 & 36.3 \\
LSQ & 0.070 & 93.8 & 28.3 & 45.5 & 93.3 & 64.1 & 38.2 & 52.2 & 32.8 & 80.5 & 54.4 & 0.052 & 94.2 & 39.3 & 35.0 & 17.1 & \underline{13.7} & 27.2 & \textbf{86.1} & 36.4 \\
Denoising QAT & \underline{0.060} & \underline{94.3} & 29.3 & 49.3 & 94.7 & 64.5 & 39.7 & 52.3 & 32.8 & 79.0 & 55.2 & 0.046 & 94.7 & 41.2 & 34.6 & 18.2 & 13.0 & 29.2 & 80.3 & 36.1 \\
BitDistiller & 0.065 & 94.2 & 29.7 & \underline{52.5} & 95.1 & \underline{65.8} & \underline{40.6} & \underline{57.7} & \underline{33.0} & \underline{83.5} & \underline{57.3} & \underline{0.041} & \underline{94.8} & \underline{42.5} & \underline{35.9} & \underline{19.3} & 13.1 & \textbf{31.2} & 84.4 & \underline{37.7} \\
\rowcolor{oursrow}
\textbf{QUASAR}~(ours) & \textbf{0.054} & \textbf{94.9} & \textbf{31.7} & \textbf{60.5} & \underline{95.2} & \textbf{66.7} & \textbf{41.5} & \textbf{60.2} & \textbf{35.8} & \textbf{84.4} & \textbf{59.5} & \textbf{0.036} & \textbf{95.2} & \textbf{43.4} & \textbf{36.4} & \textbf{19.7} & \textbf{16.2} & 26.0 & \underline{85.8} & \textbf{37.9} \\
\midrule
\multicolumn{21}{l}{\textit{4-bit (INT4)}}\\
RTN & 0.108 & 92.2 & 35.7 & 62.1 & 96.8 & 69.7 & 44.4 & 65.2 & 38.8 & 83.3 & 62.0 & 0.044 & 94.4 & 44.5 & 35.5 & 20.1 & 14.2 & 29.6 & \textbf{89.0} & 38.8 \\
GPTQ & 0.027 & 96.2 & \underline{38.8} & \textbf{74.8} & \textbf{97.8} & 71.4 & \textbf{46.5} & 70.6 & 38.2 & \underline{87.4} & 65.7 & 0.015 & 96.8 & 47.9 & 40.0 & 20.6 & \textbf{16.5} & \textbf{31.2} & 88.5 & 40.8 \\
AWQ & 0.050 & 94.7 & 36.6 & 70.6 & 97.6 & \underline{71.6} & 45.7 & 70.8 & 39.0 & 86.8 & 64.8 & 0.032 & 95.2 & 44.8 & 39.8 & \underline{21.3} & 13.2 & 30.8 & 88.1 & 39.7 \\
\arrayrulecolor{black!25}\cmidrule[0.4pt](lr){1-21}\arrayrulecolor{black}
Standard QAT & 0.019 & \underline{96.9} & 36.8 & 72.1 & \underline{97.7} & 71.4 & \textbf{46.5} & 70.9 & 39.0 & 87.1 & 65.2 & 0.014 & \underline{97.0} & 47.7 & \underline{42.1} & 21.1 & 15.8 & \underline{31.0} & \textbf{89.0} & \textbf{41.1} \\
LSQ & 0.024 & 96.4 & \textbf{39.3} & 68.5 & 97.5 & 70.8 & 45.6 & 69.3 & 36.4 & 87.3 & 64.3 & 0.016 & 96.8 & \underline{48.3} & 41.9 & 21.2 & 15.5 & 29.8 & 87.5 & 40.7 \\
Denoising QAT & \underline{0.018} & \underline{96.9} & \underline{38.8} & 71.5 & 97.2 & 70.9 & \underline{46.4} & \underline{71.6} & \textbf{42.9} & 86.9 & \underline{65.8} & 0.014 & \underline{97.0} & 48.1 & \underline{42.1} & \textbf{21.7} & 14.3 & \textbf{31.2} & \underline{88.9} & \textbf{41.1} \\
BitDistiller & \underline{0.018} & \underline{96.9} & 38.5 & \underline{73.9} & 97.3 & 71.4 & \textbf{46.5} & 71.1 & 40.4 & 86.8 & 65.7 & \underline{0.013} & \underline{97.0} & \textbf{48.4} & 41.0 & 21.0 & 15.7 & \underline{31.0} & \underline{88.9} & \underline{41.0} \\
\rowcolor{oursrow}
\textbf{QUASAR}~(ours) & \textbf{0.016} & \textbf{97.1} & 38.5 & 73.8 & 97.5 & \textbf{71.7} & 46.2 & \textbf{72.2} & \underline{41.0} & \textbf{88.0} & \textbf{66.1} & \textbf{0.012} & \textbf{97.3} & 47.6 & \textbf{42.2} & 20.8 & \underline{16.4} & 30.0 & \textbf{89.0} & \underline{41.0} \\
\bottomrule
\end{tabular}
\end{table}

\textbf{Healing with QAD.} Figure~\ref{fig:heald-curves} presents the training and evaluation loss of QUASAR and baseline QAT methods, and Table~\ref{tab:heald-hard} presents the downstream accuracy of the final quantized checkpoints of QAT and PTQ. Across both models and all bit widths, QUASAR reaches the lowest training and evaluation loss (KL). At INT2, it reduces KL by about 30\% over the
strongest QAT baseline. On reasoning, coding, and long-context tasks, QUASAR improves average
accuracy by 13.3 points for Qwen and 2.2 points for Llama
(Table~\ref{tab:heald-hard}). Additional results and
analyses are in Appendix~\ref{app:healing}.

\phantomsection\label{sec:exp-adaptation}\label{sec:adaptation}
\textbf{Adaptation with QAT.} Table~\ref{tab:adapt-qwen} presents the downstream accuracy of INT4/INT3/INT2 Qwen3-4B-Base adapted using different QAT methods, and Figure~\ref{fig:adapt-curves} in the Appendix presents the training and evaluation loss curves. QUASAR has the lowest held-out
perplexity and the best or tied-best average accuracy at every bit width. At
INT2, QUASAR reaches 29.6 average accuracy, 10.9 points above the strongest QAT
baseline and 29.0 points above full-precision fine-tuning followed by PTQ
(Table~\ref{tab:adapt-qwen}). It is also the only INT2 quantized model to outperform
the original base model and score above zero on HMMT'25.
Additional results are in Appendix~\ref{app:adaptation}.

\begin{table}[t]\centering\fontsize{6}{7}\selectfont
\renewcommand{\arraystretch}{0.74}
\setlength{\aboverulesep}{0.2ex}\setlength{\belowrulesep}{0.3ex}
\caption{Comparison of QAT and PTQ methods for adapting Qwen3-4B-Base on OpenMathReasoning at INT2, INT3, and INT4. We report held-out perplexity and accuracy on five math benchmarks.}
\label{tab:adapt-qwen}
\newlength{\adaptcolw}\setlength{\adaptcolw}{0.032\textwidth}\newlength{\pplw}\setlength{\pplw}{0.046\textwidth}\newlength{\pplwtwo}\setlength{\pplwtwo}{0.068\textwidth}
\setlength{\tabcolsep}{0.9pt}
\begin{tabular}{l >{\columncolor{black!8}\raggedleft\arraybackslash}p{\pplwtwo} >{\raggedleft\arraybackslash}p{\adaptcolw} >{\raggedleft\arraybackslash}p{\adaptcolw} >{\raggedleft\arraybackslash}p{\adaptcolw} >{\raggedleft\arraybackslash}p{\adaptcolw} >{\raggedleft\arraybackslash}p{\adaptcolw} >{\columncolor{black!8}\raggedleft\arraybackslash}p{\adaptcolw} !{\hspace{1pt}{\color{black!25}\vrule}\hspace{1pt}} >{\columncolor{black!8}\raggedleft\arraybackslash}p{\pplw} >{\raggedleft\arraybackslash}p{\adaptcolw} >{\raggedleft\arraybackslash}p{\adaptcolw} >{\raggedleft\arraybackslash}p{\adaptcolw} >{\raggedleft\arraybackslash}p{\adaptcolw} >{\raggedleft\arraybackslash}p{\adaptcolw} >{\columncolor{black!8}\raggedleft\arraybackslash}p{\adaptcolw} !{\hspace{1pt}{\color{black!25}\vrule}\hspace{1pt}} >{\columncolor{black!8}\raggedleft\arraybackslash}p{\pplw} >{\raggedleft\arraybackslash}p{\adaptcolw} >{\raggedleft\arraybackslash}p{\adaptcolw} >{\raggedleft\arraybackslash}p{\adaptcolw} >{\raggedleft\arraybackslash}p{\adaptcolw} >{\raggedleft\arraybackslash}p{\adaptcolw} >{\columncolor{black!8}\raggedleft\arraybackslash}p{\adaptcolw}}
\toprule
& \multicolumn{7}{c}{2-bit (INT2)} & \multicolumn{7}{c}{3-bit (INT3)} & \multicolumn{7}{c}{4-bit (INT4)} \\
\cmidrule(lr){2-8}\cmidrule(lr){9-15}\cmidrule(lr){16-22}
Method & \multicolumn{1}{>{\columncolor{black!8}}c}{\rotatebox{90}{PPL $\downarrow$}} & \multicolumn{1}{c}{\rotatebox{90}{MATH-500}} & \multicolumn{1}{c}{\rotatebox{90}{GSM8K}} & \multicolumn{1}{c}{\rotatebox{90}{AIME'24}} & \multicolumn{1}{c}{\rotatebox{90}{AIME'25}} & \multicolumn{1}{c}{\rotatebox{90}{HMMT'25}} & \multicolumn{1}{>{\columncolor{black!8}}c}{\rotatebox{90}{\textbf{Avg.}}} & \multicolumn{1}{>{\columncolor{black!8}}c}{\rotatebox{90}{PPL $\downarrow$}} & \multicolumn{1}{c}{\rotatebox{90}{MATH-500}} & \multicolumn{1}{c}{\rotatebox{90}{GSM8K}} & \multicolumn{1}{c}{\rotatebox{90}{AIME'24}} & \multicolumn{1}{c}{\rotatebox{90}{AIME'25}} & \multicolumn{1}{c}{\rotatebox{90}{HMMT'25}} & \multicolumn{1}{>{\columncolor{black!8}}c}{\rotatebox{90}{\textbf{Avg.}}} & \multicolumn{1}{>{\columncolor{black!8}}c}{\rotatebox{90}{PPL $\downarrow$}} & \multicolumn{1}{c}{\rotatebox{90}{MATH-500}} & \multicolumn{1}{c}{\rotatebox{90}{GSM8K}} & \multicolumn{1}{c}{\rotatebox{90}{AIME'24}} & \multicolumn{1}{c}{\rotatebox{90}{AIME'25}} & \multicolumn{1}{c}{\rotatebox{90}{HMMT'25}} & \multicolumn{1}{>{\columncolor{black!8}}c}{\rotatebox{90}{\textbf{Avg.}}} \\
\midrule
Base (no SFT) & -- & 41.0 & 66.4 & 9.6 & 3.3 & 0.8 & 24.2 & -- & 41.0 & 66.4 & 9.6 & 3.3 & 0.8 & 24.2 & -- & 41.0 & 66.4 & 9.6 & 3.3 & 0.8 & 24.2 \\
FP-SFT & 1.59 & 83.8 & 91.3 & 22.9 & 22.5 & 10.0 & 46.1 & 1.59 & 83.8 & 91.3 & 22.9 & 22.5 & 10.0 & 46.1 & 1.59 & 83.8 & 91.3 & 22.9 & 22.5 & 10.0 & 46.1 \\
\midrule
RTN & $3.5{\times}10^{5}$ & 0.9 & 0.3 & 0.0 & 0.0 & \underline{0.0} & 0.2 & 2.24 & 12.8 & 7.6 & 0.0 & 1.7 & 0.0 & 4.4 & 1.66 & 78.2 & 90.0 & 15.0 & 13.3 & 7.5 & 40.8 \\
GPTQ & 3.12 & 2.1 & 1.1 & 0.0 & 0.0 & \underline{0.0} & 0.6 & 1.68 & 71.0 & 85.0 & 11.2 & 11.2 & 4.2 & 36.5 & 1.61 & 81.7 & 90.6 & 19.1 & 20.5 & 8.5 & 44.1 \\
AWQ & 130 & 1.5 & 1.2 & 0.0 & 0.0 & \underline{0.0} & 0.5 & 1.83 & 57.3 & 65.1 & 10.0 & 8.3 & 1.7 & 28.5 & 1.63 & 80.3 & 90.5 & 15.8 & 17.5 & 8.3 & 42.5 \\
\arrayrulecolor{black!25}\cmidrule[0.4pt](lr){1-22}\arrayrulecolor{black}
Standard QAT & 1.79 & 38.7 & 51.9 & \underline{2.1} & \underline{0.8} & \underline{0.0} & \underline{18.7} & 1.64 & 75.8 & 87.5 & 15.0 & \textbf{16.2} & 4.6 & 39.8 & 1.60 & 82.1 & \underline{90.8} & \textbf{21.2} & 20.6 & \underline{10.7} & \underline{45.1} \\
LSQ & 1.82 & 36.5 & 49.9 & 0.8 & 0.4 & \underline{0.0} & 17.5 & 1.65 & 73.7 & 87.0 & 12.5 & 11.7 & 3.8 & 37.7 & 1.62 & 79.0 & 89.1 & 17.9 & 17.1 & 6.7 & 42.0 \\
Denoising QAT & \underline{1.78} & \underline{39.2} & 52.5 & 0.8 & \underline{0.8} & \underline{0.0} & 18.7 & \underline{1.64} & \underline{75.9} & 87.3 & 14.2 & 13.3 & \textbf{7.1} & 39.6 & \underline{1.60} & 82.1 & 90.4 & 20.0 & \textbf{21.4} & 9.9 & 44.8 \\
BitDistiller & 1.87 & 33.6 & \underline{57.0} & 1.2 & 0.4 & \underline{0.0} & 18.4 & 1.64 & 75.7 & \underline{87.8} & \textbf{16.2} & 15.4 & 5.0 & \underline{40.0} & 1.60 & \underline{82.7} & 90.6 & 20.3 & \underline{20.8} & 9.3 & 44.7 \\
\rowcolor{oursrow}
\textbf{QUASAR}~(ours) & \textbf{1.71} & \textbf{60.7} & \textbf{75.4} & \textbf{4.0} & \textbf{6.5} & \textbf{1.5} & \textbf{29.6} & \textbf{1.62} & \textbf{78.5} & \textbf{88.6} & \underline{15.9} & \textbf{16.2} & \underline{6.0} & \textbf{41.1} & \textbf{1.60} & \textbf{82.8} & \textbf{90.9} & \underline{20.8} & 19.8 & \textbf{11.2} & \textbf{45.1} \\
\bottomrule
\end{tabular}
\end{table}

\subsection{NVFP4 and Checkpoint Comparisons}\label{sec:exp-nvfp4}\label{sec:heal-nvfp4}

\textbf{NVFP4 QAD.} Table~\ref{tab:merged} reports downstream accuracy for Qwen3-8B, Qwen3.5-9B, and Muse-Glimmer-30B in NVFP4, all trained with the same QAD healing recipe as in Section~\ref{sec:exp-healing}. Figure~\ref{fig:nvfp4-curves} in Appendix~\ref{app:nvfp4} shows the corresponding training and evaluation loss curves for the two Qwen models. All checkpoints are evaluated using vLLM's native NVFP4 path~\citep{pagedattention,vllm-modelopt-docs}. For Qwen3-8B and Qwen3.5-9B, QUASAR reduces held-out KL divergence by 12\% and 15\% relative to Standard QAT, respectively. On the more challenging reasoning, coding, and long-horizon benchmarks, Qwen3.5-9B achieves an average score of 55.7, compared with 52.6 for Standard QAT and 39.4 for GPTQ (Table~\ref{tab:merged}). These results also highlight the value of QAD: both QAD checkpoints remain stable over long generations, whereas the PTQ checkpoints often fail to complete them, despite GPTQ achieving lower held-out KL divergence than Standard QAT. Full results are provided in Appendix~\ref{app:nvfp4}.

\begin{table}[t]\centering\fontsize{6}{7}\selectfont
\renewcommand{\arraystretch}{0.74}\setlength{\aboverulesep}{0.2ex}\setlength{\belowrulesep}{0.3ex}
\caption{Comparison of QUASAR with QAT and PTQ baselines across NVFP4, INT4, and Q4\_0 (llama.cpp GGUF format). We report fidelity to BF16 and accuracy on ten reasoning, coding, and long-horizon benchmarks.}
\label{tab:merged}
\ifdefined\rcw\else\newlength{\rcw}\fi\setlength{\rcw}{0.043\textwidth}\ifdefined\rklw\else\newlength{\rklw}\fi\setlength{\rklw}{0.043\textwidth}\setlength{\tabcolsep}{0.8pt}
\begin{tabular}{l l l r >{\columncolor{black!8}\raggedleft\arraybackslash}p{\rklw} >{\columncolor{black!8}\raggedleft\arraybackslash}p{\rcw} >{\raggedleft\arraybackslash}p{\rcw} >{\raggedleft\arraybackslash}p{\rcw} >{\raggedleft\arraybackslash}p{\rcw} >{\raggedleft\arraybackslash}p{\rcw} >{\raggedleft\arraybackslash}p{\rcw} >{\raggedleft\arraybackslash}p{\rcw} >{\raggedleft\arraybackslash}p{\rcw} >{\raggedleft\arraybackslash}p{\rcw} >{\raggedleft\arraybackslash}p{\rcw} >{\raggedleft\arraybackslash}p{\rcw} >{\columncolor{black!8}\raggedleft\arraybackslash}p{\rcw}}
\toprule
& & & & \multicolumn{2}{>{\columncolor{black!8}}c}{Fidelity} & \multicolumn{11}{c}{Accuracy (\%) $\uparrow$} \\
\cmidrule(lr){5-6}\cmidrule(l){7-17}
& & & & & & \multicolumn{4}{c}{Hard reasoning and coding} & \multicolumn{6}{c}{Long horizon} & \\
\cmidrule(lr){7-10}\cmidrule(lr){11-16}
Model & Method & Format & \multicolumn{1}{c}{bpw} & \multicolumn{1}{>{\columncolor{black!8}}c}{\rotatebox{90}{KL $\downarrow$}} & \multicolumn{1}{>{\columncolor{black!8}}c}{\rotatebox{90}{Top-1 $\uparrow$}} & \multicolumn{1}{c}{\rotatebox{90}{HMMT'26}} & \multicolumn{1}{c}{\rotatebox{90}{AIME'25}} & \multicolumn{1}{c}{\rotatebox{90}{SuperGPQA}} & \multicolumn{1}{c}{\rotatebox{90}{OJBench}} & \multicolumn{1}{c}{\rotatebox{90}{MRCR}} & \multicolumn{1}{c}{\rotatebox{90}{MRCR $\ge$64k}} & \multicolumn{1}{c}{\rotatebox{90}{MRCR 8-needle}} & \multicolumn{1}{c}{\rotatebox{90}{LongProc}} & \multicolumn{1}{c}{\rotatebox{90}{LongBench-v2}} & \multicolumn{1}{c}{\rotatebox{90}{RULER}} & \multicolumn{1}{>{\columncolor{black!8}}c}{\rotatebox{90}{\textbf{Avg.}}} \\
\midrule
 & BF16 & BF16 & 16 & -- & -- & 44.6 & 68.4 & 48.8 & 22.4 & 25.8 & 19.2 & 16.2 & 48.9 & 38.4 & 80.5 & 41.3 \\
\arrayrulecolor{black!25}\cmidrule[0.4pt](lr){2-17}\arrayrulecolor{black}
 & RTN & NVFP4 W4A4 & 4.50 & .053 & 93.0 & 37.4 & \textbf{64.6} & 45.6 & 14.7 & 22.4 & \underline{15.2} & 14.1 & 39.3 & \underline{36.0} & 76.6 & 36.6 \\
 & GPTQ & NVFP4 W4A4 & 4.50 & \underline{.039} & \underline{94.0} & \underline{40.3} & \underline{64.0} & \underline{46.2} & 16.8 & \underline{23.6} & 14.6 & \underline{15.2} & \underline{39.5} & \underline{36.0} & \underline{78.5} & \underline{37.5} \\
 & Standard QAT & NVFP4 W4A4 & 4.50 & .041 & \underline{94.0} & 39.0 & 61.2 & 45.7 & \textbf{19.4} & 21.6 & 14.1 & 12.9 & 36.5 & \textbf{37.4} & 77.9 & 36.6 \\
\rowcolor{oursrow}
\cellcolor{white}\multirow{-5}{*}{Qwen3-8B} & \textbf{QUASAR}~(ours) & NVFP4 W4A4 & 4.50 & \cellcolor{oursrow}\textbf{.036} & \cellcolor{oursrow}\textbf{94.3} & \textbf{40.8} & 63.6 & \textbf{46.8} & \underline{18.5} & \textbf{24.3} & \textbf{16.4} & \textbf{15.8} & \textbf{40.3} & 33.8 & \textbf{78.9} & \cellcolor{oursrow}\textbf{37.9} \\
\midrule
 & BF16 & BF16 & 16 & -- & -- & 47.9 & 74.5 & 58.8 & 19.4 & 79.4 & 74.5 & 53.9 & 66.8 & 43.7 & 90.1 & 60.9 \\
\arrayrulecolor{black!25}\cmidrule[0.4pt](lr){2-17}\arrayrulecolor{black}
 & RTN & NVFP4 W4A4 & 4.50 & .044 & 93.2 & 21.9 & 37.3 & 49.1 & \underline{6.5} & 54.4 & 42.6 & 32.1 & 35.5 & 23.1 & \textbf{91.4} & 39.4 \\
 & GPTQ & NVFP4 W4A4 & 4.50 & \underline{.029} & \underline{94.5} & 20.0 & 29.9 & 48.6 & 3.4 & 56.8 & 44.6 & 32.6 & 42.1 & 24.3 & \underline{91.3} & 39.4 \\
 & Standard QAT & NVFP4 W4A4 & 4.50 & .033 & 94.2 & \underline{43.1} & \underline{65.3} & \underline{55.5} & \textbf{13.4} & \underline{63.7} & \underline{52.3} & \underline{39.7} & \underline{59.0} & \underline{43.9} & 89.8 & \underline{52.6} \\
\rowcolor{oursrow}
\cellcolor{white}\multirow{-5}{*}{Qwen3.5-9B} & \textbf{QUASAR}~(ours) & NVFP4 W4A4 & 4.50 & \cellcolor{oursrow}\textbf{.028} & \cellcolor{oursrow}\textbf{94.6} & \textbf{43.8} & \textbf{68.8} & \textbf{56.0} & \textbf{13.4} & \textbf{71.7} & \textbf{63.8} & \textbf{45.0} & \textbf{59.7} & \textbf{44.1} & 90.5 & \cellcolor{oursrow}\textbf{55.7} \\
\midrule[\heavyrulewidth]
 & BF16 & BF16 & 16 & -- & -- & 71.7 & 92.2 & 60.1 & 25.0 & 46.1 & 38.0 & 24.9 & 67.0 & 60.4 & 88.4 & 57.4 \\
\arrayrulecolor{black!25}\cmidrule[0.4pt](lr){2-17}\arrayrulecolor{black}
 & RTN & NVFP4 W4A16 & 4.50 & .054 & \underline{93.2} & \underline{69.5} & 87.7 & \underline{58.9} & 25.0 & \underline{44.3} & 35.6 & \underline{26.2} & 55.7 & \textbf{59.4} & 88.3 & 55.1 \\
 & GPTQ & NVFP4 W4A4 & 4.50 & .066 & 92.4 & 66.4 & \underline{90.3} & 58.5 & 25.4 & 43.1 & \textbf{40.7} & 25.5 & 60.1 & 57.3 & 84.7 & 55.2 \\
\rowcolor{oursrow}
\cellcolor{white} & \textbf{QUASAR}~(ours) & NVFP4 W4A16 & 4.50 & \cellcolor{oursrow}\textbf{.018} & \cellcolor{oursrow}\textbf{96.1} & \textbf{72.3} & \textbf{90.8} & \textbf{59.3} & \textbf{27.2} & \textbf{48.5} & \underline{40.0} & \textbf{28.5} & \textbf{67.3} & \underline{58.4} & \textbf{90.2} & \cellcolor{oursrow}\textbf{58.2} \\
\rowcolor{oursrow}
\cellcolor{white}\multirow{-5}{*}{Muse-Glimmer-30B} & \textbf{QUASAR}~(ours) & NVFP4 W4A4 & 4.50 & \cellcolor{oursrow}\underline{.053} & \cellcolor{oursrow}\underline{93.2} & 68.3 & 88.0 & 58.9 & \underline{26.3} & 44.1 & 35.4 & 25.5 & \underline{62.8} & 54.5 & \underline{89.7} & \cellcolor{oursrow}\underline{55.3} \\
\midrule
 & BF16 & BF16 & 16 & -- & -- & 32.8 & 40.9 & 42.1 & 21.6 & 26.7 & 15.4 & 12.6 & 56.7 & 42.9 & 89.7 & 38.1 \\
\arrayrulecolor{black!25}\cmidrule[0.4pt](lr){2-17}\arrayrulecolor{black}
 & RTN & INT4 g64 & 4.25 & .079 & 91.5 & 24.4 & 31.8 & 39.0 & 16.4 & \underline{22.0} & \underline{9.3} & \underline{11.8} & 35.1 & \textbf{41.7} & 87.5 & 31.9 \\
 & Google QAT & INT4 g32 & 4.50 & \underline{.064} & \underline{92.6} & \underline{28.6} & \underline{33.3} & \underline{39.7} & \underline{17.7} & 21.0 & 8.0 & 11.5 & \underline{51.1} & 39.8 & \textbf{89.6} & \underline{34.0} \\
\rowcolor{oursrow}
\cellcolor{white} & \textbf{QUASAR}~(ours) & INT4 g64 & 4.25 & \cellcolor{oursrow}\textbf{.022} & \cellcolor{oursrow}\textbf{95.5} & \textbf{29.0} & \textbf{37.5} & \textbf{40.5} & \textbf{18.1} & \textbf{24.4} & \textbf{12.0} & \textbf{12.7} & \textbf{53.6} & \underline{40.4} & \underline{89.5} & \cellcolor{oursrow}\textbf{35.8} \\
\arrayrulecolor{black!25}\cmidrule[0.4pt](lr){2-17}\arrayrulecolor{black}
 & RTN & GGUF Q4\_0 & 4.50 & .194 & 85.7 & 18.8 & 32.8 & 36.5 & 14.2 & 22.9 & 9.1 & 11.0 & 39.0 & 38.0 & 83.8 & 30.6 \\
 & Google QAT & GGUF Q4\_0 & 4.50 & .088 & 90.6 & \textbf{29.0} & 35.9 & \underline{40.8} & \textbf{19.4} & 22.9 & 9.4 & \underline{12.8} & \textbf{57.8} & \underline{39.0} & \underline{89.3} & \underline{35.6} \\
 & imatrix (PTQ) & GGUF Q4\_0 & 4.52 & \underline{.067} & \underline{91.1} & \underline{28.3} & \textbf{40.1} & \textbf{40.9} & \underline{19.0} & \textbf{24.7} & \underline{11.7} & \textbf{13.0} & 46.2 & \underline{39.0} & 88.9 & 35.2 \\
\rowcolor{oursrow}
\cellcolor{white}\multirow{-8}{*}{Gemma-4 E4B} & \textbf{QUASAR}~(ours) & GGUF Q4\_0 & 4.50 & \cellcolor{oursrow}\textbf{.044} & \cellcolor{oursrow}\textbf{92.8} & 27.8 & \underline{37.6} & 39.6 & \textbf{19.4} & \underline{24.4} & \textbf{11.8} & 12.7 & \underline{53.1} & \textbf{40.4} & \textbf{89.6} & \cellcolor{oursrow}\textbf{35.6} \\
\bottomrule
\end{tabular}
\end{table}

\phantomsection\label{sec:exp-released}
\textbf{Comparison with Google's QAT checkpoints.}
Table~\ref{tab:merged} compares QUASAR with Google's public QAT checkpoint and PTQ baselines for Gemma-4 E4B. For the INT4 Gemma model, QUASAR achieves a held-out KL divergence of 0.022 and an average accuracy of 35.8 on the ten downstream benchmarks, compared with 0.064 and 34.0 for Google's QAT checkpoint. This corresponds to 66\% lower KL divergence and 1.8 points higher average accuracy. Notably, QUASAR achieves these results with only 600M tokens of distillation and 5 hours of training, while Google does not disclose the training data or compute used for its QAT checkpoint. QUASAR also uses a larger quantization group size of 64, compared with 32 for Google's checkpoint. Appendix~\ref{app:released-artifacts} extends this comparison to Gemma-4 12B.

\subsection{Ablations and Training Overhead}\label{sec:heal-ablations}
We perform experiments to ablate the effectiveness of the components of QUASAR. Figure~\ref{fig:heald-abl} presents ablations of QUASAR's two components: scale search and saliency for weighted least squares. Removing both recovers Standard QAT. The experiment uses Qwen3-4B-Thinking-2507 at INT2 with the healing recipe of Section~\ref{sec:exp-healing}. Scale search provides most of the gain, while Adam and Fisher saliency perform similarly, despite Adam being computationally cheaper. QUASAR achieves the lowest training and evaluation loss across ablations. Appendix~\ref{app:healing}, Table~\ref{tab:abl-main} shows further details.

Furthermore, we compare the efficiency of QUASAR and Standard QAT. QUASAR uses loss-aware reconstruction which introduces slight latency overhead. Table~\ref{tab:heald-training-overhead} in Appendix presents the average latency breakdown for QAD of Qwen3-4B at INT3 on 8x H100 GPUs. QUASAR only adds 1.4\% wall-clock step time compared with Standard QAT (2.97 and 2.93 seconds per step, respectively), with the overhead only coming from the weight reconstruction. Compared with Standard QAT, QUASAR introduces no memory overhead.

\begin{figure}[!htb]
\centering
\includegraphics[width=\linewidth]{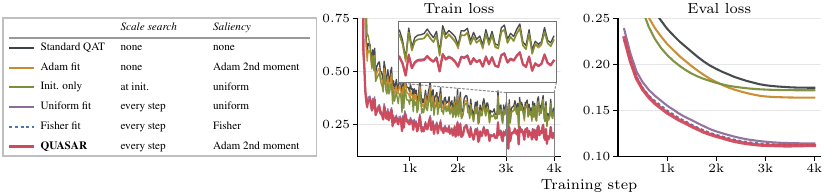}
\caption{Ablation Study: INT2 QAD for Qwen3-4B-Thinking-2507 with different choices for QUASAR's components. Training and evaluation (held-out) loss both use forward KL.}
\label{fig:heald-abl}
\end{figure}

\section{Conclusion}

We introduce QUASAR, which integrates loss-aware reconstruction into the training loop to lower the loss floor of QAT. QUASAR combines scale search with a saliency-weighted dequantization fit, using an exponential moving average of squared gradients as the saliency signal. It adds only 1.4\% latency overhead per training step and incurs no inference overhead. Our theoretical analysis shows that loss-aware reconstruction error directly controls QAT convergence and the quantized model's final loss, providing a principled foundation for QUASAR's objective. We evaluate QUASAR across four model families and four quantization formats. Across these settings, QUASAR achieves lower training and evaluation loss than competitive QAT baselines and improves downstream accuracy over QAT and PTQ methods. Notably, QUASAR's INT4 Gemma-4 checkpoint outperforms Google's QAT checkpoint in KL divergence and benchmark accuracy.

\endgroup

\bibliography{refs}
\bibliographystyle{iclr2027_conference}

\appendix
\newpage
\section*{Appendix}

\section{Method Details}\label{app:method}

This appendix provides additional details on QUASAR. Using real weight groups,
it examines per-weight reconstruction error, the clipping ranges selected by
scale search, and the saliency signal $h$, complementing
Section~\ref{sec:method}. It also reports training-step
overhead (Table~\ref{tab:heald-training-overhead}), gives algorithm variants
for formats with fast inference kernels (Appendix~\ref{app:formats}), and
compares two saliency weightings (Appendix~\ref{app:saliency-power}).

\begin{figure}[H]
\centering
\includegraphics[width=\linewidth]{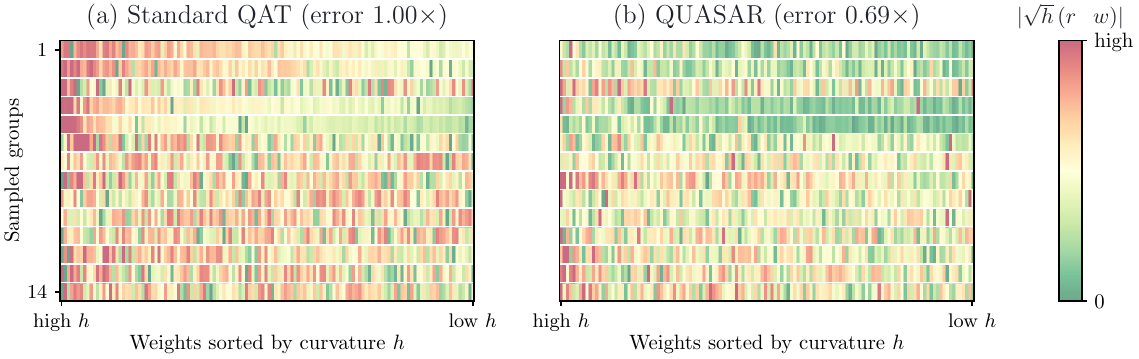}
\caption{Per-weight loss-aware reconstruction error for 14 Qwen3-4B weight
groups at INT3. Color shows $|\sqrt{h}\,(r-w)|$ under (a) Standard QAT with
min--max reconstruction and (b) QUASAR. Each row contains 128 weights ordered
by saliency $h$.}
\label{fig:beforeafter}
\end{figure}

At INT3, 99.6\% of Qwen3-4B weight groups select a range narrower than
min--max, reducing $\widehat S_t$ to 69\% of the full-range baseline
(Figure~\ref{fig:range-decomp}). Figure~\ref{fig:heald-ss} shows the selected
ranges across models and bit widths.

\begin{figure}[H]
\centering
\includegraphics[width=\linewidth]{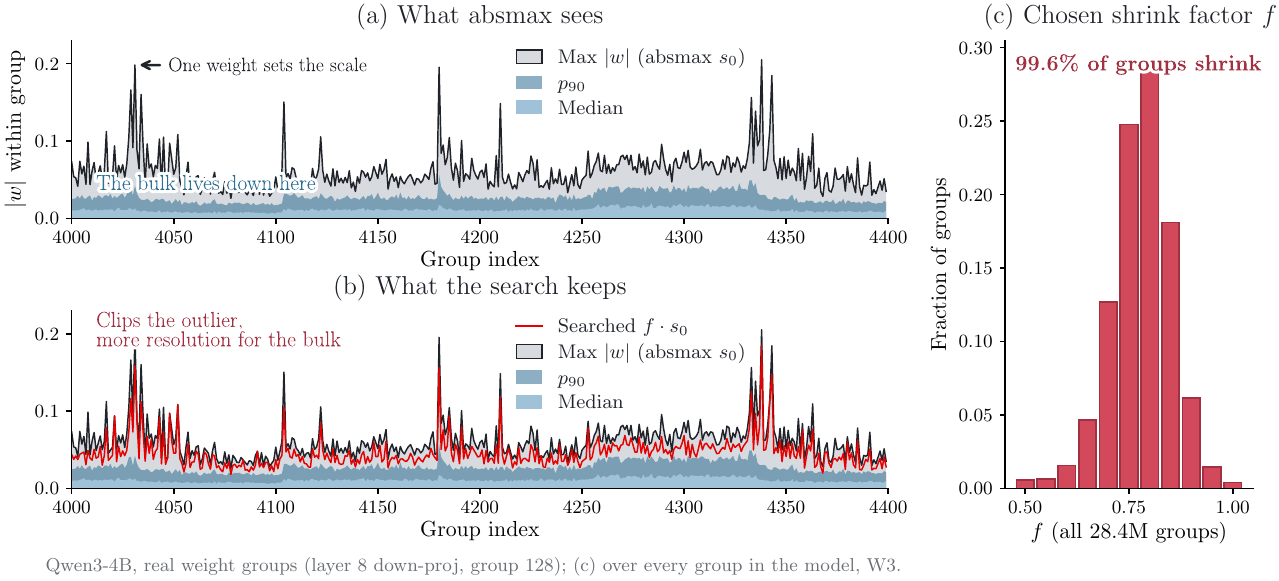}
\caption{Scale search for Qwen3-4B weight groups at INT3. (a)~One extreme
weight sets the min--max range. (b)~QUASAR clips the outlier and covers the
bulk. (c)~Distribution of selected range factors $f^\star$ across 28.4M
groups; 99.6\% use a range narrower than min--max.}
\label{fig:range-decomp}
\end{figure}

Across both models, the median selected range factor decreases from 0.95 at
INT4 to 0.60 at INT2 (Figure~\ref{fig:heald-ss}). With fewer codes, the search
clips extreme weights more aggressively to better represent the bulk of the
distribution.

\begin{figure}[!htb]
\centering
\includegraphics[width=0.92\linewidth]{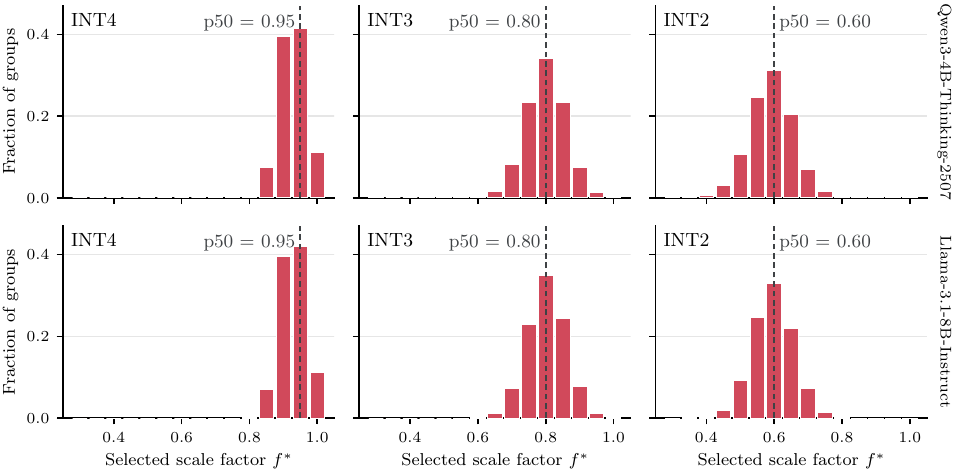}
\caption{Distribution of selected clipping-range factors $f^\star$ across
weight groups after QAD at INT4, INT3, and INT2 for
Qwen3-4B-Thinking-2507 and Llama-3.1-8B-Instruct. Lower bit widths favor
narrower ranges.}
\label{fig:heald-ss}
\end{figure}

\begin{figure}[H]
\centering
\includegraphics[width=\linewidth]{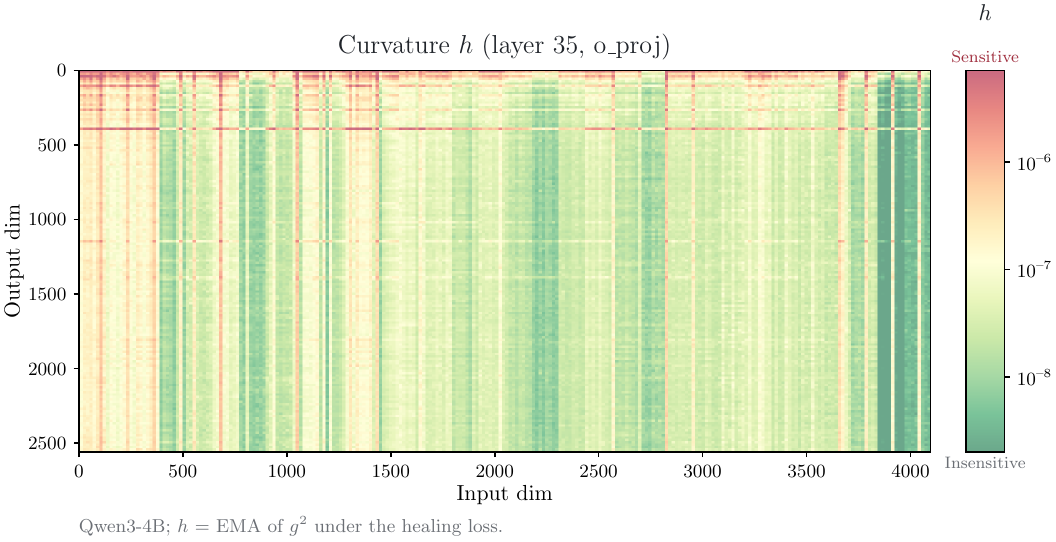}
\caption{Saliency $h$ in one Qwen3-4B layer. Weight groups span 128 weights
along the input dimension. The large variation within each group lets QUASAR
prioritize the most salient weights.}
\label{fig:curvature-heatmap}
\end{figure}

\begin{table}[!htbp]
\centering
\footnotesize
\renewcommand{\arraystretch}{0.92}
\setlength{\tabcolsep}{5pt}
\caption{Wall-clock time of one training step, per component and in seconds:
quantization-aware distillation of Qwen3-4B-Thinking-2507 at INT3 on 8x
H100 GPUs. QUASAR differs from Standard QAT only in the weight reconstruction row.}
\label{tab:heald-training-overhead}
\setlength{\tabcolsep}{7pt}
\begin{tabular}{@{}l *{2}{>{\centering\arraybackslash}p{0.18\linewidth}}@{}}
\toprule
Component & Standard QAT & QUASAR \\
\midrule
Teacher forward & 0.373 & 0.373 \\
Student compute & 0.366 & 0.366 \\
\textbf{Weight reconstruction} & \textbf{0.294} & \textbf{0.335} \\
Loss computation & 0.065 & 0.065 \\
Backward & 1.805 & 1.805 \\
Optimizer update & 0.021 & 0.022 \\
Other + synchronization & 0.006 & 0.006 \\
\midrule
\textbf{Total step} & \textbf{2.930} & \textbf{2.972} \\
\% vs. Standard QAT & -- & $+1.4\%$ \\
\bottomrule
\end{tabular}
\end{table}

\subsection{Algorithm Statements for Deployment Formats}\label{app:formats}

QUASAR applies to affine and symmetric integer formats, NVFP4, and
MXFP4~\citep{mxfp4}.
Each variant follows Algorithm~\ref{alg:quasar}: it searches over code
assignments, fits the allowed dequantization parameters when the format
permits, and selects the reconstruction with the lowest loss-aware error. The
code grid and stored parameters depend on the format. Table~\ref{tab:formats}
summarizes the variants.

QUASAR checkpoints use standard layouts supported by vLLM's quantization
backends. In vLLM, Marlin~\citep{marlin} and
Machete~\citep{vllm-gptqmodel-docs} serve INT4 weights, NVFP4 runs natively
on Blackwell GPUs~\citep{vllm-modelopt-docs}, and MXFP4 weight-only
checkpoints are supported~\citep{vllm-mxfp4-docs}.
INT3 and INT2 currently lack dedicated fast kernels, so we treat them as
accuracy frontiers rather than deployment formats.

\begin{table}[H]
\centering
\scriptsize
\caption{QUASAR variants for deployment formats with fast inference kernels.
The formats differ in their code grids, group sizes, stored parameters, and
dequantization constraints.}
\label{tab:formats}
\begin{tabular}{@{}llllll@{}}
\toprule
Format & Codes & Group & Stored parameters & Fit & Statement \\
\midrule
INT sym.\ (GPTQ layout) & INT4 & 32--128 & scale & Eq.~\eqref{eq:symfit} & Alg.~\ref{alg:quasar-sym} \\
INT affine (AWQ layout) & INT4 & 32--128 & scale, int.\ zero-point & Eq.~\eqref{eq:intzpfit} & Alg.~\ref{alg:quasar-intzp} \\
NVFP4 & E2M1 & 16 & E4M3 scale, FP32 tensor scale & Eq.~\eqref{eq:symfit}, E4M3-rounded & Alg.~\ref{alg:quasar-nvfp4} \\
MXFP4 & E2M1 & 32 & E8M0 scale & search only & Alg.~\ref{alg:quasar-mxfp4} \\
\bottomrule
\end{tabular}
\end{table}

\paragraph{Affine integer formats.}
For affine quantization, $r_i=sq_i+z$. Given codes $q$, minimizing
Equation~\eqref{eq:fisher-obj} over $s$ and $z$ is weighted least squares:
\begin{equation}\label{eq:ridgefit}
  s^\star
  =
  \frac{\sum_i h_i(q_i-\bar q_h)(w_i-\bar w_h)}
       {\sum_i h_i(q_i-\bar q_h)^2},
  \qquad
  z^\star
  =
  \bar w_h-s^\star\bar q_h ,
\end{equation}
where $\bar q_h=\sum_i h_iq_i/\sum_i h_i$ and $\bar w_h=\sum_i h_iw_i/\sum_i h_i$
are the saliency-weighted means. Algorithm~\ref{alg:quasar-affine} applies
this fit to every candidate code assignment.

\begin{algorithm}[H]
\caption{\textbf{QUASAR for affine integer quantization.} Candidate ranges
assign codes; weighted least squares then fits the stored scale and offset.}
\label{alg:quasar-affine}
\small
\begin{algorithmic}[1]
\Require latent weights $w\in\R^{g}$ of one group; saliency
$h\in\R_{+}^{g}$ (Adam's second moment); bit width $b$; clipping factors
$\mathcal F\subseteq(0,1]$
\State $c \gets \tfrac{1}{2}(\min_i w_i + \max_i w_i)$;\;
       $\rho \gets \tfrac{1}{2}(\max_i w_i - \min_i w_i)$
       \Comment{Standard QAT uses $f{=}1$}
\State $\widehat S^\star \gets \infty$
\For{$f \in \mathcal F$}
  \State $\tilde s_f \gets 2f\rho/(2^b-1)$;\;
         $\tilde z_f \gets c-f\rho$
         \Comment{assignment parameters}
  \State $q_f \gets
         \operatorname{clamp}_{[0,2^b-1]}
         \!\left(\round\!\left((w-\tilde z_f\mathbf{1})/\tilde s_f\right)\right)$
         \Comment{code assignment $w \mapsto q$}
  \State $(s_f,z_f) \gets$ Eq.~\eqref{eq:ridgefit} with $q=q_f$
         \Comment{dequantization fit}
  \State $r_f \gets s_fq_f+z_f\mathbf{1}$
  \State $\widehat S_f \gets \sum_i h_i(r_{f,i}-w_i)^2$
         \Comment{loss-aware reconstruction error}
  \State \textbf{if} $\widehat S_f < \widehat S^\star$
         \textbf{then} $\widehat S^\star \gets \widehat S_f$;
         $(q, s, z) \gets (q_f, s_f, z_f)$
         \textbf{end if}
\EndFor
\State \Return $r \gets sq+z\mathbf{1}$
\end{algorithmic}
\end{algorithm}

All candidates are evaluated in one batched kernel. The assignment parameters
are temporary; the best codes and fitted scale and offset define the
reconstruction. Backpropagation uses the same STE as standard QAT. This adds
1.4\% to the step time (Table~\ref{tab:heald-training-overhead}) and preserves
the standard RTN/GPTQ inference layout.

\paragraph{Symmetric integer formats.}
Symmetric integer formats store one scale per group and no zero-point, so
$r_i=sq_i$. For fixed signed codes, weighted least squares gives
\begin{equation}\label{eq:symfit}
  s^\star
  =
  \frac{\sum_i h_iq_iw_i}{\sum_i h_iq_i^2}.
\end{equation}
This is the GPTQ layout used by W4A16 and W4A8 serving kernels.

\begin{algorithm}[H]
\caption{\textbf{QUASAR for symmetric integer quantization.} Signed codes use
a zero-centered assignment, and weighted least squares fits one scale.}
\label{alg:quasar-sym}
\small
\begin{algorithmic}[1]
\Require latent weights $w\in\R^{g}$ of one group; saliency
$h\in\R_{+}^{g}$ (Adam's second moment); bit width $b$; clipping factors
$\mathcal F \subseteq (0, 1]$
\State $M \gets 2^{b-1}$;\; $\rho \gets \max_i |w_i|$
       \Comment{Standard QAT uses $f{=}1$}
\State $\widehat S^\star \gets \infty$
\For{$f \in \mathcal F$}
  \State $\tilde s_f \gets f\rho/M$
         \Comment{assignment scale}
  \State $q_f \gets
         \operatorname{clamp}_{[-M,M-1]}
         \!\left(\round\!\left(w/\tilde s_f\right)\right)$
         \Comment{code assignment $w \mapsto q$}
  \State $s_f \gets \sum_i h_iq_{f,i}w_i \,/\, \sum_i h_iq_{f,i}^2$
         \Comment{Eq.~\eqref{eq:symfit}; dequantization fit}
  \State $r_f \gets s_fq_f$
  \State $\widehat S_f \gets \sum_i h_i(r_{f,i}-w_i)^2$
         \Comment{loss-aware reconstruction error}
  \State \textbf{if} $\widehat S_f < \widehat S^\star$
         \textbf{then} $\widehat S^\star \gets \widehat S_f$;
         $(q,s) \gets (q_f,s_f)$ \textbf{end if}
\EndFor
\State \Return $r \gets sq$
\end{algorithmic}
\end{algorithm}

\paragraph{Affine integer formats with an integer zero-point.}
Some affine layouts, including AWQ, store an integer zero-point $q_0$ and
reconstruct weights as $r_i=s(q_i-q_0)$. For each code assignment, QUASAR
first computes the continuous fit, rounds $q_0=-z^\star/s^\star$ to a valid
code, and refits the scale:
\begin{equation}\label{eq:intzpfit}
  q_0
  =
  \operatorname{clamp}\!\left(\round\!\left(-z^\star/s^\star\right),\,0,\,2^b-1\right),
  \qquad
  s^\star_{q_0}
  =
  \frac{\sum_i h_i(q_i-q_0)\,w_i}{\sum_i h_i(q_i-q_0)^2}.
\end{equation}
Algorithm~\ref{alg:quasar-intzp} applies this constraint to every candidate.

\begin{algorithm}[H]
\caption{\textbf{QUASAR for affine integer quantization with an integer
zero-point.} The continuous fit is snapped to an integer zero-point, then the
scale is refit.}
\label{alg:quasar-intzp}
\small
\begin{algorithmic}[1]
\Require latent weights $w\in\R^{g}$ of one group; saliency
$h\in\R_{+}^{g}$ (Adam's second moment); bit width $b$; clipping factors
$\mathcal F \subseteq (0, 1]$
\State $c \gets \tfrac{1}{2}(\min_i w_i + \max_i w_i)$;\;
       $\rho \gets \tfrac{1}{2}(\max_i w_i - \min_i w_i)$
       \Comment{Standard QAT uses $f{=}1$}
\State $\widehat S^\star \gets \infty$
\For{$f \in \mathcal F$}
  \State $\tilde s_f \gets 2f\rho/(2^b-1)$;\;
         $\tilde z_f \gets c-f\rho$
         \Comment{assignment parameters}
  \State $q_f \gets
         \operatorname{clamp}_{[0,2^b-1]}
         \!\left(\round\!\left((w-\tilde z_f\mathbf{1})/\tilde s_f\right)\right)$
         \Comment{code assignment $w \mapsto q$}
  \State $(\bar s_f,\bar z_f) \gets$ Eq.~\eqref{eq:ridgefit} with $q=q_f$
         \Comment{unconstrained fit}
  \State $q_{0,f} \gets
         \operatorname{clamp}_{[0,2^b-1]}
         \!\left(\round\!\left(-\bar z_f/\bar s_f\right)\right)$
         \Comment{Eq.~\eqref{eq:intzpfit}}
  \State
         $s_f \gets \sum_i h_i (q_{f,i}-q_{0,f})\,w_i \,/\, \sum_i h_i (q_{f,i}-q_{0,f})^2$
         \Comment{Eq.~\eqref{eq:intzpfit}; dequantization fit}
  \State $r_f \gets s_f(q_f-q_{0,f}\mathbf{1})$
  \State $\widehat S_f \gets \sum_i h_i(r_{f,i}-w_i)^2$
         \Comment{loss-aware reconstruction error}
  \State \textbf{if} $\widehat S_f < \widehat S^\star$
         \textbf{then} $\widehat S^\star \gets \widehat S_f$;
         $(q, s, q_0) \gets (q_f, s_f, q_{0,f})$
         \textbf{end if}
\EndFor
\State \Return $r \gets s(q-q_0\mathbf{1})$
\end{algorithmic}
\end{algorithm}

\paragraph{NVFP4.}
NVFP4 represents groups of 16 weights with E2M1 codes, an E4M3 scale per
group, and a tensor-level FP32 scale $\sigma$~\citep{alvarez2025nvfp4}. It is
symmetric and has no zero-point. We set $\sigma$ from the tensor's maximum
absolute weight and keep it fixed. For each candidate, QUASAR rounds the
assignment scale to E4M3, assigns E2M1 codes, and fits the stored E4M3
dequantization scale. It keeps the reconstruction with the lowest loss-aware
error (Algorithm~\ref{alg:quasar-nvfp4}).

\begin{algorithm}[H]
\caption{\textbf{QUASAR for NVFP4.} E2M1 codes use E4M3-rounded assignment and
fitted group scales under a fixed FP32 tensor scale.}
\label{alg:quasar-nvfp4}
\small
\begin{algorithmic}[1]
\Require group weights $w\in\R^{16}$; saliency $h\in\R_{+}^{16}$ (Adam's
second moment); clipping factors $\mathcal F \subseteq (0, 1]$; tensor scale
$\sigma$ (fixed)
\State $\tilde w \gets w/\sigma$;\;
       $\rho \gets \max_i |\tilde w_i| / 6$ \Comment{$6$ = largest E2M1 magnitude}
\State $\widehat S^\star \gets \infty$
\For{$f \in \mathcal F$}
  \State $\tilde s_f \gets Q_{\mathrm{E4M3}}(f\rho)$
         \Comment{assignment scale}
  \State $q_f \gets Q_{\mathrm{E2M1}}(\tilde w/\tilde s_f)$
         \Comment{code assignment $w \mapsto q$}
  \State $s_f \gets Q_{\mathrm{E4M3}}\!\left(
         \sum_i h_iq_{f,i}\tilde w_i \,/\, \sum_i h_iq_{f,i}^2\right)$
         \Comment{Eq.~\eqref{eq:symfit}; dequantization fit}
  \State $r_f \gets \sigma s_fq_f$
  \State $\widehat S_f \gets \sum_i h_i(r_{f,i}-w_i)^2$
         \Comment{loss-aware reconstruction error}
  \State \textbf{if} $\widehat S_f < \widehat S^\star$
         \textbf{then} $\widehat S^\star \gets \widehat S_f$;
         $(q,s) \gets (q_f,s_f)$ \textbf{end if}
\EndFor
\State \Return $r \gets \sigma sq$
\end{algorithmic}
\end{algorithm}

\paragraph{MXFP4.}
MXFP4 stores E2M1 codes in groups of 32 with one E8M0 power-of-two scale per
group. Because the stored scale is discrete, QUASAR searches the E8M0 scale
grid directly rather than fitting a continuous scale. Saliency $h$ determines
which candidate has the lowest loss-aware error
(Algorithm~\ref{alg:quasar-mxfp4}).

\begin{algorithm}[H]
\caption{\textbf{QUASAR for MXFP4.} E2M1 codes share an E8M0 power-of-two
scale, so scale search also determines the stored dequantization scale.}
\label{alg:quasar-mxfp4}
\small
\begin{algorithmic}[1]
\Require group weights $w\in\R^{32}$; saliency $h\in\R_{+}^{32}$ (Adam's
second moment); search depth $k$
\State $e_0 \gets$ smallest integer with
       $2^{e_0}\cdot 6 \ge \max_i |w_i|$ \Comment{$6$ = largest E2M1 magnitude}
\State $\mathcal F \gets \{2^{-\delta}:\delta=0,1,\ldots,k\}$
\State $\widehat S^\star \gets \infty$
\For{$f \in \mathcal F$}
  \State $s_f \gets f2^{e_0}$
         \Comment{assignment and dequantization scale}
  \State $q_f \gets Q_{\mathrm{E2M1}}(w/s_f)$
         \Comment{code assignment $w \mapsto q$}
  \State $r_f \gets s_fq_f$
  \State $\widehat S_f \gets \sum_i h_i(r_{f,i}-w_i)^2$
         \Comment{loss-aware reconstruction error}
  \State \textbf{if} $\widehat S_f < \widehat S^\star$
         \textbf{then} $\widehat S^\star \gets \widehat S_f$;
         $(q,s) \gets (q_f,s_f)$ \textbf{end if}
\EndFor
\State \Return $r \gets sq$
\end{algorithmic}
\end{algorithm}

\subsection{Choice of saliency weighting}
\label{app:saliency-power}

QUASAR weights squared reconstruction error by saliency $h_t$, our diagonal
curvature estimate. A natural alternative is $h_t^2$, motivated by the squared
linearized gradient mismatch:
\[
\norm{\Hess_t(r_t-w_t)}^2
=(r_t-w_t)^\top\Hess_t^2(r_t-w_t).
\]
We compare both choices with all other settings fixed: AdamW at INT2, INT3,
and INT4, and SGD at three stable INT2 learning rates. Weighting by $h_t$
gives lower held-out loss in all six settings. Weighting by $h_t^2$ is
unstable at INT4 and at the higher SGD learning rates
(Figure~\ref{fig:saliency-power}).

\begin{figure}[t]
\centering
\includegraphics[width=0.98\linewidth]{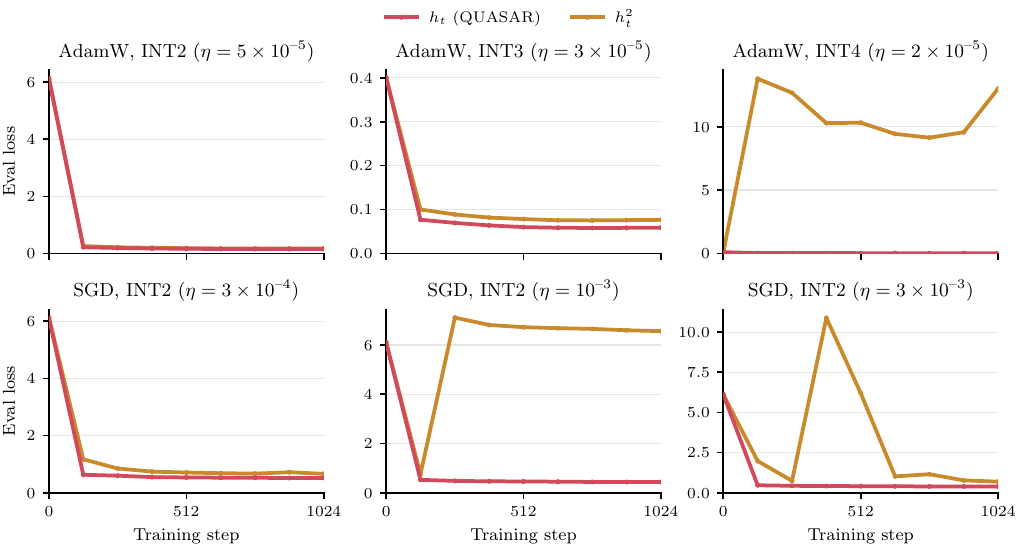}
\caption{Comparison of saliency weights $h_t$ and $h_t^2$. We report held-out
loss under AdamW at INT2, INT3, and INT4 (top) and under SGD at three INT2
learning rates (bottom).}
\label{fig:saliency-power}
\end{figure}

\clearpage

\section{Proofs and Deferred Theory Discussion}
\label{app:theory}
This appendix contains the proofs of the results stated in
Section~\ref{sec:theory}.

\subsection{Proofs of the main results}
\label{app:quasar-certificate}

We use the notation $\E_t[\cdot]$ introduced in
Section~\ref{sec:theory} and write $H(y):=\nabla^2\loss(y)$ for the Hessian
at $y$. We first record a sufficient condition for
Assumption~3. Fix any reconstruction index $t$ under consideration, including
the terminal index $t=T$, and put
\[
  d_t=r_t-w_t,
  \qquad
  D_t=\operatorname{diag}(h_t).
\]
Suppose that, at every point $y=w_t+u d_t$ on the reconstruction segment,
\begin{equation}
  H(y)^2\preceq C D_t,
  \qquad 0\leq u\leq1.
  \label{eq:quasar-hessian-sufficient}
\end{equation}
By the fundamental theorem of calculus, Jensen's inequality, and the symmetry
of the Hessian,
\[
\begin{aligned}
  \norm{\nabla\loss(r_t)-\nabla\loss(w_t)}^2
  &=
  \norm{\int_0^1 H(w_t+u d_t)d_t\,du}^2 \\
  &\leq
  \int_0^1 d_t^\top H(w_t+u d_t)^2d_t\,du \\
  &\leq
  C d_t^\top D_t d_t
  =C\widehat S_t.
\end{aligned}
\]
Thus \eqref{eq:quasar-hessian-sufficient} implies Assumption~3. In
particular, if $\loss$ is $L$-smooth and
$h_{t,i}\geq h_{\min}>0$, then, for every $y,d\in\R^d$,
\[
  \norm{H(y)d}^2
  \leq
  L^2\norm{d}^2
  \leq
  \frac{L^2}{h_{\min}}\sum_i h_{t,i}d_i^2.
\]
Consequently, Assumption~3 holds with $C=L^2/h_{\min}$.

\begin{proof}[Proof of Theorem~\ref{thm:quasar-certificate}]
For part~(i), set
\[
  a_t=\nabla\loss(w_t),
  \qquad
  c_t=\nabla\loss(r_t).
\]
By smoothness and the update in Assumption~2,
\[
\begin{aligned}
  \loss(w_{t+1})
  \leq{}&
  \loss(w_t)
  -\eta\langle a_t,c_t+\xi_t\rangle \\
  &+\frac{L\eta^2}{2}
  \norm{c_t+\xi_t}^2.
\end{aligned}
\]
Conditional on $\mathcal H_t$, both $a_t$ and $c_t$ are fixed. Moreover,
\[
  \E_t\norm{c_t+\xi_t}^2
  =\norm{c_t}^2+\E_t\norm{\xi_t}^2
  \leq (1+M)\norm{c_t}^2+\sigma^2.
\]
It follows that
\[
  \E_t[\loss(w_{t+1})]
  \leq
  \loss(w_t)
  -\eta\langle a_t,c_t\rangle
  +\frac{L\eta^2}{2}
   \left((1+M)\norm{c_t}^2+\sigma^2\right).
\]
Using the identity
\[
  2\langle a_t,c_t\rangle
  =\norm{a_t}^2+\norm{c_t}^2-\norm{c_t-a_t}^2
\]
and Assumption~3, we obtain
\begin{equation}
\begin{aligned}
  \E_t[\loss(w_{t+1})]
  \leq{}&
  \loss(w_t)
  -\frac{\eta}{2}\norm{a_t}^2 \\
  &-\frac{\eta\bigl(1-L\eta(1+M)\bigr)}{2}
    \norm{c_t}^2
  +\frac{\eta C}{2}\widehat S_t
  +\frac{L\eta^2}{2}\sigma^2.
\end{aligned}
\label{eq:quasar-one-step}
\end{equation}
We take expectations and sum this inequality for $0\leq t<T$. The terms
involving the loss telescope, and therefore
\[
\begin{aligned}
  \frac{\eta}{2}\sum_{t=0}^{T-1}\E\Big[
    \norm{\nabla\loss(w_t)}^2
    +\bigl(1-L\eta(1+M)\bigr)\norm{\nabla\loss(r_t)}^2
  \Big]
  \leq{}&
  \loss(w_0)-\E\loss(w_T) \\
  &+\frac{\eta C}{2}
    \sum_{t=0}^{T-1}\E\widehat S_t
  +\frac{L\eta^2T}{2}\sigma^2.
\end{aligned}
\]
Since $\loss(w_T)\geq\loss^\star$, the desired result follows after division
by $\eta T/2$. This proves part~(i).

For part~(ii), put $\gamma=1-\eta\mu$. The assumption on $\eta$ ensures that
$0<\gamma<1$ and that
$1-L\eta(1+M)>0$. We may thus omit the term involving $\norm{c_t}^2$ in
\eqref{eq:quasar-one-step}. The PL inequality at $w_t$ gives
\[
  \E[\loss(w_{t+1})-\loss^\star]
  \leq
  \gamma\,\E[\loss(w_t)-\loss^\star]
  +\frac{\eta C}{2}\E\widehat S_t
  +\frac{L\eta^2}{2}\sigma^2.
\]
Iteration yields
\begin{equation}
\begin{aligned}
  \E[\loss(w_T)-\loss^\star]
  \leq{}&
  \gamma^T(\loss(w_0)-\loss^\star) \\
  &+\frac{\eta C}{2}
    \sum_{t=0}^{T-1}\gamma^{T-1-t}\E\widehat S_t \\
  &+\frac{L\eta\sigma^2}{2\mu}
    \bigl(1-\gamma^T\bigr).
\end{aligned}
\label{eq:quasar-latent-recursion}
\end{equation}

It remains to compare $w_T$ and $r_T$. Applying the smoothness inequality at
$w_T$ with the point $w_T-L^{-1}\nabla\loss(w_T)$, we get
\[
  \loss\!\left(w_T-\frac{1}{L}\nabla\loss(w_T)\right)
  \leq
  \loss(w_T)-\frac{1}{2L}\norm{\nabla\loss(w_T)}^2.
\]
The left-hand side is at least $\loss^\star$, and hence
\[
  \norm{\nabla\loss(w_T)}^2
  \leq
  2L\bigl(\loss(w_T)-\loss^\star\bigr).
\]
Assumption~3 at the terminal index $t=T$ gives
\[
  \norm{\nabla\loss(r_T)-\nabla\loss(w_T)}^2
  \leq C\widehat S_T.
\]
By the PL inequality at $r_T$ and
$\norm{x+y}^2\leq2\norm{x}^2+2\norm{y}^2$, it follows that
\[
\begin{aligned}
  \loss(r_T)-\loss^\star
  &\leq
  \frac{1}{2\mu}\norm{\nabla\loss(r_T)}^2 \\
  &\leq
  \frac{1}{\mu}\norm{\nabla\loss(w_T)}^2
  +\frac{1}{\mu}
    \norm{\nabla\loss(r_T)-\nabla\loss(w_T)}^2 \\
  &\leq
  \frac{2L}{\mu}\bigl(\loss(w_T)-\loss^\star\bigr)
  +\frac{C}{\mu}\widehat S_T.
\end{aligned}
\]
Taking expectations and substituting
\eqref{eq:quasar-latent-recursion} yields the exact intermediate bound
\begin{equation}
\begin{aligned}
  \E[\loss(r_T)-\loss^\star]
  \leq{}&
  \frac{2L}{\mu}\gamma^T(\loss(w_0)-\loss^\star) \\
  &+\frac{L\eta C}{\mu}
    \sum_{t=0}^{T-1}\gamma^{T-1-t}\E\widehat S_t \\
  &+\frac{L^2\eta\sigma^2}{\mu^2}(1-\gamma^T)
  +\frac{C}{\mu}\E\widehat S_T.
\end{aligned}
\label{eq:quasar-final-exact}
\end{equation}
The random index $\tau$ in part~(ii) is independent of the training randomness
and has distribution
\[
  \Pr(\tau=t)
  =
  \frac{\eta\mu \gamma^{T-1-t}}{1-\gamma^T},
  \qquad t\in\{0,\ldots,T-1\},
\]
so later training steps receive greater weight.
Because $1-\gamma=\eta\mu$, these probabilities sum to one, and independence gives
\[
\begin{aligned}
  \E\widehat S_\tau
  &=
  \sum_{t=0}^{T-1}
  \frac{\eta\mu \gamma^{T-1-t}}{1-\gamma^T}\E\widehat S_t, \\
  \sum_{t=0}^{T-1}\gamma^{T-1-t}\E\widehat S_t
  &=
  \frac{1-\gamma^T}{\eta\mu}\E\widehat S_\tau.
\end{aligned}
\]
Thus the second term in \eqref{eq:quasar-final-exact} equals
$\frac{LC}{\mu^2}(1-\gamma^T)\E\widehat S_\tau$. Since the reconstruction
errors are nonnegative and $1-\gamma^T\leq1$, this is at most
$\frac{LC}{\mu^2}\E\widehat S_\tau$. Applying the same bound to the noise
term and recalling $\gamma=1-\eta\mu$ gives \eqref{eq:quasar-final-bound}, proving
part~(ii).
\end{proof}

\subsection{Deferred discussion of the assumptions and the bound}
\label{app:theory-discussion}

\begin{remark}[Assumptions]
Assumptions~1 and~2 are standard in stochastic optimization. Assumption~1
requires a twice continuously differentiable, smooth objective with a finite
infimum, but not convexity. Assumption~2 allows gradient noise to grow with the
gradient norm and recovers bounded variance when $M=0$; the same condition is
used to analyze biased SGD~\citep{ajalloeian2020convergence}.

Assumption~3 is specific to reconstruction: small loss-aware reconstruction
error must imply a small gradient mismatch.
Appendix~\ref{app:quasar-certificate} derives a sufficient curvature
condition. Under $L$-smoothness and a positive saliency floor
$h_{t,i}\geq h_{\min}>0$, it holds with $C=L^2/h_{\min}$
\citep{ghadimi2013stochastic,bottou2018optimization,ajalloeian2020convergence}.
Adam's second moment motivates the saliency estimate but does not guarantee
the condition. Empirically, saliency tracks the Hessian diagonal, and the full
quadratic form is nearly proportional to $\widehat S$
(Figure~\ref{fig:curvature-proxy}, Appendix~\ref{app:theory-validation}).

Assumption~4 is the Polyak--{\L}ojasiewicz (PL) condition and is needed only
for part~(ii) of Theorem~\ref{thm:quasar-certificate}. Strong convexity implies
PL, but PL does not require convexity and ensures that every stationary point
is a global minimum~\citep{polyak1963gradient,karimi2016linear}. Local PL-type
conditions have also been established for some overparameterized neural
networks~\citep{liu2022loss}.
\end{remark}

\paragraph{Comparison with standard and biased SGD.}
To compare the guarantees, define
\[
    B_T
    :=
    \frac{\loss(w_0)-\loss^\star}{\eta T}+\eta\sigma^2,
    \qquad
    \overline S_T
    :=
    \frac1T\sum_{t=0}^{T-1}\mathbb E\widehat S_t.
\]
Let $\alpha_\eta:=1-L\eta(1+M)>0$.
Up to constant factors, the relevant bounds are
\begin{equation}
\begin{aligned}
\text{Standard SGD:}\quad
&\frac1T\sum_{t=0}^{T-1}
  \mathbb E\|\nabla\loss(w_t)\|^2
  \lesssim B_T, \\
\text{Biased SGD:}\quad
&\frac1T\sum_{t=0}^{T-1}
  \mathbb E\|\nabla\loss(w_t)\|^2
  \lesssim \frac{B_T+\zeta^2}{1-m}, \\
\text{QUASAR:}\quad
&\frac1T\sum_{t=0}^{T-1}\mathbb E\!\left[
  \|\nabla\loss(w_t)\|^2+\alpha_\eta\|\nabla\loss(r_t)\|^2
  \right]
  \lesssim B_T+C\overline S_T.
\end{aligned}
\label{eq:sgd-comparison-summary}
\end{equation}
Biased-SGD analysis assumes
$\|b_t\|^2\leq m\|\nabla\loss(w_t)\|^2+\zeta^2$ with $0\leq m<1$
\citep{ajalloeian2020convergence}. In QAT, the bias is the
reconstruction-induced gradient mismatch $b_t=e_t$. Assumption~3 bounds this
mismatch by $C\widehat S_t$, which QUASAR minimizes at every step. The
resulting bound contains the usual SGD terms plus $C\overline S_T$ and
controls stationarity at both the latent and reconstructed weights.

\clearpage

\FloatBarrier
\section{Empirical validation of the theory}
\label{app:theory-validation}

This appendix tests the assumptions and predictions of
Section~\ref{sec:theory}. Across both model families and all tested bit widths,
end-of-training reconstruction error closely tracks held-out KL for Standard
QAT, Denoising QAT, and QUASAR. QUASAR attains the lowest value of both in
every setting (Figure~\ref{fig:theory-tie}; per-projection results in
Figure~\ref{fig:sref-permodule}). The same ordering holds under the plain-SGD
dynamics analyzed by the theorem (Figure~\ref{fig:sgd-ladder}). The
relationship is also causal: restricting the scale search raises both
$\widehat S_T$ and final KL. It is predictive before training as well:
reconstruction error at initialization predicts final KL
(Figures~\ref{fig:dose-trained} and~\ref{fig:init-predictor}).

\begin{figure}[!htb]
\centering
\includegraphics[width=0.98\linewidth]{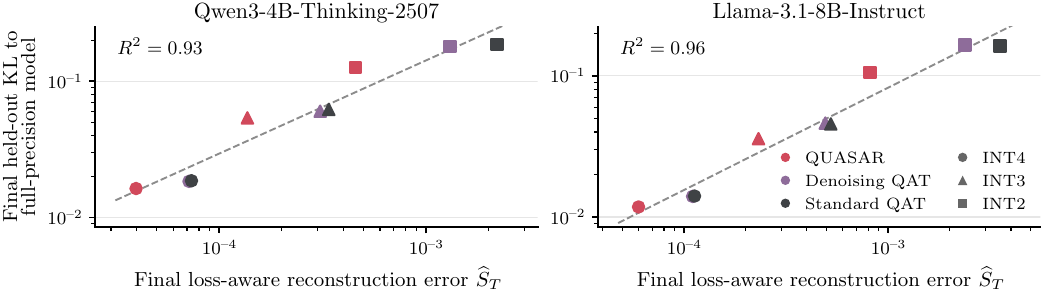}
\caption{\textbf{Reconstruction error tracks final KL.}
End-of-training loss-aware reconstruction error versus final held-out KL to the
full-precision teacher, for Standard QAT, Denoising QAT, and QUASAR at INT4,
INT3, and INT2.}
\label{fig:theory-tie}
\end{figure}

\begin{figure}[h]
\centering
\includegraphics[width=\linewidth]{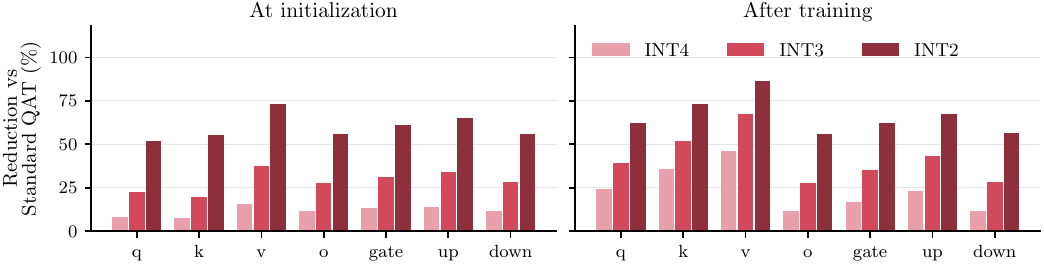}
\caption{\textbf{QUASAR reduces reconstruction error across projection types.}
Median reduction in loss-aware reconstruction error relative to Standard QAT
across Qwen3-4B-Thinking-2507 modules at INT4, INT3, and INT2, measured at
initialization (left) and after training (right). Error is measured against
the full-precision weights.}
\label{fig:sref-permodule}
\end{figure}

\begin{figure}[!htb]
\centering
\scalebox{1}[0.9]{\includegraphics[width=0.75\linewidth]{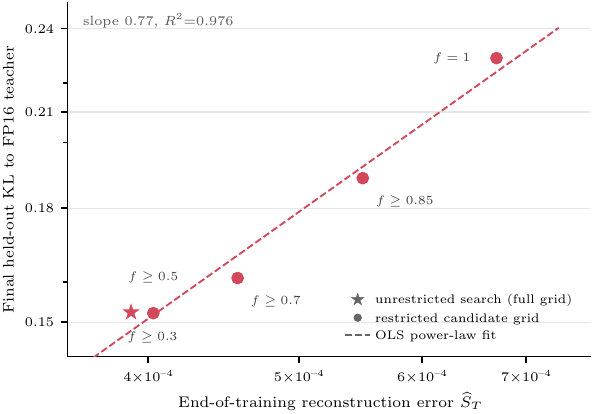}}
\caption{\textbf{Reconstruction error drives final loss.}
Five QUASAR INT2 runs differ only in the lower bound imposed on the
scale-search factor $f$. Restricting the search raises both the final
reconstruction error $\widehat S_T$ and held-out KL after 1{,}000 training
steps. The star marks the unrestricted search.}
\label{fig:dose-trained}
\end{figure}

\begin{figure}[!htb]
\centering
\scalebox{1}[0.9]{\includegraphics[width=0.75\linewidth]{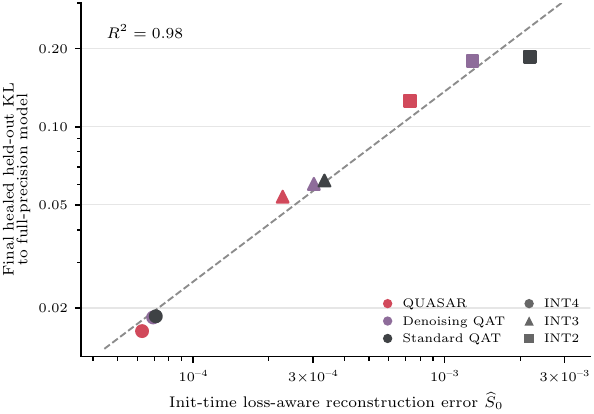}}
\caption{\textbf{Reconstruction error predicts final loss before training.}
Across Standard QAT, Denoising QAT, and QUASAR at INT4, INT3, and INT2,
reconstruction error at initialization predicts held-out KL to the
full-precision model after 4{,}096 steps ($R^2=0.98$). Results use
Qwen3-4B-Thinking-2507.}
\label{fig:init-predictor}
\end{figure}

\begin{figure}[!htb]
\centering
\scalebox{1}[0.9]{\includegraphics[width=0.75\linewidth]{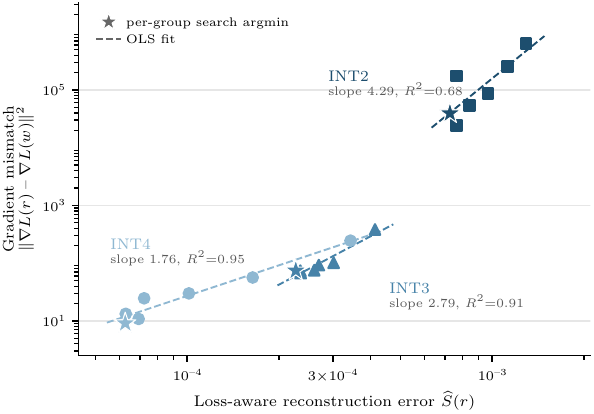}}
\caption{\textbf{Reconstruction error controls gradient mismatch.}
Across INT4, INT3, and INT2, the squared gradient mismatch rises with the
loss-aware reconstruction error $\widehat S$, supporting Assumption~3.
Circles, triangles, and squares show fixed clipping factors for INT4, INT3,
and INT2, respectively; stars show QUASAR's per-group scale-search solutions.}
\label{fig:eq11-dose}
\end{figure}

\begin{figure}[!htb]
\centering
\scalebox{1}[0.9]{\includegraphics[width=0.7\linewidth]{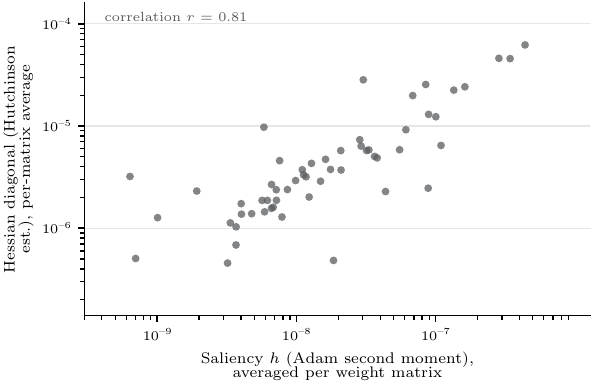}}
\caption{\textbf{QUASAR's saliency tracks curvature.}
Each point represents one weight matrix from a trained INT2
Qwen3-4B-Thinking checkpoint and compares its mean saliency $h$ (Adam's
second moment) with a 32-sample Hutchinson estimate of its mean Hessian
diagonal. The log-scale correlation is $r=0.81$
\citep{hawqv2,pyhessian}.}
\label{fig:curvature-proxy}
\end{figure}

\begin{figure}[!htb]
\centering
\scalebox{1}[0.9]{\includegraphics[width=0.75\linewidth]{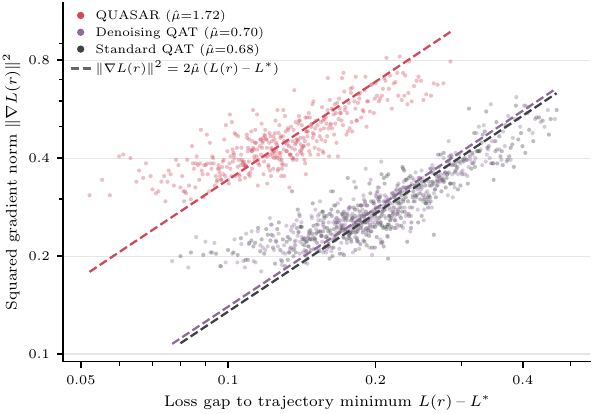}}
\caption{\textbf{Training trajectories support a PL relation.}
Each point is one step of an INT2 Qwen3-4B-Thinking run, with loss and
gradient measured at the reconstruction $r$. For all three methods, the
squared gradient norm scales with the gap above the trajectory minimum.
QUASAR's estimated $\hat\mu=1.72$ is about $2.5\times$ the baseline values.
In Theorem~\ref{thm:quasar-certificate}(ii), a larger $\mu$ gives faster
geometric convergence and smaller floor terms, which scale with $1/\mu$ or
$1/\mu^2$. This provides an optimization-side view of the healing curves
(Figure~\ref{fig:heald-curves}): QUASAR's reconstruction adds less loss, so
more of the remaining gap produces useful gradient, enabling faster healing
and a lower loss floor.}
\label{fig:pl-diag}
\end{figure}

\begin{figure}[!htb]
\centering
\includegraphics[width=0.72\linewidth]{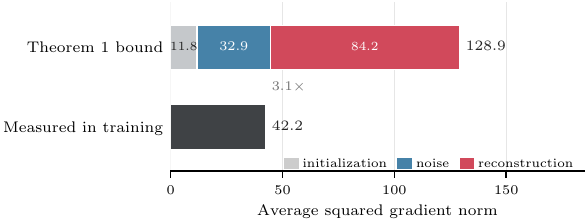}
\caption{\textbf{Reconstruction error dominates the measured bound.}
We evaluate each term of Theorem~\ref{thm:quasar-certificate}(i) on an INT2
QUASAR run with plain SGD and learning rate $10^{-3}$. We measure
$\lambda_{\max}$ by power iteration (approximating $L$), $\sigma^2$ from
per-batch gradients, and $C$ using Figure~\ref{fig:eq11-dose}. The resulting
bound is within
$3.1\times$ of the observed average squared gradient norm, and reconstruction
error is its largest term.}
\label{fig:bound-decomp}
\end{figure}

\paragraph{Validation under SGD dynamics.}
\label{app:sgd-validation}

Theorem~\ref{thm:quasar-certificate} analyzes STE--SGD, whereas the main
experiments use AdamW. We therefore repeat the INT2 quantization-aware
distillation experiment with plain SGD for 1{,}024 steps at global batch 32,
using a constant learning rate and no momentum, warmup, or gradient clipping.
QUASAR replaces AdamW's second moment with an explicit EMA of squared
gradients; the remaining recipe is unchanged. At every stable learning rate,
QUASAR reaches lower training and held-out loss floors than Standard QAT,
preserving the ordering observed with AdamW
(Figure~\ref{fig:sgd-ladder}).

\begin{figure}[t]
\centering
\includegraphics[width=\linewidth]{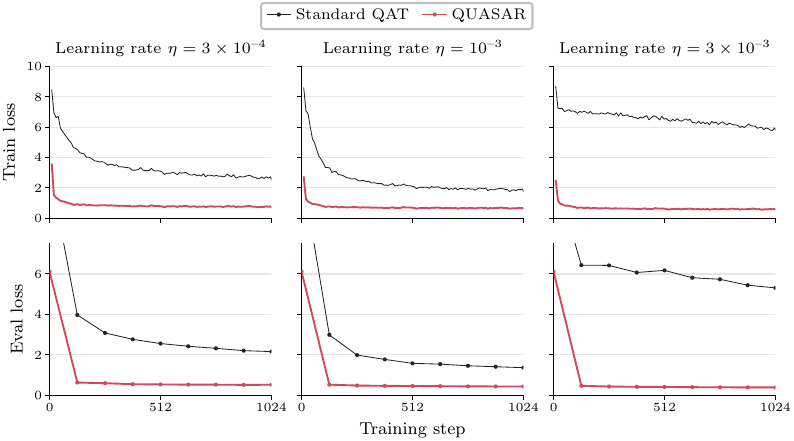}
\caption{\textbf{QUASAR retains its advantage under plain SGD.}
Training (top) and held-out (bottom) loss for Standard QAT and QUASAR at INT2
across three constant learning rates. QUASAR reaches lower loss floors at
every stable learning rate.}
\label{fig:sgd-ladder}
\end{figure}

\FloatBarrier

\section{Experimental Details}\label{app:setup}

This appendix specifies the training and evaluation protocols for
Section~\ref{sec:heal-distilled}. Format-specific algorithms are given in
Appendix~\ref{app:formats}, and released checkpoints are listed in
Appendix~\ref{app:released-artifacts}.

\paragraph{Quantization scope.}
The controlled INT experiments use asymmetric weight-only INT4, INT3, or
INT2 quantization with group size 128. We quantize every linear projection in
the transformer blocks and keep activations, embeddings, the language-model
head, and norms in BF16. Healing updates only the quantized linear weights;
adaptation updates all student parameters.

\paragraph{Healing.}
We apply QAD for 4{,}096 steps at sequence length 4{,}096. Qwen3-4B uses
global batch 32 and learning rates \(2,3,5\times10^{-5}\) at INT4, INT3, and
INT2. Llama-3.1-8B uses global batch 128 and learning rates
\(1,2,3\times10^{-5}\). This corresponds to roughly 60k supervised tokens
per step. Training uses KL distillation from the full-precision teacher.
QUASAR uses AdamW's second moment as saliency and searches clipping factors
from 0.30 to 1.00 in increments of 0.05.

\paragraph{Adaptation.}
We train Qwen3-4B-Base on OpenMathReasoning for 4{,}096 steps at sequence
length 18{,}432 and global batch 8. All methods use cross-entropy, AdamW with
learning rate \(5\times10^{-5}\), a cosine schedule with 1\% warmup, no
weight decay, and gradient clipping at 1. QUASAR uses the same integer
quantizer and scale-search grid as in healing.

\paragraph{NVFP4.}
NVFP4 QAD uses AdamW with learning rate \(10^{-6}\) and global batch 32.
We materialize the final checkpoints and evaluate them through vLLM's native
NVFP4 path~\citep{pagedattention,vllm-modelopt-docs}.

\paragraph{Baselines.}
Standard QAT uses fixed min--max ranges; LSQ~\citep{lsq} learns the scales;
Denoising QAT~\citep{baseline} refits the dequantization parameters;
BitDistiller~\citep{bitdistiller} combines clipping with self-distillation;
RTN rounds to the nearest grid point; GPTQ~\citep{gptq} minimizes a
second-order reconstruction error; and AWQ~\citep{awq} protects
activation-salient weights.

\paragraph{Evaluation protocol.}
For healing, we report fidelity to the full-precision model and downstream
accuracy. For adaptation, we report held-out perplexity and accuracy on five
math benchmarks. NVFP4 results include fidelity, reasoning, coding, and
long-context evaluations. Full task protocols are given with the
corresponding results in Appendices~\ref{app:healing},
\ref{app:adaptation}, and~\ref{app:nvfp4}.

\section{Additional Healing Results}\label{app:healing}

This appendix extends the healing results of
Section~\ref{sec:exp-healing} with additional fidelity and downstream
metrics, training dynamics, and an LLM-judge analysis of INT2 response
quality. It also reports representative generations and QUASAR ablations.
NVFP4 results are in Appendix~\ref{app:nvfp4}.

\subsection{Standard Tasks and Fidelity}

\begin{table}[t]\centering\fontsize{6}{7}\selectfont
\renewcommand{\arraystretch}{0.74}
\setlength{\aboverulesep}{0.2ex}\setlength{\belowrulesep}{0.3ex}
\caption{\textbf{QUASAR gives the highest fidelity in all six model--bit-width settings.}
Held-out forward KL, top-1 agreement, and accuracy on eight benchmarks for
QAD of Qwen3-4B-Thinking-2507 and Llama-3.1-8B-Instruct.}
\label{tab:heald-summary}
\newlength{\colw}\setlength{\colw}{0.033\textwidth}\newlength{\klw}\setlength{\klw}{0.046\textwidth}
\setlength{\tabcolsep}{1.1pt}
\begin{tabular}{l >{\columncolor{black!8}\raggedleft\arraybackslash}p{\klw} >{\columncolor{black!8}\raggedleft\arraybackslash}p{\colw} >{\raggedleft\arraybackslash}p{\colw} >{\raggedleft\arraybackslash}p{\colw} >{\raggedleft\arraybackslash}p{\colw} >{\raggedleft\arraybackslash}p{\colw} >{\raggedleft\arraybackslash}p{\colw} >{\raggedleft\arraybackslash}p{\colw} >{\raggedleft\arraybackslash}p{\colw} >{\raggedleft\arraybackslash}p{\colw} >{\columncolor{black!8}\raggedleft\arraybackslash}p{\colw} !{\hspace{1.5pt}{\color{black!25}\vrule}\hspace{1.5pt}} >{\columncolor{black!8}\raggedleft\arraybackslash}p{\klw} >{\columncolor{black!8}\raggedleft\arraybackslash}p{\colw} >{\raggedleft\arraybackslash}p{\colw} >{\raggedleft\arraybackslash}p{\colw} >{\raggedleft\arraybackslash}p{\colw} >{\raggedleft\arraybackslash}p{\colw} >{\raggedleft\arraybackslash}p{\colw} >{\raggedleft\arraybackslash}p{\colw} >{\raggedleft\arraybackslash}p{\colw} >{\raggedleft\arraybackslash}p{\colw} >{\columncolor{black!8}\raggedleft\arraybackslash}p{\colw}}
\toprule
& \multicolumn{11}{c}{Qwen3-4B-Thinking-2507} & \multicolumn{11}{c}{Llama-3.1-8B-Instruct} \\
\cmidrule(lr){2-12}\cmidrule(lr){13-23}
Method & \multicolumn{1}{>{\columncolor{black!8}}c}{\rotatebox{90}{KL $\downarrow$}} & \multicolumn{1}{>{\columncolor{black!8}}c}{\rotatebox{90}{Top-1 $\uparrow$}} & \multicolumn{1}{c}{\rotatebox{90}{GSM8K}} & \multicolumn{1}{c}{\rotatebox{90}{MMLU}} & \multicolumn{1}{c}{\rotatebox{90}{ARC-C}} & \multicolumn{1}{c}{\rotatebox{90}{ARC-E}} & \multicolumn{1}{c}{\rotatebox{90}{HellaSwag}} & \multicolumn{1}{c}{\rotatebox{90}{WinoGrande}} & \multicolumn{1}{c}{\rotatebox{90}{TruthfulQA}} & \multicolumn{1}{c}{\rotatebox{90}{IFEval}} & \multicolumn{1}{>{\columncolor{black!8}}c}{\rotatebox{90}{\textbf{Avg.}}} & \multicolumn{1}{>{\columncolor{black!8}}c}{\rotatebox{90}{KL $\downarrow$}} & \multicolumn{1}{>{\columncolor{black!8}}c}{\rotatebox{90}{Top-1 $\uparrow$}} & \multicolumn{1}{c}{\rotatebox{90}{GSM8K}} & \multicolumn{1}{c}{\rotatebox{90}{MMLU}} & \multicolumn{1}{c}{\rotatebox{90}{ARC-C}} & \multicolumn{1}{c}{\rotatebox{90}{ARC-E}} & \multicolumn{1}{c}{\rotatebox{90}{HellaSwag}} & \multicolumn{1}{c}{\rotatebox{90}{WinoGrande}} & \multicolumn{1}{c}{\rotatebox{90}{TruthfulQA}} & \multicolumn{1}{c}{\rotatebox{90}{IFEval}} & \multicolumn{1}{>{\columncolor{black!8}}c}{\rotatebox{90}{\textbf{Avg.}}} \\
\midrule
Teacher (FP16) & -- & -- & 87.0 & 68.7 & 53.1 & 76.7 & 65.7 & 66.1 & 57.5 & 54.3 & 66.1 & -- & -- & 70.1 & 68.3 & 55.6 & 79.9 & 79.5 & 73.6 & 54.5 & 73.9 & 69.4 \\
\midrule
\multicolumn{23}{l}{\textit{2-bit (INT2)}}\\
RTN & 11.166 & 0.5 & 0.0 & 25.4 & 25.6 & 25.9 & 26.0 & 52.0 & 46.8 & 8.3 & 26.3 & 10.810 & 0.5 & 0.0 & 25.5 & 26.2 & 26.1 & 26.1 & 51.8 & 47.6 & 11.8 & 26.9 \\
GPTQ & 1.140 & 67.7 & 0.2 & 23.4 & 25.1 & 30.0 & 30.1 & 49.3 & \textbf{52.9} & 8.9 & 27.5 & 2.380 & 51.8 & 0.0 & 24.9 & 23.3 & 28.9 & 30.0 & 48.0 & 49.0 & 9.1 & 26.7 \\
AWQ & 3.016 & 37.0 & 0.0 & 23.4 & 23.2 & 31.6 & 29.1 & 51.1 & \underline{52.8} & 9.6 & 27.6 & 6.769 & 12.9 & 0.0 & 23.6 & 23.3 & 25.9 & 27.2 & 48.2 & 48.8 & 7.9 & 25.6 \\
\arrayrulecolor{black!25}\cmidrule[0.4pt](lr){1-23}\arrayrulecolor{black}
Standard QAT & 0.186 & 89.9 & 47.4 & 34.5 & 31.2 & 44.8 & 46.9 & 54.5 & 51.1 & 33.5 & 43.0 & 0.163 & 90.1 & 1.8 & 23.4 & 22.9 & 30.9 & 34.8 & 49.3 & 46.8 & 44.4 & 31.8 \\
LSQ & 0.205 & 89.2 & 42.6 & 34.6 & 30.6 & 44.9 & 47.6 & 55.2 & 47.2 & 35.1 & 42.2 & 0.172 & 89.6 & 42.2 & 39.7 & 32.8 & 52.8 & 55.7 & 53.5 & 47.4 & 42.1 & 45.8 \\
Denoising QAT & \underline{0.179} & \underline{90.1} & 45.6 & 36.0 & 30.1 & 47.9 & 46.7 & 56.1 & 49.9 & 34.9 & 43.4 & 0.165 & 90.1 & 0.0 & 24.6 & 26.2 & 25.6 & 26.9 & 50.0 & \underline{49.1} & 40.9 & 30.4 \\
BitDistiller & 0.209 & 89.1 & \underline{49.0} & \underline{43.2} & \underline{38.0} & \textbf{60.1} & \underline{52.6} & \underline{57.9} & 52.6 & \underline{37.7} & \underline{48.9} & \underline{0.151} & \underline{90.2} & \underline{53.1} & \underline{49.3} & \underline{43.5} & \underline{69.9} & \underline{66.6} & \underline{66.3} & 47.7 & \underline{51.8} & \underline{56.0} \\
\rowcolor{oursrow}
\textbf{QUASAR}~(ours) & \textbf{0.126} & \textbf{92.0} & \textbf{68.8} & \textbf{48.9} & \textbf{38.8} & \underline{55.9} & \textbf{56.4} & \textbf{59.2} & 50.5 & \textbf{47.5} & \textbf{53.2} & \textbf{0.105} & \textbf{92.0} & \textbf{66.4} & \textbf{52.5} & \textbf{43.8} & \textbf{73.1} & \textbf{68.1} & \textbf{66.8} & \textbf{49.5} & \textbf{56.0} & \textbf{59.5} \\
\midrule
\multicolumn{23}{l}{\textit{3-bit (INT3)}}\\
RTN & 0.597 & 78.8 & 20.5 & 53.6 & 39.4 & 56.1 & 54.1 & 58.1 & 51.4 & 18.9 & 44.0 & 0.276 & 85.6 & 19.0 & 44.8 & 40.2 & 64.6 & 69.5 & 66.9 & 47.0 & 52.9 & 50.6 \\
GPTQ & 0.130 & 91.5 & 77.3 & 61.0 & 48.0 & \underline{75.5} & 61.0 & 61.8 & 54.0 & 49.2 & 61.0 & 0.076 & 92.8 & 61.4 & 59.2 & 49.1 & 71.6 & 75.7 & 71.0 & \textbf{53.1} & 69.1 & 63.8 \\
AWQ & 0.240 & 88.1 & 73.2 & 60.1 & 46.4 & 72.3 & 60.7 & 62.7 & 49.2 & 41.8 & 58.3 & 0.197 & 88.0 & 23.9 & 53.2 & 43.0 & 67.6 & 72.2 & 67.6 & 42.1 & 58.6 & 53.5 \\
\arrayrulecolor{black!25}\cmidrule[0.4pt](lr){1-23}\arrayrulecolor{black}
Standard QAT & 0.062 & 94.2 & 78.5 & 62.0 & 46.8 & 74.1 & 61.9 & 64.4 & \underline{56.1} & \underline{53.2} & 62.1 & 0.046 & 94.7 & 70.7 & 59.7 & 47.4 & 72.8 & 75.2 & \textbf{72.9} & 52.1 & 69.9 & 65.1 \\
LSQ & 0.070 & 93.8 & \textbf{83.0} & 62.3 & \textbf{49.4} & \textbf{75.8} & 62.7 & \textbf{65.6} & 55.4 & 52.9 & \textbf{63.4} & 0.052 & 94.2 & \textbf{76.8} & \underline{62.3} & 51.8 & 76.9 & 75.8 & 71.6 & 50.9 & 67.5 & 66.7 \\
Denoising QAT & \underline{0.060} & \underline{94.3} & 80.4 & 62.7 & 46.2 & 73.5 & 61.6 & 64.2 & \underline{56.1} & 51.0 & 62.0 & 0.046 & 94.7 & 68.8 & 59.8 & 48.0 & 74.5 & 75.0 & \underline{72.2} & 51.0 & \underline{70.2} & 64.9 \\
BitDistiller & 0.065 & 94.2 & 81.0 & \underline{63.7} & 47.0 & 69.0 & \underline{63.0} & 63.5 & \textbf{56.3} & \textbf{53.4} & 62.1 & \underline{0.041} & \underline{94.8} & \underline{74.2} & 61.6 & \underline{53.1} & \underline{79.0} & \underline{76.4} & 71.9 & \underline{52.2} & 68.6 & \underline{67.1} \\
\rowcolor{oursrow}
\textbf{QUASAR}~(ours) & \textbf{0.054} & \textbf{94.9} & \underline{81.7} & \textbf{64.5} & \underline{48.3} & 72.9 & \textbf{63.2} & \underline{64.7} & 55.5 & 50.6 & \underline{62.7} & \textbf{0.036} & \textbf{95.2} & 74.1 & \textbf{63.3} & \textbf{54.9} & \textbf{79.1} & \textbf{77.1} & 71.0 & 51.6 & \textbf{71.5} & \textbf{67.9} \\
\midrule
\multicolumn{23}{l}{\textit{4-bit (INT4)}}\\
RTN & 0.108 & 92.2 & 80.2 & 66.1 & 48.8 & 69.7 & 63.5 & 65.4 & \textbf{58.4} & 54.2 & 63.3 & 0.044 & 94.4 & 70.7 & 64.1 & 53.3 & 76.9 & 78.0 & 73.9 & 50.7 & 72.1 & 67.5 \\
GPTQ & 0.027 & 96.2 & 84.9 & 67.1 & 51.4 & \underline{75.0} & 64.7 & 65.4 & 56.5 & 51.4 & 64.6 & 0.015 & 96.8 & 69.0 & \textbf{67.3} & 52.0 & 77.3 & \underline{78.7} & 73.2 & \textbf{55.1} & \underline{73.6} & 68.3 \\
AWQ & 0.050 & 94.7 & 84.6 & 67.6 & \underline{51.5} & 74.9 & \textbf{65.2} & \textbf{66.6} & 56.7 & \underline{56.6} & \underline{65.5} & 0.032 & 95.2 & 65.6 & 66.6 & 53.7 & 78.3 & \underline{78.7} & 72.4 & 53.7 & 73.4 & 67.8 \\
\arrayrulecolor{black!25}\cmidrule[0.4pt](lr){1-23}\arrayrulecolor{black}
Standard QAT & 0.019 & \underline{96.9} & 84.9 & 67.0 & 49.7 & 73.5 & \underline{65.1} & 65.2 & 56.7 & 54.3 & 64.6 & 0.014 & \underline{97.0} & \underline{76.3} & 66.3 & 54.4 & 77.7 & 78.4 & 73.8 & 53.9 & 73.0 & 69.2 \\
LSQ & 0.024 & 96.4 & \underline{86.0} & 66.8 & 50.6 & \underline{75.0} & \textbf{65.2} & 65.6 & 57.1 & 56.4 & 65.3 & 0.016 & 96.8 & 73.2 & 66.5 & 54.6 & \textbf{80.1} & \underline{78.7} & \underline{74.1} & 53.2 & \textbf{74.1} & \underline{69.3} \\
Denoising QAT & \underline{0.018} & \underline{96.9} & 85.0 & 67.1 & 49.2 & 72.9 & 64.7 & 65.7 & 57.3 & \textbf{56.9} & 64.9 & 0.014 & \underline{97.0} & \textbf{77.0} & 66.2 & \textbf{55.7} & 77.8 & 78.6 & 73.9 & 53.6 & 72.1 & \underline{69.3} \\
BitDistiller & \underline{0.018} & \underline{96.9} & 84.8 & \underline{67.7} & \textbf{52.6} & \textbf{76.0} & 64.8 & \underline{65.9} & 57.3 & 54.3 & 65.4 & \underline{0.013} & \underline{97.0} & 70.4 & \underline{67.0} & 55.1 & \underline{79.7} & \textbf{78.8} & 73.3 & 53.8 & \textbf{74.1} & 69.0 \\
\rowcolor{oursrow}
\textbf{QUASAR}~(ours) & \textbf{0.016} & \textbf{97.1} & \textbf{86.4} & \textbf{68.0} & 50.6 & \textbf{76.0} & \underline{65.1} & 65.2 & \underline{57.7} & 55.8 & \textbf{65.6} & \textbf{0.012} & \textbf{97.3} & 74.6 & 66.5 & \underline{55.5} & 78.9 & \textbf{78.8} & \textbf{74.2} & \underline{54.2} & 72.5 & \textbf{69.4} \\
\bottomrule
\end{tabular}
\end{table}

\begin{figure}[!htb]
\centering
\includegraphics[width=0.9\linewidth]{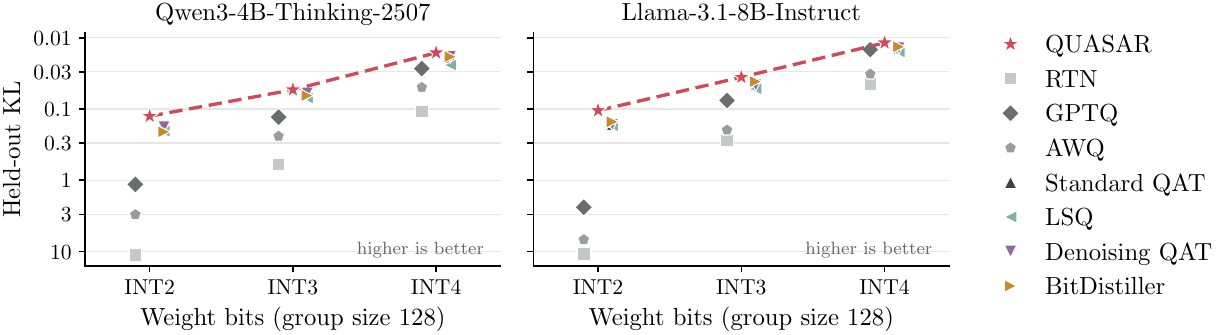}
\caption{\textbf{QUASAR gives the lowest held-out KL across both models and all bit widths.}
Forward KL to the full-precision model after QAD of
Qwen3-4B-Thinking-2507 (left) and Llama-3.1-8B-Instruct (right).
Lower is better.}
\label{fig:healf-matrix-kl}
\end{figure}

\begin{figure}[!htb]
\centering
\includegraphics[width=0.9\linewidth]{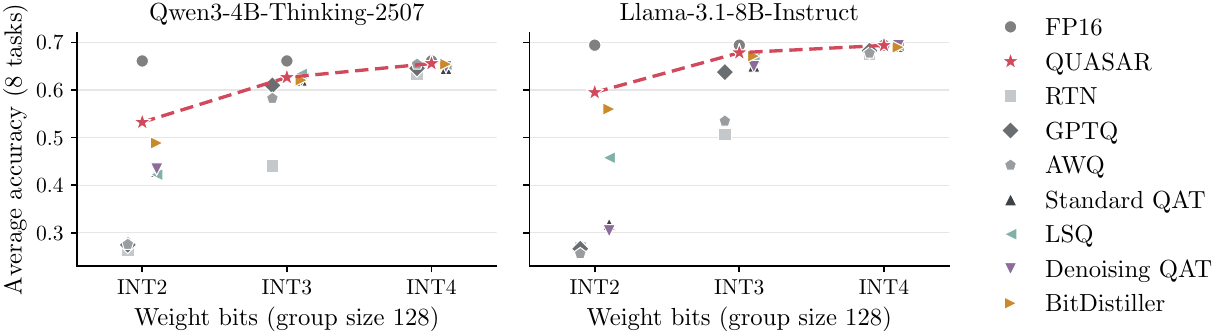}
\caption{\textbf{QUASAR attains the highest INT2 average accuracy on both models.}
Average accuracy over the eight benchmarks of
Table~\ref{tab:heald-summary} versus bit width for
Qwen3-4B-Thinking-2507 (left) and Llama-3.1-8B-Instruct (right).
Gray circles mark the full-precision models.}
\label{fig:healf-matrix-cap}
\end{figure}

\begin{figure}[!htb]
\centering
\includegraphics[width=0.78\linewidth]{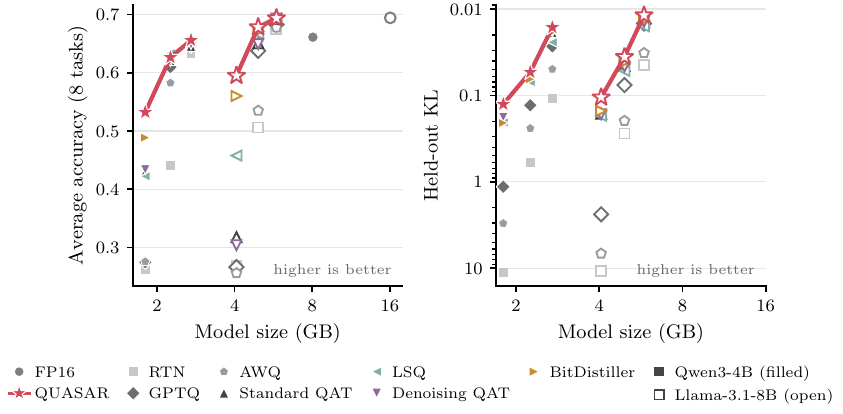}
\caption{\textbf{QUASAR gives the best fidelity at each quantized model size.}
Average accuracy over eight benchmarks (left) and held-out KL to the
full-precision model (right) versus model size. Filled markers denote
Qwen3-4B-Thinking-2507, open markers denote Llama-3.1-8B-Instruct, and gray
circles mark the full-precision models.}
\label{fig:healf-size}
\end{figure}

\begin{figure}[!htb]
\centering
\includegraphics[width=0.92\linewidth]{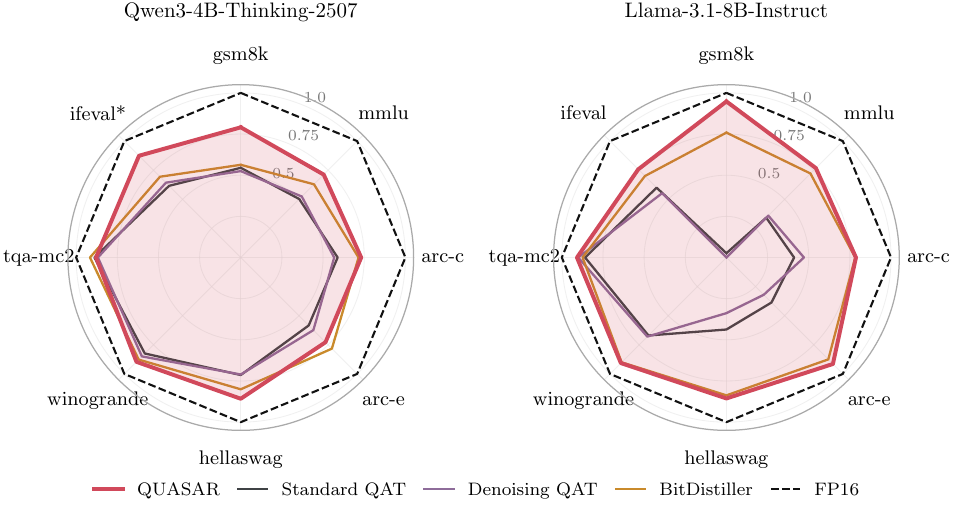}
\caption{\textbf{QUASAR preserves INT2 accuracy broadly across tasks.}
Per-task accuracy normalized by the full-precision model for
Qwen3-4B-Thinking-2507 (left) and Llama-3.1-8B-Instruct (right) on the eight
benchmarks of Table~\ref{tab:heald-summary}.}
\label{fig:healf-radar}
\end{figure}

\begin{figure}[h]
\centering
\includegraphics[width=0.45\linewidth]{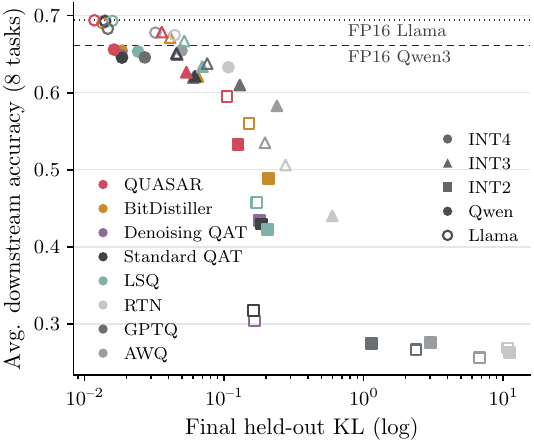}
\caption{\textbf{Lower held-out KL correlates with higher downstream accuracy.}
Average accuracy over eight benchmarks versus final forward KL for the QAT
and PTQ methods of Table~\ref{tab:heald-summary}, across both models and all
three bit widths. Dashed lines mark full-precision accuracy.}
\label{fig:kl-vs-acc}
\end{figure}

\FloatBarrier
\subsection{Training Dynamics}

\begin{figure}[h]
\centering
\includegraphics[width=\linewidth]{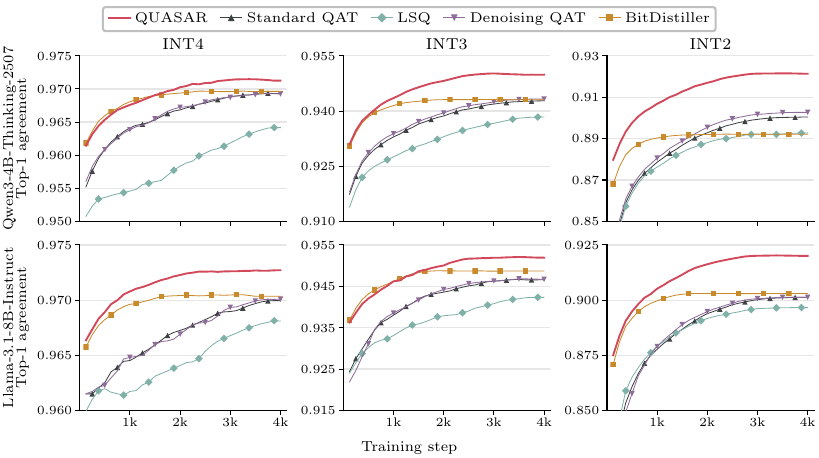}
\caption{\textbf{QUASAR reaches the highest final top-1 agreement in all six settings.}
Top-1 agreement with the full-precision model during QAD of
Qwen3-4B-Thinking-2507 (top) and Llama-3.1-8B-Instruct (bottom) at INT4,
INT3, and INT2.}
\label{fig:healf-agree}
\end{figure}

\begin{figure}[h]
\centering
\includegraphics[width=0.6\linewidth]{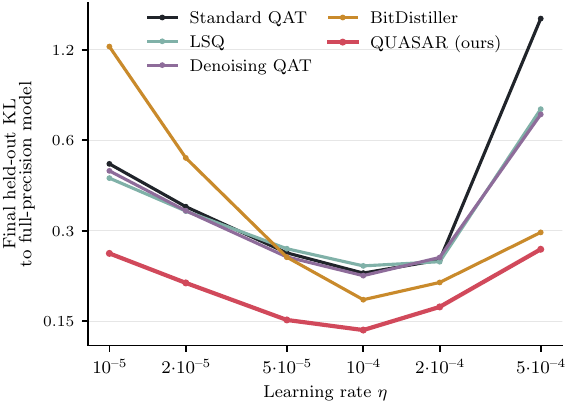}
\caption{\textbf{QUASAR is robust to learning rate.}
Final held-out KL after 1{,}000 steps of INT2 QAD of
Qwen3-4B-Thinking-2507 across a \(50\times\) learning-rate range. QUASAR
attains the lowest KL at every tested rate; Standard QAT diverges at high
rates, while BitDistiller degrades at low rates.}
\label{fig:lr-robustness}
\end{figure}

\begin{figure}[h]
\centering
\includegraphics[width=0.8\linewidth]{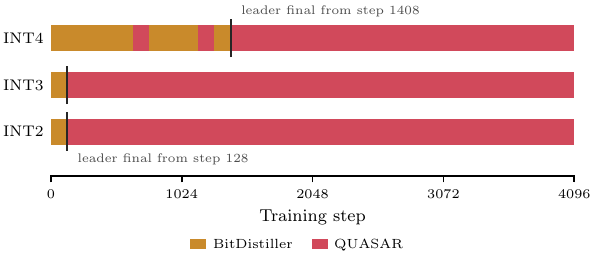}
\caption{\textbf{QUASAR becomes the final leader early in training.}
Each strip identifies the method with the lowest held-out KL among six methods
for Qwen3-4B-Thinking-2507. QUASAR leads from step 128 at INT3 and INT2, and
from step 1{,}408 at INT4, through the end of training.}
\label{fig:horizon}
\end{figure}

\FloatBarrier
\subsection{LLM-as-a-Judge Preference at 2 Bits}
Across 128 WildChat prompts per model, Llama-3.3-70B-Instruct prefers QUASAR
to all seven quantized baselines for both models
(Figure~\ref{fig:judge-w2}; Table~\ref{tab:judge-w2}). The margin is largest
against PTQ, whose INT2 responses are often repetitive or incomplete
(Figure~\ref{fig:judge-w2-health}). The full-precision models remain preferred.
Figure~\ref{fig:heal-transcript} shows one representative response.

\begin{figure}[!htb]
\centering
\includegraphics[width=\linewidth]{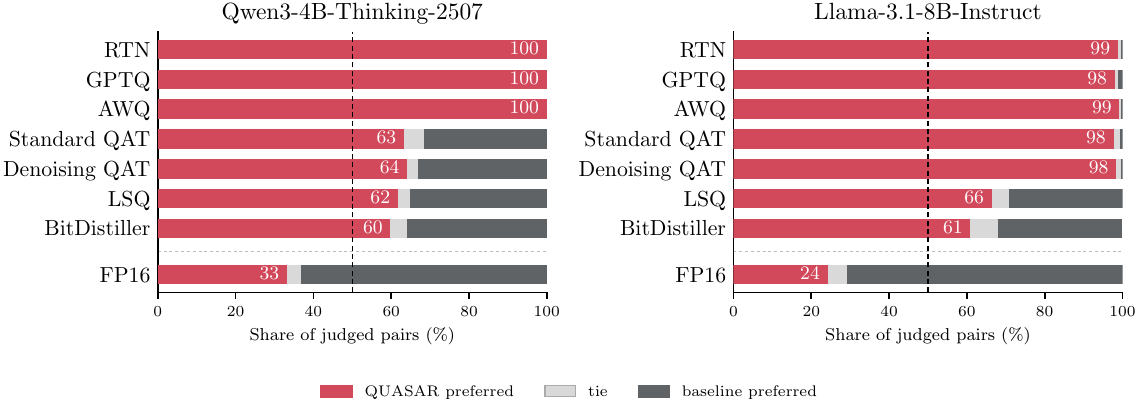}
\caption{\textbf{The judge prefers QUASAR to every INT2 quantized baseline.}
Pairwise win, tie, and loss rates on responses to 128 WildChat prompts for
Qwen3-4B-Thinking-2507 and Llama-3.1-8B-Instruct.
Llama-3.3-70B-Instruct serves as the judge.}
\label{fig:judge-w2}
\end{figure}

\begin{table}[!htb]\centering\footnotesize
\renewcommand{\arraystretch}{1.02}
\caption{\textbf{QUASAR wins all 14 comparisons with quantized baselines.}
Mean per-prompt preference in \([-1,1]\), with 95\% bootstrap confidence
intervals over 128 prompts. Positive values favor QUASAR. The final row
compares QUASAR with the full-precision model.}
\label{tab:judge-w2}
\setlength{\tabcolsep}{8pt}
\begin{tabular}{@{}l rr@{}}
\toprule
vs.\ QUASAR & Qwen3-4B-Thinking-2507 & Llama-3.1-8B-Instruct \\
\midrule
RTN & $+1.00$ \([+1.00, +1.00]\) & $+0.98$ \([+0.96, +1.00]\) \\
GPTQ & $+1.00$ \([+1.00, +1.00]\) & $+0.97$ \([+0.94, +0.99]\) \\
AWQ & $+1.00$ \([+1.00, +1.00]\) & $+0.99$ \([+0.97, +1.00]\) \\
\arrayrulecolor{black!25}\cmidrule[0.4pt](lr){1-3}\arrayrulecolor{black}
Standard QAT & $+0.32$ \([+0.21, +0.42]\) & $+0.97$ \([+0.94, +0.99]\) \\
LSQ & $+0.27$ \([+0.15, +0.38]\) & $+0.37$ \([+0.25, +0.49]\) \\
Denoising QAT & $+0.31$ \([+0.21, +0.41]\) & $+0.98$ \([+0.95, +1.00]\) \\
BitDistiller & $+0.24$ \([+0.12, +0.35]\) & $+0.29$ \([+0.17, +0.41]\) \\
\arrayrulecolor{black!25}\cmidrule[0.4pt](lr){1-3}\arrayrulecolor{black}
FP16 & $-0.30$ \([-0.40, -0.20]\) & $-0.46$ \([-0.56, -0.37]\) \\
\bottomrule
\end{tabular}
\end{table}

\begin{figure}[h]
\centering
\includegraphics[width=0.85\linewidth]{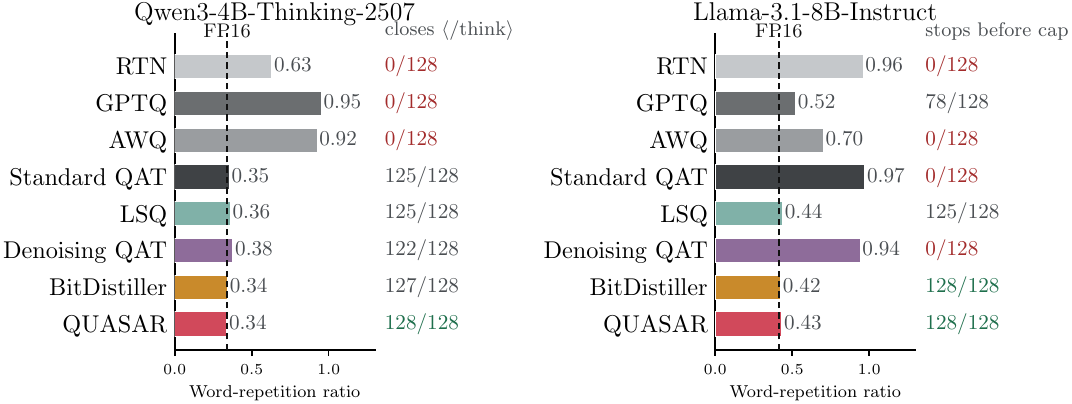}
\caption{\textbf{QUASAR matches the full-precision models' response stability at INT2.}
Mean word-repetition ratio over 128 WildChat prompts for
Qwen3-4B-Thinking-2507 (left) and Llama-3.1-8B-Instruct (right), with greedy
decoding. Dashed lines mark the full-precision models; counts show responses
that close the reasoning block for Qwen or stop before the token cap for
Llama.}
\label{fig:judge-w2-health}
\end{figure}

\begin{figure}[p]\centering\scriptsize
\begin{tabular}{@{}p{0.15\linewidth}p{0.81\linewidth}@{}}
\toprule
\multicolumn{2}{@{}l}{\textbf{Prompt:} \emph{Write a story about a man being stranded on an island.}}\\
\midrule
\textbf{FP16} & The salt stung Leo's eyes before he even registered the crash. One moment, he was shouting into the cockpit of his battered sedan, the next, the world\,\ldots \\[2pt]
\midrule\multicolumn{2}{@{}l}{\textit{2-bit (INT2)}}\\[1pt]
RTN & max o mil MD o. over md edge md mdish ai \} ats \} \} further display $\circ$scs \textless{}igrated found "arql eventually o, same, if ( cs1) \textless{} \textless{}$\circ$ \textless{}ql \textless{}, \textbackslash ql o \textless{} \textless{}! ( \textless{} \,\ldots \\[2pt]
GPTQ & *A [R] :Help "l" (or: "a" in 1000+). 0 * F : 0 (To the 0) : 0 : 0 : 0 : 0 : 0 : 0 : 0 : 0 : 0 : 10000: 0 : 10000: 0 : 0 : 0 : 0 : 0 : 0 : 0 : 0 : 0 : \,\ldots \\[2pt]
AWQ & go \{r=0 \{r=1, 100 100 100 100$\circ$REEN$\circ$.1 $\circ$$\circ$100$\circ$$\circ$$\circ$ 100$\circ$ 10 100. 100 10 10 1 10 1 1 1 0 1 100 1 1 1 1 1 ( 1 $\circ$ 1 2 1 02 1 1 1 1 1 1 1. 1 1 1 1 1 1 1 1 1 1 1\,\ldots \\[2pt]
Standard QAT & The first thing Elias noticed wasn't the ocean, but the *absence* of it. Not the roar of waves, but the silence of his own thoughts, a hollow space wh\,\ldots \\[2pt]
Denoising QAT & Leo woke to the sound of water. Not the gentle lapping of a tide, but a relentless, crashing roar against the rocks of a tiny, rocky island he'd never\,\ldots \\[2pt]
LSQ & Leo woke to the smell of salt and wet earth. The first thing he noticed wasn't the storm, but the *absence* of sound. The roar of the ocean that had b\,\ldots \\[2pt]
BitDistiller & Kai's first thought wasn't panic, but a slow, heavy dread. The salt air stung his eyes, thick with the scent of wet earth and something else---something\,\ldots \\[2pt]
\rowcolor{oursrow}
\textbf{QUASAR (ours)} & The last thing Alex remembered was the *thud* of the car door slamming shut, the roar of the storm tearing through the windshield, and the sickening l\,\ldots \\[2pt]
\midrule\multicolumn{2}{@{}l}{\textit{3-bit (INT3)}}\\[1pt]
RTN & Okay, the user wants a story about a man stranded on an island. Hmm, let's unpack that. First, I should consider what they might need. Are they writin\,\ldots \\[2pt]
GPTQ & The sea had taken him. Not with a crash, but with a slow, suffocating pull. Elias had been a fisherman, not a sailor, and the *Mariana*---his small, stu\,\ldots \\[2pt]
AWQ & The storm had been a fury of glass and thunder, tearing the sea from its bones. Elias woke to the *crack* of splintered wood and the sharp, metallic s\,\ldots \\[2pt]
Standard QAT & Elias woke to the sound of the ocean, not the crash of waves, but the *thump-thump-thump* of his own heart against his ribs. Salt stung his eyes, and \,\ldots \\[2pt]
Denoising QAT & Leo woke to the smell of salt and something else -- wet earth, maybe, or decaying leaves. His head throbbed, a dull ache behind his eyes. He was lying \,\ldots \\[2pt]
LSQ & \#\# The Island That Forgot to Breathe The last thing Elias remembered was the *thump* of the boat hitting the water, the sickening lurch of the world s\,\ldots \\[2pt]
BitDistiller & Leo's last memory was the *thwack* of a wave against the hull of the *Marlin*, the sharp sting of salt in his eyes, and the terrifying, silent *thud* \,\ldots \\[2pt]
\rowcolor{oursrow}
\textbf{QUASAR (ours)} & Leo's last memory was the *scream* of the engine, the sickening lurch as the boat plunged into the churning grey sea, and the world dissolving into a \,\ldots \\[2pt]
\midrule\multicolumn{2}{@{}l}{\textit{4-bit (INT4)}}\\[1pt]
RTN & Leo didn't know how he got there. One moment he was a man in a sleek, unremarkable boat, the next---*thud*---a violent wave slammed the hull, the engine d\,\ldots \\[2pt]
GPTQ & Elias's world dissolved in a roar. One moment, he was gripping the steering wheel of his aging sedan, the rain lashing the windshield like a thousand \,\ldots \\[2pt]
AWQ & Leo's last thought before the world dissolved into water and silence was the *thump* of the boat's hull against the rocks. Then, the roar of the ocean\,\ldots \\[2pt]
Standard QAT & Leo didn't know how long he'd been adrift when the first wave of panic hit. It wasn't the cold, the salt, or the crushing weight of the ocean's indiff\,\ldots \\[2pt]
Denoising QAT & Leo woke to the sound of the ocean, not the roar of a storm, but the low, insistent *thrum* of the sea against the rocks. His head throbbed, a dull ac\,\ldots \\[2pt]
LSQ & Leo's last memory was the *thud* of the plane hitting the water, not the roar of the engines or the frantic shouts of the crew. He woke up gasping, th\,\ldots \\[2pt]
BitDistiller & The crash wasn't loud. It was a sickening *thud* of metal and earth, followed by a silence so profound it felt like the island itself had held its bre\,\ldots \\[2pt]
\rowcolor{oursrow}
\textbf{QUASAR (ours)} & The impact wasn't a crash. It was a *sigh* of the world ending. Ben's world dissolved into a sickening lurch, the roar of the ocean swallowing his scr\,\ldots \\[2pt]
\bottomrule
\end{tabular}
\caption{The beginning of each method's response to one held-out prompt,
for Qwen3-4B-Thinking-2507 at INT2, INT3, and INT4 with greedy decoding.}
\label{fig:heal-transcript}
\end{figure}

\FloatBarrier
\subsection{Ablations}

\begin{table}[!htb]\centering\scriptsize
\renewcommand{\arraystretch}{0.85}
\setlength{\aboverulesep}{0.2ex}\setlength{\belowrulesep}{0.3ex}
\caption{Ablation Study: Effect of QUASAR's scale search and weighted dequantization fit on
Qwen3-4B-Thinking-2507 with INT2 QAD. We report final held-out KL and its
change from Standard QAT.}
\label{tab:abl-main}
\setlength{\tabcolsep}{5pt}
\begin{tabular}{l cccccc>{\columncolor{oursrow}}c}
\toprule
Variant & Standard QAT & Adam fit & Init.\ only & Uniform fit & Fisher fit & \textbf{QUASAR}~(ours) \\
\midrule
Scale search & none & none & at init. & every step & every step & every step \\
Saliency & none & Adam 2nd moment & uniform & uniform & Fisher & Adam 2nd moment \\
\midrule
KL $\downarrow$ & 0.175 & 0.164 & 0.172 & 0.114 & \underline{0.113} & \textbf{0.112} \\
$\Delta$ vs.\ Standard QAT & -- & $-6.1\%$ & $-1.6\%$ & $-34.7\%$ & \underline{$-35.5\%$} & $\mathbf{-35.9\%}$ \\
\bottomrule
\end{tabular}
\end{table}

\clearpage

\section{Additional Adaptation Results}\label{app:adaptation}

This appendix extends the adaptation results of
Section~\ref{sec:exp-adaptation} with training curves and three analyses of
the INT2 Qwen3-4B-Base checkpoints: answer outcomes, variation across
repeated samples, and divergence from FP-SFT along reasoning traces.

QUASAR's advantage is largest on benchmarks requiring long reasoning.
FP-SFT generates median responses of 27k tokens on AIME'24 and 20k on
HMMT'25. At INT2, QUASAR is the only quantized model to score above zero on
HMMT'25, suggesting that it better preserves the adapted behavior over long
generations.

\begin{figure}[!htb]
  \centering
  \includegraphics[width=\linewidth]{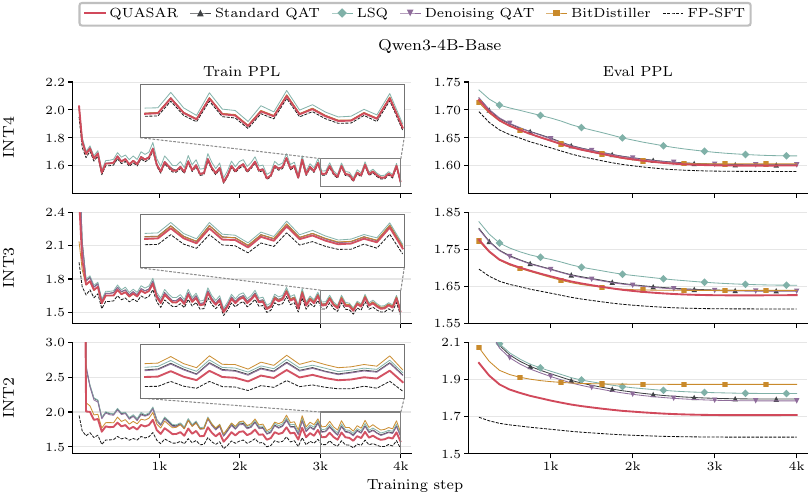}
  \caption{\textbf{QUASAR reaches the lowest held-out perplexity at every bit width.}
  Training perplexity over the full run (left) and final 1{,}000 steps (center),
  and held-out perplexity (right), for QAT of Qwen3-4B-Base on
  OpenMathReasoning. FP-SFT is included as a full-precision reference.}
  \label{fig:adapt-curves}
\end{figure}

\begin{figure}[!htb]
\centering
\includegraphics[width=\linewidth]{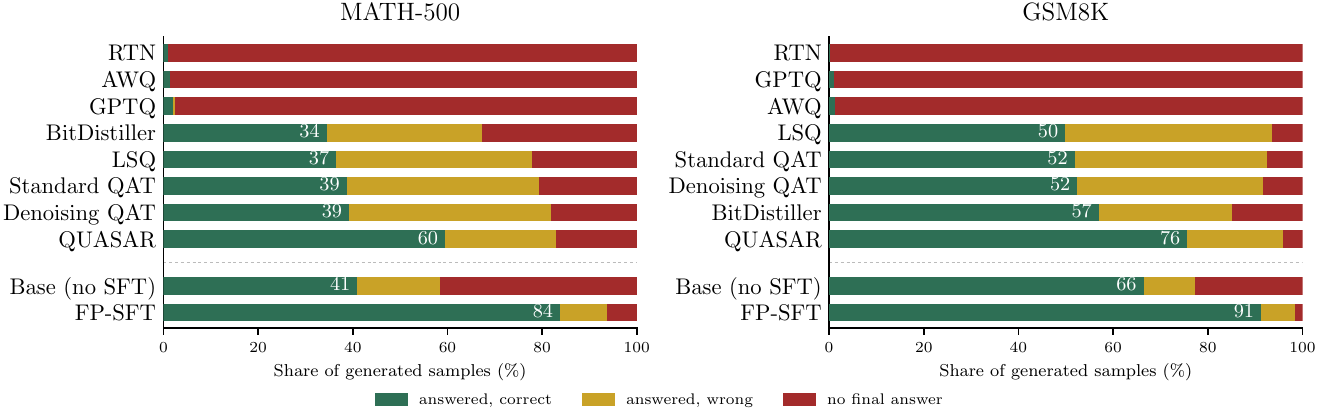}
\caption{\textbf{Answer outcomes for the INT2 checkpoints.}
Share of generated samples that produce a correct answer, an incorrect answer,
or no final answer on MATH-500 and GSM8K. The base model and FP-SFT are
included as references.}
\label{fig:adapt-failures}
\end{figure}

\begin{figure}[h]
\centering
\includegraphics[width=0.95\linewidth]{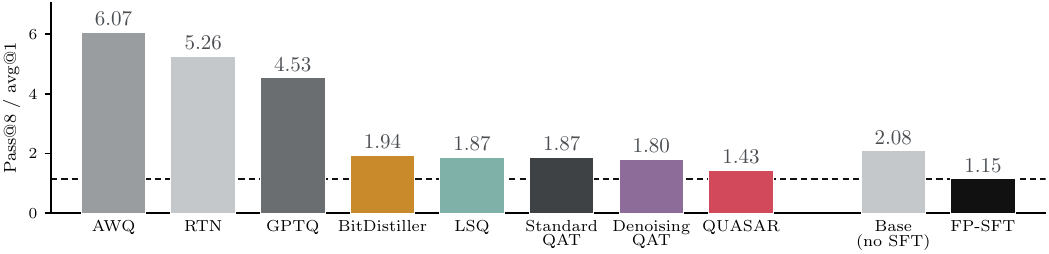}
\caption{\textbf{Variation across repeated INT2 samples.}
Ratio of pass@8 to avg@1 on MATH-500 for Qwen3-4B-Base checkpoints, using
eight samples per problem at temperature 0.6. A ratio of \(1\) means that the
same problems are solved on every attempt; higher values indicate less
consistent successes. The base model and FP-SFT are included as references.}
\label{fig:adapt-dispersion}
\end{figure}

\begin{figure}[h]
\centering
\includegraphics[width=0.82\linewidth]{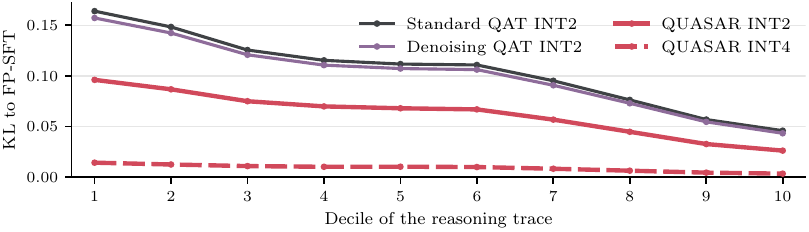}
\caption{\textbf{Divergence from FP-SFT along reasoning traces.}
Forward KL to FP-SFT by trace decile for selected quantized checkpoints,
averaged over 1{,}000 MATH-500 traces generated by FP-SFT. Every checkpoint
is evaluated on the same token sequences.}
\label{fig:adapt-divergence}
\end{figure}

\clearpage
\section{Additional NVFP4 Results}\label{app:nvfp4}

This appendix collects the training curves and the full evaluation
results for the NVFP4 experiments of Section~\ref{sec:exp-nvfp4}.

\begin{figure}[!htb]
\centering
\includegraphics[width=\linewidth]{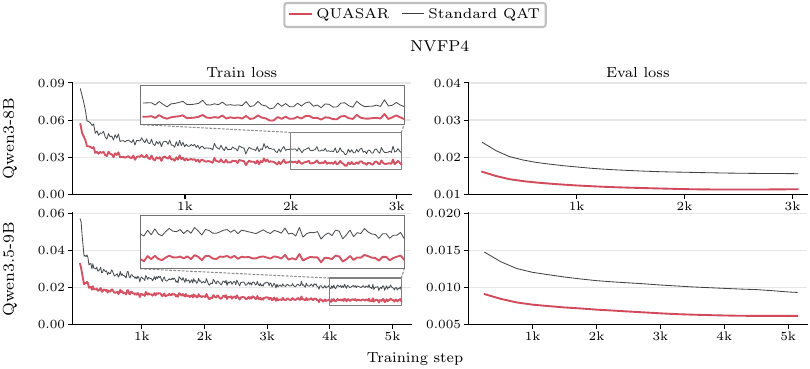}
\caption{\textbf{QUASAR reaches lower held-out loss for both NVFP4 models.}
Training forward KL over the full run (left) and final 1{,}000 steps (center),
and held-out forward KL (right), for NVFP4 QAD of Qwen3-8B (top) and
Qwen3.5-9B (bottom).}
\label{fig:nvfp4-curves}
\end{figure}

\begin{table}[!htb]\centering\scriptsize
\renewcommand{\arraystretch}{0.85}
\caption{\textbf{QUASAR gives the best average accuracy among the quantized checkpoints.}
Accuracy on nine downstream benchmarks for Qwen3-8B and
Qwen3.5-9B. The materialized NVFP4 checkpoints are evaluated through vLLM's
native W4A4 path.}
\label{tab:nvfp4-checkpoints-full}
\setlength{\tabcolsep}{0pt}
\begin{tabular}{>{\hspace*{2pt}}l<{\hspace*{2pt}} >{\columncolor{black!8}[.1pt][.1pt]\raggedleft\arraybackslash}p{\dimexpr0.053\textwidth+4pt\relax}<{\hspace*{2pt}} *{9}{>{\raggedleft\arraybackslash}p{\dimexpr0.053\textwidth+4pt\relax}<{\hspace*{2pt}}}}
\toprule
& \multicolumn{10}{c}{Downstream task accuracy (\%) $\uparrow$} \\
\cmidrule(l){2-11}
Method & \multicolumn{1}{>{\columncolor{black!8}[.1pt][.1pt]}c}{\textbf{Avg.}} & \multicolumn{1}{c}{GSM8K} & \multicolumn{1}{c}{MMLU} & \multicolumn{1}{c}{ARC-C} & \multicolumn{1}{c}{ARC-E} & \multicolumn{1}{c}{Hella.} & \multicolumn{1}{c}{WinoG.} & \multicolumn{1}{c}{TQA} & \multicolumn{1}{c}{IFEval} & \multicolumn{1}{c}{GPQA-D} \\
\midrule
\multicolumn{11}{>{\hspace*{2pt}}l<{\hspace*{2pt}}}{\textit{Qwen3-8B}}\\
\addlinespace[1pt]
BF16 & 68.6 & 86.7 & 72.9 & 56.7 & 80.9 & 75.0 & 67.7 & 54.4 & 84.8 & 38.4 \\
\arrayrulecolor{black!25}\cmidrule[0.4pt](lr){1-11}\arrayrulecolor{black}
Standard QAT & 67.1 & 83.9 & 70.8 & 55.4 & 80.3 & \textbf{73.6} & \textbf{67.1} & 53.9 & 80.6 & 37.9 \\
\rowcolor{oursrow}[\dimexpr\tabcolsep+1pt\relax][\dimexpr\tabcolsep+1pt\relax]
\textbf{QUASAR}~(ours) & \textbf{67.6} & \textbf{87.0} & \textbf{71.1} & \textbf{55.7} & \textbf{80.4} & \textbf{73.6} & 66.8 & \textbf{54.2} & \textbf{81.2} & \textbf{38.4} \\
\midrule
\multicolumn{11}{>{\hspace*{2pt}}l<{\hspace*{2pt}}}{\textit{Qwen3.5-9B}}\\
\addlinespace[1pt]
BF16 & 71.2 & 74.4 & 78.6 & 55.7 & 74.2 & 78.1 & 72.8 & 53.7 & 74.4 & 79.1 \\
\arrayrulecolor{black!25}\cmidrule[0.4pt](lr){1-11}\arrayrulecolor{black}
Standard QAT & 69.4 & 73.3 & 76.7 & 56.3 & \textbf{75.6} & 76.7 & 70.8 & 52.8 & 66.0 & 76.1 \\
\rowcolor{oursrow}[\dimexpr\tabcolsep+1pt\relax][\dimexpr\tabcolsep+1pt\relax]
\textbf{QUASAR}~(ours) & \textbf{70.2} & \textbf{73.7} & \textbf{77.0} & \textbf{56.6} & \textbf{75.6} & \textbf{77.2} & \textbf{70.9} & \textbf{53.9} & \textbf{68.5} & \textbf{78.1} \\
\bottomrule
\end{tabular}
\end{table}

\begin{figure}[!htb]\centering
\includegraphics[width=4.0in]{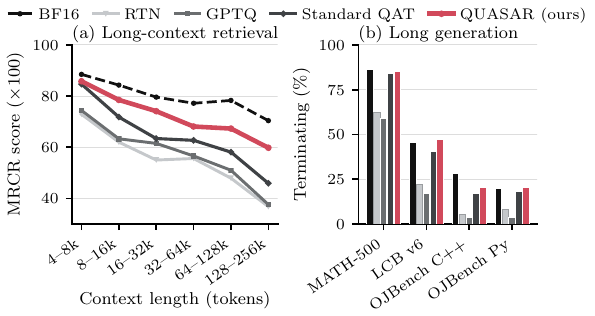}
\caption{\textbf{QUASAR better preserves long-horizon behavior in NVFP4.}
Expanded Qwen3.5-9B results from Table~\ref{tab:merged}: (a) OpenAI-MRCR
score by context length; (b) share of long generations that terminate within
the 32k-token output budget. Higher is better.}
\label{fig:nvfp4-longhorizon}
\end{figure}

\begin{table}[!htb]\centering\scriptsize
\renewcommand{\arraystretch}{0.85}
\caption{\textbf{QUASAR has the lowest or tied-lowest non-termination rate among the quantized checkpoints.}
Percentage of generations that reach the 32{,}768-token output budget without
stopping, using the same generations as Table~\ref{tab:merged}. Lower is
better.}
\label{tab:nvfp4-budget}
\setlength{\tabcolsep}{0pt}
\begin{tabular}{>{\hspace*{2pt}}l<{\hspace*{2pt}} >{\raggedleft\arraybackslash}p{\dimexpr0.130\textwidth+4pt\relax}<{\hspace*{2pt}} >{\raggedleft\arraybackslash}p{\dimexpr0.130\textwidth+4pt\relax}<{\hspace*{2pt}} >{\raggedleft\arraybackslash}p{\dimexpr0.130\textwidth+4pt\relax}<{\hspace*{2pt}} >{\raggedleft\arraybackslash}p{\dimexpr0.130\textwidth+4pt\relax}<{\hspace*{2pt}}}
\toprule
Method & \multicolumn{1}{c}{MATH-500} & \multicolumn{1}{c}{LiveCodeBench v6} & \multicolumn{1}{c}{OJBench (C++)} & \multicolumn{1}{c}{OJBench (Python)} \\
\midrule
\multicolumn{5}{>{\hspace*{2pt}}l<{\hspace*{2pt}}}{\textit{Qwen3-8B}}\\
\addlinespace[1pt]
BF16 & 0.5 & 13.6 & 1.3 & 2.2 \\
\arrayrulecolor{black!25}\cmidrule[0.4pt](lr){1-5}\arrayrulecolor{black}
RTN & 1.4 & 40.8 & 20.3 & 26.7 \\
GPTQ & \textbf{0.5} & 24.3 & 11.6 & 10.3 \\
Standard QAT & \textbf{0.5} & 21.6 & 4.7 & 5.6 \\
\rowcolor{oursrow}[\dimexpr\tabcolsep+1pt\relax][\dimexpr\tabcolsep+1pt\relax]
\textbf{QUASAR}~(ours) & \textbf{0.5} & \textbf{18.4} & \textbf{3.0} & \textbf{2.2} \\
\midrule
\multicolumn{5}{>{\hspace*{2pt}}l<{\hspace*{2pt}}}{\textit{Qwen3.5-9B}}\\
\addlinespace[1pt]
BF16 & 13.9 & 54.6 & 71.5 & 80.2 \\
\arrayrulecolor{black!25}\cmidrule[0.4pt](lr){1-5}\arrayrulecolor{black}
RTN & 37.8 & 78.1 & 94.4 & 91.8 \\
GPTQ & 41.1 & 82.8 & 96.1 & 96.1 \\
Standard QAT & 16.0 & 59.2 & 82.8 & 81.9 \\
\rowcolor{oursrow}[\dimexpr\tabcolsep+1pt\relax][\dimexpr\tabcolsep+1pt\relax]
\textbf{QUASAR}~(ours) & \textbf{14.8} & \textbf{52.9} & \textbf{79.7} & \textbf{79.3} \\
\bottomrule
\end{tabular}
\end{table}

\clearpage

\clearpage
\section{Comparison with Google's QAT and Other 4-bit Checkpoints}\label{app:released-artifacts}
Table~\ref{tab:released-full} extends Table~\ref{tab:merged} with
Qwen3.5-4B and Gemma-4 12B. The RTN rows, and the GPTQ rows at NVFP4 W4A16
and INT4 g64, are our own conversions of the BF16 models; all other rows are
public artifacts, listed in Table~\ref{tab:released-artifacts}.

\begin{table}[!htb]\centering\fontsize{6}{7}\selectfont
\renewcommand{\arraystretch}{0.74}\setlength{\aboverulesep}{0.2ex}\setlength{\belowrulesep}{0.3ex}
\caption{Comparison with released 4-bit checkpoints across four model families. We report fidelity to BF16 and accuracy on each model's benchmark suite. GGUF KL values are comparable only within each block. $\ddagger$ denotes mixed NVFP4, FP8, and BF16 precision.}
\label{tab:released-full}
\ifdefined\rcw\else\newlength{\rcw}\fi\setlength{\rcw}{0.043\textwidth}\setlength{\tabcolsep}{0.8pt}
\begin{tabular}{l l r >{\columncolor{black!8}\raggedleft\arraybackslash}p{\rcw} >{\columncolor{black!8}\raggedleft\arraybackslash}p{\rcw} >{\raggedleft\arraybackslash}p{\rcw} >{\raggedleft\arraybackslash}p{\rcw} >{\raggedleft\arraybackslash}p{\rcw} >{\raggedleft\arraybackslash}p{\rcw} >{\raggedleft\arraybackslash}p{\rcw} >{\raggedleft\arraybackslash}p{\rcw} >{\raggedleft\arraybackslash}p{\rcw} >{\raggedleft\arraybackslash}p{\rcw} >{\columncolor{black!8}\raggedleft\arraybackslash}p{\rcw}}
\toprule
Model & Format & \multicolumn{1}{c}{bpw} & \multicolumn{2}{>{\columncolor{black!8}}c}{Fidelity} & \multicolumn{9}{c}{Accuracy (\%) $\uparrow$} \\
\cmidrule(lr){4-5}\cmidrule(l){6-14}
\multicolumn{3}{l}{\textit{Qwen3.5-4B}} & \multicolumn{1}{>{\columncolor{black!8}}c}{\rotatebox{90}{KL $\downarrow$}} & \multicolumn{1}{>{\columncolor{black!8}}c}{\rotatebox{90}{Top-1 $\uparrow$}} & \multicolumn{1}{c}{\rotatebox{90}{GSM8K-P}} & \multicolumn{1}{c}{\rotatebox{90}{MMLU-P}} & \multicolumn{1}{c}{\rotatebox{90}{IFEval}} & \multicolumn{1}{c}{\rotatebox{90}{MATH}} & \multicolumn{1}{c}{\rotatebox{90}{AIME25}} & \multicolumn{1}{c}{\rotatebox{90}{GPQA-D}} &  &  & \multicolumn{1}{>{\columncolor{black!8}}c}{\rotatebox{90}{\textbf{Avg.}}} \\
BF16 & BF16 & 16 & -- & -- & 94.0 & 78.9 & 88.0 & 84.4 & 79.2 & 77.3 &  &  & 83.6 \\
\arrayrulecolor{black!25}\cmidrule[0.4pt](lr){1-14}\arrayrulecolor{black}
RTN & NVFP4 W4A16 & 4.50 & .042 & 93.8 & \underline{94.9} & 77.5 & 86.3 & 83.8 & 67.1 & 74.1 &  &  & 80.6 \\
GPTQ & NVFP4 W4A16 & 4.50 & \underline{.023} & \underline{95.7} & \textbf{95.1} & 75.2 & 84.5 & 83.2 & 52.9 & 69.9 &  &  & 76.8 \\
GPTQ & INT4 g128 & 7.49 & .049 & 93.1 & 94.0 & 77.5 & 87.1 & 83.0 & 65.0 & 71.2 &  &  & 79.6 \\
AWQ & INT4 g32 & 4.50 & .042 & 93.9 & 93.9 & 77.6 & \underline{87.6} & \underline{84.2} & \underline{72.5} & \underline{74.6} &  &  & \underline{81.7} \\
ModelOpt PTQ & NVFP4 W4A4 & 4.50 & .043 & 93.7 & 94.2 & \underline{77.9} & \textbf{88.2} & \textbf{84.6} & 65.8 & 73.6 &  &  & 80.7 \\
\rowcolor{oursrow}
\textbf{QUASAR}~(ours) & NVFP4 W4A16 & 4.50 & \cellcolor{oursrow}\textbf{.018} & \cellcolor{oursrow}\textbf{96.2} & 94.7 & \textbf{78.5} & 87.2 & 83.8 & \textbf{74.2} & \textbf{77.9} &  &  & \cellcolor{oursrow}\textbf{82.7} \\
\arrayrulecolor{black!25}\cmidrule[0.4pt](lr){1-14}\arrayrulecolor{black}
\multicolumn{14}{l}{\quad\textit{GGUF (llama.cpp)}}\\
RTN & GGUF Q4\_0 & 4.50 & .372 & 89.0 & \underline{94.3} & 77.2 & \underline{90.6} & 84.2 & 64.2 & 73.7 &  &  & 80.7 \\
KD-QAT g64 & GGUF Q4\_0 & 4.50 & .363 & 88.4 & \textbf{94.9} & 77.4 & 84.5 & \underline{84.4} & 64.6 & 73.9 &  &  & 80.0 \\
imatrix (PTQ) & GGUF Q4\_0 & 4.84 & \underline{.284} & \underline{90.4} & \underline{94.3} & \underline{78.0} & \underline{90.6} & \textbf{86.0} & \underline{70.8} & \textbf{76.1} &  &  & \underline{82.6} \\
\rowcolor{oursrow}
\textbf{QUASAR}~(ours) & GGUF Q4\_0 & 4.50 & \cellcolor{oursrow}\textbf{.235} & \cellcolor{oursrow}\textbf{91.3} & \textbf{94.9} & \textbf{78.2} & \textbf{90.9} & 83.8 & \textbf{77.9} & \underline{74.6} &  &  & \cellcolor{oursrow}\textbf{83.4} \\
\midrule
\multicolumn{3}{l}{\textit{Muse-Glimmer-30B}} & \multicolumn{1}{>{\columncolor{black!8}}c}{\rotatebox{90}{KL $\downarrow$}} & \multicolumn{1}{>{\columncolor{black!8}}c}{\rotatebox{90}{Top-1 $\uparrow$}} & \multicolumn{1}{c}{\rotatebox{90}{GPQA-D}} & \multicolumn{1}{c}{\rotatebox{90}{MMLU-P}} & \multicolumn{1}{c}{\rotatebox{90}{AIME25}} & \multicolumn{1}{c}{\rotatebox{90}{LCB}} & \multicolumn{1}{c}{\rotatebox{90}{RUL32K}} & \multicolumn{1}{c}{\rotatebox{90}{RUL64K}} & \multicolumn{1}{c}{\rotatebox{90}{RUL128K}} &  & \multicolumn{1}{>{\columncolor{black!8}}c}{\rotatebox{90}{\textbf{Avg.}}} \\
BF16 & BF16 & 16 & -- & -- & 62.1 & 83.1 & 86.7 & 53.4 & 86.6 & 85.2 & 81.0 &  & 76.9 \\
\arrayrulecolor{black!25}\cmidrule[0.4pt](lr){1-14}\arrayrulecolor{black}
RTN & NVFP4 W4A16 & 4.50 & .054 & 93.2 & \textbf{63.6} & 82.9 & 83.3 & 48.1 & 84.5 & 84.0 & \underline{83.6} &  & 75.7 \\
GPTQ & NVFP4 W4A4 & 4.50 & .066 & 92.4 & \underline{63.1} & \underline{83.0} & \underline{86.7} & \textbf{54.2} & \underline{86.2} & 82.4 & 80.6 &  & 76.6 \\
AutoQuantize & NVFP4+FP8$^{\ddagger}$ & 5.52 & \underline{.030} & \underline{94.9} & \underline{63.1} & 82.9 & 83.3 & \underline{52.7} & 85.6 & \underline{84.8} & 83.5 &  & 76.6 \\
\rowcolor{oursrow}
\textbf{QUASAR}~(ours) & NVFP4 W4A16 & 4.50 & \cellcolor{oursrow}\textbf{.018} & \cellcolor{oursrow}\textbf{96.1} & \textbf{63.6} & \textbf{83.3} & \textbf{90.0} & 49.6 & \textbf{87.2} & \textbf{86.4} & \textbf{84.1} &  & \cellcolor{oursrow}\textbf{77.7} \\
\rowcolor{oursrow}
\textbf{QUASAR}~(ours) & NVFP4 W4A4 & 4.50 & \cellcolor{oursrow}.053 & \cellcolor{oursrow}93.2 & \underline{63.1} & 82.7 & 83.3 & \underline{52.7} & \textbf{87.2} & \textbf{86.4} & 83.5 &  & \cellcolor{oursrow}\underline{77.0} \\
\midrule
\multicolumn{3}{l}{\textit{Gemma-4 12B}} & \multicolumn{1}{>{\columncolor{black!8}}c}{\rotatebox{90}{KL $\downarrow$}} & \multicolumn{1}{>{\columncolor{black!8}}c}{\rotatebox{90}{Top-1 $\uparrow$}} & \multicolumn{1}{c}{\rotatebox{90}{MMLU}} & \multicolumn{1}{c}{\rotatebox{90}{IFEval}} & \multicolumn{1}{c}{\rotatebox{90}{ARC-C}} & \multicolumn{1}{c}{\rotatebox{90}{AGIEval}} & \multicolumn{1}{c}{\rotatebox{90}{MMLU-P}} & \multicolumn{1}{c}{\rotatebox{90}{TQA}} & \multicolumn{1}{c}{\rotatebox{90}{MATH-h}} &  & \multicolumn{1}{>{\columncolor{black!8}}c}{\rotatebox{90}{\textbf{Avg.}}} \\
BF16 & BF16 & 16 & -- & -- & 72.8 & 88.0 & 50.2 & 24.3 & 49.9 & 62.1 & 74.7 &  & 60.3 \\
\arrayrulecolor{black!25}\cmidrule[0.4pt](lr){1-14}\arrayrulecolor{black}
RTN & INT4 g64 & 4.25 & .194 & 89.2 & 60.6 & 79.7 & 46.4 & 19.4 & 39.0 & \underline{60.6} & 50.5 &  & 50.9 \\
GPTQ & INT4 g64 & 4.25 & \underline{.041} & \underline{95.2} & \underline{68.5} & 86.9 & 47.8 & \underline{22.7} & \textbf{48.2} & \textbf{62.1} & 70.9 &  & 58.2 \\
Google QAT & INT4 g32 & 4.50 & .044 & 94.8 & 67.8 & \underline{87.2} & \underline{50.1} & 22.6 & \underline{47.9} & 60.0 & \textbf{73.2} &  & \underline{58.4} \\
\rowcolor{oursrow}
\textbf{QUASAR}~(ours) & INT4 g64 & 4.25 & \cellcolor{oursrow}\textbf{.034} & \cellcolor{oursrow}\textbf{95.5} & \textbf{70.4} & \textbf{87.6} & \textbf{51.3} & \textbf{24.2} & 45.8 & 59.2 & \underline{71.9} &  & \cellcolor{oursrow}\textbf{58.6} \\
\arrayrulecolor{black!25}\cmidrule[0.4pt](lr){1-14}\arrayrulecolor{black}
\multicolumn{14}{l}{\quad\textit{GGUF (llama.cpp)}}\\
RTN & GGUF Q4\_0 & 4.50 & .156 & 89.8 & 59.2 & 87.4 & 47.8 & 22.7 & 38.4 & 60.2 & 62.3 &  & 54.0 \\
Google QAT & GGUF Q4\_0 & 4.50 & \underline{.069} & \underline{93.5} & \textbf{71.0} & \textbf{89.1} & \underline{48.9} & 22.8 & \textbf{45.6} & \textbf{61.3} & \textbf{74.8} &  & \textbf{59.1} \\
imatrix (PTQ) & GGUF Q4\_0 & 4.52 & .108 & 91.6 & 64.9 & 87.1 & 48.6 & \underline{23.5} & 43.4 & \underline{60.3} & 66.8 &  & 56.4 \\
\rowcolor{oursrow}
\textbf{QUASAR}~(ours) & GGUF Q4\_0 & 4.50 & \cellcolor{oursrow}\textbf{.048} & \cellcolor{oursrow}\textbf{94.3} & \underline{70.3} & \underline{87.6} & \textbf{51.5} & \textbf{24.2} & \underline{45.5} & 59.1 & \underline{72.5} &  & \cellcolor{oursrow}\underline{58.7} \\
\midrule
\multicolumn{3}{l}{\textit{Gemma-4 E4B}} & \multicolumn{1}{>{\columncolor{black!8}}c}{\rotatebox{90}{KL $\downarrow$}} & \multicolumn{1}{>{\columncolor{black!8}}c}{\rotatebox{90}{Top-1 $\uparrow$}} & \multicolumn{1}{c}{\rotatebox{90}{TrivQA}} & \multicolumn{1}{c}{\rotatebox{90}{NQ}} & \multicolumn{1}{c}{\rotatebox{90}{TQA}} & \multicolumn{1}{c}{\rotatebox{90}{ARC-E}} & \multicolumn{1}{c}{\rotatebox{90}{ARC-C}} & \multicolumn{1}{c}{\rotatebox{90}{GSM8K}} & \multicolumn{1}{c}{\rotatebox{90}{IFEval}} & \multicolumn{1}{c}{\rotatebox{90}{MATH-h}} & \multicolumn{1}{>{\columncolor{black!8}}c}{\rotatebox{90}{\textbf{Avg.}}} \\
BF16 & BF16 & 16 & -- & -- & 23.6 & 3.9 & 59.1 & 79.7 & 56.5 & 84.1 & 85.2 & 63.4 & 56.9 \\
\arrayrulecolor{black!25}\cmidrule[0.4pt](lr){1-14}\arrayrulecolor{black}
RTN & INT4 g64 & 4.25 & .079 & 91.5 & \textbf{25.2} & 3.0 & \textbf{59.3} & \underline{78.3} & 55.2 & 79.7 & \underline{83.0} & 56.1 & 55.0 \\
GPTQ & INT4 g64 & 4.25 & \underline{.031} & \underline{94.7} & 21.3 & \underline{3.3} & 56.1 & 78.2 & 54.9 & \textbf{83.5} & \textbf{83.4} & \underline{59.4} & \underline{55.0} \\
Google QAT & INT4 g32 & 4.50 & .064 & 92.6 & 20.1 & 2.4 & 55.4 & \underline{78.3} & \underline{56.0} & 81.5 & 81.9 & \underline{59.4} & 54.4 \\
\rowcolor{oursrow}
\textbf{QUASAR}~(ours) & INT4 g64 & 4.25 & \cellcolor{oursrow}\textbf{.022} & \cellcolor{oursrow}\textbf{95.5} & \underline{23.8} & \textbf{4.7} & \underline{58.2} & \textbf{79.8} & \textbf{56.2} & \underline{82.0} & 82.6 & \textbf{60.7} & \cellcolor{oursrow}\textbf{56.0} \\
\arrayrulecolor{black!25}\cmidrule[0.4pt](lr){1-14}\arrayrulecolor{black}
\multicolumn{14}{l}{\quad\textit{GGUF (llama.cpp)}}\\
RTN & GGUF Q4\_0 & 4.50 & .194 & 85.7 & 17.3 & 2.0 & \underline{58.5} & 76.2 & 53.8 & 75.8 & 80.2 & 48.7 & 51.6 \\
Google QAT & GGUF Q4\_0 & 4.50 & .088 & 90.6 & 18.5 & 2.3 & 55.5 & 78.7 & \textbf{56.1} & 81.4 & 81.7 & 58.7 & 54.1 \\
Google QAT repack & GGUF Q4\_0 & 4.50 & .071 & \underline{91.8} & 20.0 & 2.2 & 56.2 & 79.1 & \underline{55.4} & \textbf{83.5} & \textbf{83.2} & \textbf{61.8} & \underline{55.2} \\
imatrix (PTQ) & GGUF Q4\_0 & 4.52 & \underline{.067} & 91.1 & \underline{20.6} & \underline{3.0} & \textbf{58.8} & \underline{79.3} & 53.8 & 80.4 & \textbf{83.2} & 59.7 & 54.9 \\
\rowcolor{oursrow}
\textbf{QUASAR}~(ours) & GGUF Q4\_0 & 4.50 & \cellcolor{oursrow}\textbf{.044} & \cellcolor{oursrow}\textbf{92.8} & \textbf{24.0} & \textbf{4.7} & 58.2 & \textbf{79.7} & \textbf{56.1} & \underline{81.8} & \underline{82.4} & \underline{60.5} & \cellcolor{oursrow}\textbf{55.9} \\
\bottomrule
\end{tabular}
\end{table}

bpw includes the quantized transformer-block weights, scales, and zero-points.
Embeddings, output layers, norms, and small gating tensors are excluded.

\paragraph{Artifact definitions.}
For rows not described above:
\begin{itemize}[leftmargin=*,nosep]
\item \textbf{GPTQ (INT4 g128), AWQ (INT4 g32), and GPTQ (NVFP4 W4A4)} are checkpoints created with LLM Compressor~\citep{llmcompressor}; g$k$ denotes one scale per $k$ weights, and W4A4 also quantizes activations to NVFP4.
\item \textbf{ModelOpt PTQ} applies calibrated NVFP4 rounding without training. \textbf{AutoQuantize} selects NVFP4, FP8, or BF16 per layer under a 5.5-bit budget.
\item \textbf{Google QAT} uses Google's unpublished INT4 recipe with 32-weight groups. Its Q4\_0 variants convert or repack those weights for \texttt{llama.cpp}.
\item \textbf{KD-QAT g64} distills Qwen3.5-4B on a 64-weight INT4 grid before conversion to 32-weight Q4\_0 blocks.
\item \textbf{imatrix (PTQ)} selects GGUF Q4\_0 scales using activation-weighted reconstruction error; the GGUF \textbf{RTN} rows use plain round-to-nearest.
\end{itemize}

\begin{table}[h]\centering\scriptsize
\caption{Released checkpoints used in Tables~\ref{tab:merged}
and~\ref{tab:released-full}.}\label{tab:released-artifacts}
\begin{tabular}{l l l}
\toprule
Model & Row & Artifact \\
\midrule
Qwen3.5-4B & BF16 & \texttt{Qwen/Qwen3.5-4B}~\citep{hf_qwen35_4b} \\
 & GPTQ (INT4 g128) & \texttt{RedHatAI/Qwen3.5-4B-quantized.w4a16}~\citep{hf_redhat_qwen35} \\
 & AWQ (INT4 g32) & \texttt{cyankiwi/Qwen3.5-4B-AWQ-4bit}~\citep{hf_cyankiwi_qwen35_awq} \\
 & ModelOpt PTQ & \texttt{cosmicproc/Qwen3.5-4B-NVFP4}~\citep{hf_cosmicproc_qwen35_nvfp4} \\
 & imatrix (GGUF) & \texttt{bartowski/Qwen\_Qwen3.5-4B-GGUF}~\citep{hf_bartowski_qwen35_gguf} \\
 & KD-QAT g64 (GGUF) & \texttt{YoozLabs/Qwen3.5-4B-qat-GGUF}~\citep{hf_yooz_qwen35_qat_gguf} \\
\midrule
Muse-Glimmer-30B & BF16 & \texttt{meta-models/Muse-Glimmer-30B}~\citep{hf_muse_glimmer_30b} \\
 & AutoQuantize & \texttt{nvidia/Muse-Glimmer-30B-NVFP4}~\citep{hf_nvidia_muse_nvfp4} \\
 & GPTQ (NVFP4 W4A4) & \texttt{RedHatAI/Muse-Glimmer-30B-NVFP4}~\citep{hf_redhat_muse_nvfp4} \\
\midrule
Gemma-4 E4B & BF16 & \texttt{google/gemma-4-E4B-it}~\citep{hf_gemma4_e4b} \\
 & Google QAT & \texttt{google/gemma-4-E4B-it-qat-w4a16-ct}~\citep{hf_google_gemma4_e4b_qat_ct} \\
 & Google QAT (GGUF) & \texttt{google/gemma-4-E4B-it-qat-q4\_0-gguf}~\citep{hf_google_gemma4_e4b_qat_gguf} \\
 & Google QAT repack (GGUF) & \texttt{unsloth/gemma-4-E4B-it-qat-GGUF}~\citep{hf_unsloth_gemma4_e4b_qat_gguf} \\
 & imatrix (GGUF) & \texttt{unsloth/gemma-4-E4B-it-GGUF}~\citep{hf_unsloth_gemma4_e4b_gguf} \\
\midrule
Gemma-4 12B & BF16 & \texttt{google/gemma-4-12B-it}~\citep{hf_gemma4_12b} \\
 & Google QAT & \texttt{google/gemma-4-12B-it-qat-w4a16-ct}~\citep{hf_google_gemma4_12b_qat_ct} \\
 & Google QAT (GGUF) & \texttt{google/gemma-4-12B-it-qat-q4\_0-gguf}~\citep{hf_google_gemma4_12b_qat_gguf} \\
 & Google QAT repack (GGUF) & \texttt{unsloth/gemma-4-12B-it-qat-GGUF}~\citep{hf_unsloth_gemma4_12b_qat_gguf} \\
 & imatrix (GGUF) & \texttt{unsloth/gemma-4-12b-it-GGUF}~\citep{hf_unsloth_gemma4_12b_gguf} \\
\bottomrule
\end{tabular}
\end{table}

\end{document}